\documentclass[amssymb,10pt]{article}
\usepackage[dvipdfm]{graphicx}
\usepackage{amsthm}
\usepackage[usenames]{color}
\usepackage{xspace,colortbl}

\usepackage{latexsym,url,amscd,amssymb}
\usepackage{amsmath}
\usepackage{graphics}
\usepackage{epstopdf}
\usepackage{hyperref}
\usepackage{subcaption}
\usepackage{booktabs}
\usepackage{cases}
\usepackage{caption}
\usepackage{algorithm}
\usepackage{algpseudocode}
\usepackage{listings}
\usepackage{rotating}
\usepackage{array}
\usepackage{multirow}
\usepackage{makecell}
\usepackage{framed}
\usepackage{picins}
\usepackage[titletoc]{appendix}
\usepackage{setspace}
\usepackage{indentfirst}
\usepackage{natbib}
\usepackage{url} % not crucial - just used below for the URL
\DeclareMathOperator*{\argmin}{arg\,min}

\def\A{{\mathcal A}}

\def\I{{\mathcal I}}

\newtheorem{lemma}{Lemma}[section]

\newtheorem{theorem}{Theorem}[section]
\newtheorem{assumption}{Assumption}[section]

\newtheorem{remark}{Remark}[section]

\newcommand{\tabincell}[2]

\renewcommand{\thesubsection}{\arabic{section}.\arabic{subsection}}

\renewcommand{\theequation}{\arabic{section}.\arabic{equation}}

\renewcommand{\thelemma}{\arabic{section}.\arabic{lemma}}

\title {A Data-Driven Line Search Rule for Support Recovery in High-dimensional Data Analysis}%A Novel Line Search based SDAR Algorithm for High-dimensional Sparse Optimization}

\author{Peili Li \thanks{School of Statistics, Key Laboratory of Advanced Theory and Application in Statistics and Data Science-MOE, East China Normal University, Shanghai 200062, PR China (Email: plli@sfs.ecnu.edu.cn).}     \and
        Yuling Jiao \thanks{School of Mathematics and Statistics, and
Hubei Key Laboratory of Computational Science, Wuhan University, Wuhan 430072, PR China (Email: yulingjiaomath@whu.edu.cn).}  \and
        Xiliang Lu \thanks{School of Mathematics and Statistics, and
Hubei Key Laboratory of Computational Science, Wuhan University, Wuhan 430072, PR China (Email: xllv.math@whu.edu.cn).}
\and
        Lican Kang \thanks{School of Mathematics and Statistics, Wuhan University, Wuhan 430072, PR China (Email: kanglican@whu.edu.cn).}
}

\date{ }
\begin{document}
\maketitle

\begin{abstract}
\setlength\parindent{2em}

In this work, we consider the algorithm to the (nonlinear) regression problems with $\ell_0$ penalty. The existing algorithms for $\ell_0$ based optimization problem are often carried out with a fixed step size, and the selection of an appropriate step size depends on the restricted strong convexity and smoothness for the loss function, hence it is difficult to compute in practical calculation. In sprite of the ideas of support detection and root finding \cite{HJK2020}, we proposes a novel and efficient data-driven line search rule to adaptively determine the appropriate step size. We prove the $\ell_2$ error bound to the proposed algorithm without much restrictions for the cost functional.
%We also propose an adaptive version of the previous algorithm to deal with the situation that the true sparsity level is unknown in advance.
A large number of numerical comparisons with state-of-the-art algorithms in linear and logistic regression problems show the stability, effectiveness and superiority of the proposed algorithms.
\end{abstract}

{\bf Key words.} High-dimensional data analysis, sparsity assumption, $\ell_0$ penalty, line search, $\ell_2$ error bound.
%%%%%%%%%%%%%%%%%%%%%%%%%%%%%%%%%%%%%%%%%%%%%%%%%%%%%%%%%%%%%%%%%%%%%
\setcounter{equation}{0}
\section{Introduction}
\setlength\parindent{2em}
%%%%%%%%%%%%%%%%%%%%%%%%%%%%%%%%%%%%%%%%%%%%%%%%%%%%%%%%%%%%%%%%%%%%%

Recently high-dimensional data analysis with sparsity assumption had attracted increasing research interest. Sparse regression is not only produced a lot of considerable theoretical analysis and algorithms, but also was widely used in statistics, machine learning, artificial intelligence, engineering design and other fields.
The many application areas, such as matrix estimation \cite{CLL2011,LX2018}, signal/image processing \cite{FJL2014,LDP2007}, high-dimensional variable selection \cite{L2009,ZH2005} and machine learning \cite{B2006,J2011}, the data dimension is often greater or even much greater than the number of samples, then the researchers often assume that the model is sparse at this time and consider the following Lagrange version (or sparsity-constrained form) of the $\ell_0$-penalized minimization problem:
\begin{equation}\label{lagran}
\min_{\beta\in \mathbb{R}^{p}} F(\beta)+\lambda ||\beta||_0,
\end{equation}
where $F: \mathbb{R}^p \rightarrow \mathbb{R}$ is a smooth convex loss function, such as the least squares function in linear regression \cite{HJL2018, JJL2015, S1971} and the log-likelihood function in classical logistic regression \cite{HJK2020, MVB2008, YLZ2018}, etc. $\lambda>0$ is a regularization parameter, controlling the sparsity level of the regularized solution, and $||\beta||_0$ denotes the number of nonzero components in parameter vector $\beta$.

Due to the nonconvexity and discontinuity of the function $||\beta||_0$, it is very challenging to develop an efficient algorithm to accurately solve the model (\ref{lagran}). Therefore, many researchers turn to approximately solve model (\ref{lagran}) with other easy-to-handle penalty functions, including the popular LASSO \cite{CDS2001,PH2007,T1996,V2008}, group LASSO \cite{MVB2008,RF2008}, elastic net \cite{FHT2010,JLS2009}, smoothly clipped absolute deviation (SCAD) \cite{FL2001}, minimax concave penalty (MCP) \cite{Z2010}, capped-$\ell_1$ \cite{Zh2010}, etc. In addition, many algorithms with good numerical performance have been designed accordingly, such as coordinate descent type algorithms \cite{BH2011,FHHT2007,WL2008}, proximal gradient descent algorithm \cite{ANW2012,XZ2012}, alternating direction
method of multipliers (ADMM) \cite{BPC2011}, primal dual active set algorithm (PDAS) \cite{FJL2014,HJJ2019}, to name a few.

As for the challenging model (\ref{lagran}) and its sparsity-constrained form, there also exist many effective algorithms, such as adaptive forward-backward greedy algorithm \cite{Z2008}, gradient support pursuit algorithm \cite{BRB2013}, hard thresholding pursuit\cite{CQ2017,SL2017,YLZ2014,YLZ2018}, Newton-type methods \cite{HJK2020,HJL2018,YL2017,ZPX2021}, etc. However, after in-depth analysis of the above relevant literatures, we can find that almost all algorithms are based on a fixed step size to expand theoretical analysis and numerical calculation. For example, Yuan et al. \cite{YLZ2014} used the formula $\tilde{\beta}^k=\beta^{k-1}-\tau \nabla F(\beta^{k-1})$ to calculate an intermediate iteration step in gradient hard thresholding pursuit algorithm, where $\tau$ is a fixed step size. In their convergence analysis, a strict constraint condition $\tau<m_K / M_K^2$ is needed, where $K$ denotes the sparsity level of the underlying variable $\beta^*$, $m_K$ and $M_K$ are the parameters of the objective function $F(\beta)$ that $F$ satisfies $m_K$-strongly convex and $M_K$-strongly smooth. Similar condition $\tau<1/M_{2K}$ is also needed in \cite{SL2017,YLZ2018}. It is known that the parameters $m_K$ and $M_K$ are not easy to calculation in practice, then this brings challenges to the selection of an appropriate step size. In addition, for $\ell_0$ regularized generalized linear models, the fixed step size $\tau=1$ is also widely used in GSDAR algorithm which is an equivalent form of Newton algorithm \cite{HJK2020}. Although it is an effective method, it also has some drawbacks. Theoretically, they establish the $\ell_{\infty}$ error bound for the GSDAR estimator under a strong assumption of the loss function $F$, which limits the application scope of the algorithm.
 %If the needed assumption is not satisfied, the appropriate step size must be increased numerically. However, the empirical step size needs to be adjusted manually, which is quite troublesome.
%For example, theoretically, they establish the $\ell_{\infty}$ error bound for the GSDAR estimator under the following assumption
%\begin{align*}
%0<L\leq \frac{(\beta_1-\beta_2)^{\top}\cdot \nabla^2 F(\tilde{\beta})\cdot (\beta_1-\beta_2)}{\|\beta_1-\beta_2\|_1\|\beta_1-\beta_2\|_{\infty}}\leq U <\infty, \quad \forall \beta_1\neq \beta_2\quad with \quad \|\beta_1-\beta_2\|_0\leq 2T,
%\end{align*}
%with $0<U<\frac{1}{T}$, where $\tilde{\beta}=\beta_1+\nu(\beta_2 -\beta_1), \nu\in (0,1)$, and $T$ is a positive constant greater than the sparsity level $K$.

Recently, there have been proposed several methods to iteratively update the step size. For $\ell_0$ constrained high-dimensional logistic regression model, Wang et al. \cite{WXZ2019} proposed a fast Newton method with adaptively updating the step size $\tau$. They started $\tau$ with a fixed one $\tau_0=1$ and then decreasingly updated
it as following rule:
\begin{equation}\notag
\tau_{k+1}= \left\{
\begin{array}{ll}
0.1 \tau_k, \quad if \ \tau_k\geq \frac{\|\beta^k\|_{K,\infty}}{max_{j\in \I^k}|d_j^k|} \ and \ \|F_{\tau_k}(u^k;\A^k)\|>1/k \vspace{1ex},\\
\tau_k, \quad \quad otherwise,
\end{array}
\right.
\end{equation}
where $||\beta^k||_{K,\infty}$ denotes the $K$-th largest elements in absolute value of $\beta^k$, $d=\nabla F(\beta)$ is the dual information, $\A=\{i: |\beta_i-\tau d_i|\geq ||\beta-\tau d||_{K,\infty}\}$ is the active set, $\I=\A^{c}$ denotes the inactive set, and $F_{\tau}(u;\A)$ is defined by
$$F_{\tau}(u;\A):=\left(
                  \begin{array}{c}
                    d_{\A} \\
                    \beta_{\I} \\
                    d_{\A}-\nabla_{{\A}}F(\beta) \\
                    d_{\I}-\nabla_{\I}F(\beta) \\
                  \end{array}
                \right).
$$
Moreover, Zhou et al. \cite{ZPX2021} proposed a Newton-type algorithm with the Amijio line search for the general $\ell_0$-regularized optimization. However, it is worthwhile noting that the selection of step size in their paper is also related to the difficult-to-solve parameters of the loss function, so they actually adopt the following method to update $\tau$ adaptively in the numerical experiments,
\begin{equation}\notag
\tau_{k+1}= \left\{
\begin{array}{lll}
4\tau_k/5, \quad if \ k/10=\lceil k/10 \rceil \ and \ \|F_{\tau_k}(u^k;\A^k)\|>k^{-2} \vspace{1ex},\\
5\tau_k/4, \quad if \ k/10=\lceil k/10 \rceil \ and \ \|F_{\tau_k}(u^k;\A^k)\|\leq k^{-2} \vspace{1ex},\\
\tau_k, \quad \quad otherwise,
\end{array}
\right.
\end{equation}
where $\lceil n\rceil $ represents the smallest integer that is no less than $n$.

\subsection{Contribution}

In this paper, we propose a data-driven line search rule to adaptively determine the appropriate step size in the framework of GSDAR algorithm \cite{HJK2020}. The whole algorithm is named by SDARL (support detection and root finding with line search). With the line search technique, the assumptions to the loss function $F$ is much weaker than that in \cite{HJK2020} for the convergence analysis. This can naturally expand the application scope of the SDARL algorithm.
%From a specific operational point of view, we start by analysing the Karush-Kuhn-Tucker (KKT) conditions with some $\tau>0$ for $\ell_0$-penalized minimization problem (\ref{lagran}). Then based on support detection using primal and dual information and root finding, we propose a nontrivial line search rule to adaptively update the step size $\tau$ and generate a sequence of solutions for the KKT system iteratively.
We also propose an adaptive version of the SDARL to deal with the situation that the true sparsity level is unknown in advance. Theoretically, we will prove the proposed data-driven line search rule is well-defined. With the step size determined by line search, we establish the $\ell_2$ error bound of iteration sequence $\beta^k$ and the target regression coefficient $\beta^*$ without any restrictions on the restricted strong convexity and smoothness for loss function $F$. In addition, we reveal the support recovery in finite step with some reasonable condition about the target $\beta^*$. The stability, effectiveness and superiority of the algorithms in this paper are highlighted by comparing the numerical performance with state-of-the-art algorithms in the setting of linear and logistic regressions.

\subsection{Notation}
Because the regression analysis usually involves quite a few notations, we briefly summarize them here. For a vector $\beta\in \mathbb{R}^{p}$, we use $||\beta||$, $||\beta||_{T,\infty}$, $||\beta||_{\min}$ and $\text{supp}(\beta)$ to denote the Euclidean norm, the $T$-th largest elements in absolute value, the minimum absolute value and the support of $\beta$, respectively. We will use $\beta^*$ and $\hat{\beta}$ to denote the target and the estimated regression coefficient, respectively. Let $\A$ be an index set, we denote $\A^*=\text{supp}(\beta^*)$ and $\hat{\A}=\text{supp}(\hat{\beta})$ as the target and estimated support set. $|\A|$ denotes the length of the set $\A$. $\beta_{\A}=(\beta_i,i\in \A)\in \mathbb{R}^{|\A|}$. $\beta|_{\A}\in \mathbb{R}^p$ and its $i$-th element is $(\beta|_{\A})_i=\beta_i \delta(i\in \A)$, where $\delta(\cdot)$ is the indicator function. For a matrix $X$, $X_{\A}\in \mathbb{R}^{n\times |\A|}$. The first and second derivatives of function $F$ will be denoted by $\nabla F$ and $\nabla^2 F$, respectively.
\subsection{Organization}
The remainder of the paper is organized as follows: Section \ref{algorithm} describes the details of SDARL algorithm and further develops an adaptive version (abbreviated as ASDARL) of SDARL. Section \ref{conana} analyzes the well-defined property of the proposed line search rule and establishes the $\ell_2$ error bound between the estimator and the target regression coefficient. In Section \ref{mumres}, we demonstrate the performance of the proposed algorithms by comparing with some state-of-the-art algorithms in sparse linear and logistic regression problems. Finally, we conclude this paper in Section \ref{con}. Proofs for all the lemmas and theorems are provided in the Appendix.

\section{Methodology}\label{algorithm}
\setlength\parindent{2em}
In this section, we firstly give SDARL algorithm for model (\ref{lagran}), then develop ASDARL algorithm to deal with the situation that the true sparsity level is not known in advance.
\subsection{SDARL algorithm}

The existence of the global minimizers for the $\ell_0$ penalty problems has been verified in \cite{HJJ2019,N2013}, but in view of the nonconvexity and nonsmoothess, it is difficult to obtain directly. Referring to \cite{HJK2020,HJL2018}, we also pay attention to the KKT conditions of (\ref{lagran}) with some $\tau>0$, which is a necessary optimality condition for the global minimizers, and also is a sufficient condition for the local minimizers. One can refer to the following Lemma \ref{le1}.

\begin{lemma}\label{le1}
Suppose that $\bar{\beta}$ is a global minimizer of (\ref{lagran}), then there exists a $\bar{d}$ such that the following KKT system holds,
\begin{equation}\label{KKT}
\left\{
\begin{array}{ll}
\bar{d}=-\nabla F(\bar{\beta}) \vspace{1ex},\\
\bar{\beta}=H_{\lambda \tau}(\bar{\beta}+\tau \bar{d}),
\end{array}
\right.
\end{equation}
where $\tau >0$ is the step size, and the i-th element of $H_{\lambda \tau}(\beta)$ is defined by
\begin{equation}\notag
(H_{\lambda \tau}(\beta))_{i}\left\{
\begin{array}{ll}
=0, \quad \quad \quad |\beta_i|<\sqrt{2\lambda \tau} \vspace{1ex},\\
\in \{0,\beta_i\}, \quad |\beta_i|= \sqrt{2\lambda \tau} \vspace{1ex},\\
=\beta_i, \quad \quad \quad |\beta_i|> \sqrt{2\lambda \tau}.
\end{array}
\right.
\end{equation}
Conversely, if $\bar{\beta}$ and $\bar{d}$ satisfy (\ref{KKT}), then $\bar{\beta}$ is a local minimizer of (\ref{lagran}).
\end{lemma}
\begin{proof}
See Appendix \ref{leproof}.
\end{proof}

Let $\bar{\A}=supp(\bar{\beta})$ and $\bar{\I}=(\bar{\A})^{c}$. Combining (\ref{KKT}) with the definition of $H_{\lambda \tau}(\beta)$, we can get
\begin{equation}\notag
\bar{\A}=\{i:|\bar{\beta}_i+\tau \bar{d}_i|\geq\sqrt{2\lambda \tau}\}, \quad \bar{\I}=\{i:|\bar{\beta}_i+\tau \bar{d}_i|< \sqrt{2\lambda \tau}\},
\end{equation}
and
\begin{equation}\notag
\left\{
\begin{array}{llll}
\bar{\beta}_{\bar{\I}}=0 \vspace{1ex},\\
\bar{d}_{\bar{\A}}=0 \vspace{1ex},\\
\bar{\beta}_{\bar{\A}}\in \argmin\limits_{\beta_{\bar{\A}}}\widetilde{F}(\beta_{\bar{\A}}) \vspace{1ex},\\
\bar{d}_{\bar{\I}}=[-\nabla F(\bar{\beta})]_{\bar{\I}}.
\end{array}
\right.
\end{equation}
where $\widetilde{F}(\beta_{\bar{\A}})=F(\beta|_{\bar{\A}})$. Let $\{\beta^k,d^k\}$ be the output of $k$-th iteration in SDARL algorithm. We approximate $\{\bar{\A},\bar{\I}\}$ by
\begin{equation}\label{acik}
\A^k=\{i:|\beta^k_i+\tau d^k_i|\geq\sqrt{2\lambda \tau}\}, \quad \I^k=\{i:|\beta^k_i+\tau d^k_i|< \sqrt{2\lambda \tau}\}.
\end{equation}
Then we can obtain a new approximation pair $\{\beta^{k+1},d^{k+1}\}$ by
\begin{equation}\label{itek}
\left\{
\begin{array}{llll}
\beta^{k+1}_{\I^k}=\textbf{0} \vspace{1ex},\\
d^{k+1}_{\A^k}=\textbf{0} \vspace{1ex},\\
\beta^{k+1}_{\A^k}\in \argmin\limits_{\beta_{\A^k}}\widetilde{F}(\beta_{\A^k}),\\
d^{k+1}_{\I^k}=[-\nabla F(\beta^{k+1})]_{\I^k},
\end{array}
\right.
\end{equation}
where $\widetilde{F}(\beta_{\A^k})=F(\beta|_{\A^k})$. If the minimizer $\beta^{k+1}_{\A^k}$ of (\ref{itek}) is not unique, we choose the one with the smallest value in $\ell_{2}$-norm. From (\ref{acik}), we can see that the calculation of the active and inactive sets is related to the regularization parameter $\lambda$ which is sensitive in practical problems. Similar as in \cite{HJK2020,HJL2018}, we first give an assumption on the true sparsity level:
\begin{assumption}\label{ass1}
The true signal has $K$ nonzero element, and $K\leq T$, i.e., $||\beta^*||_0=K\leq T$.
\end{assumption}
Under the above assumption, we set $\sqrt{2\lambda \tau}=\|\beta^k+\tau d^k\|_{T,\infty}$, which guarantees that $|\A^k|=T$ in every iteration of SDARL algorithm. This point greatly reduces the computational complexity of the proposed algorithm. Then combining (\ref{acik}) and (\ref{itek}), we can give algorithm SDARL in the following:
\begin{framed}\label{algo}
\noindent
{\bf SDARL Algorithm:}\\
\hrule \vskip 1mm
\noindent \textbf{Step 0} Given $T>0, \nu\in (0,1), \sigma \in (0,1/2)$. Choose $\beta^0$, $d^0=-\nabla F(\beta^0)$. Set $\tau^0=1$ and

\quad $\A^{0}=\{i:|\beta_i^0+ d_i^0|\geq ||\beta^0+ d^0||_{T,\infty}\}$, $\I^{0}=(\A^{0})^{c}$, for $k=0$, \\
\textbf{Step 1} Compute
\begin{equation}\notag
\left\{
\begin{array}{llll}
\beta^{k+1}_{\I^k}=\textbf{0} \vspace{1ex},\\
d^{k+1}_{\A^k}=\textbf{0} \vspace{1ex},\\
\beta^{k+1}_{\A^k}\in \argmin\limits_{\beta_{\A^k}}\widetilde{F}(\beta_{\A^k}),\\
d^{k+1}_{\I^k}=[-\nabla F(\beta^{k+1})]_{\I^k}.
\end{array}
\right.
\end{equation}
\textbf{Step 2} Set $\tau^{k+1}=\nu^{m_{k+1}}$, where $m_{k+1}$ is the smallest non-negative integer $m$ such that
\begin{align}\label{line}
F\Big((\beta^{k+1}+\nu^{m} d^{k+1})|_{\A^{k+1}}\Big)-F\Big(\beta^{k+1}\Big)\leq -\sigma \nu^{m}\Big\|\nabla_{\A^{k+1}\backslash \A^{k}}F(\beta^{k+1})\Big\|^2,
\end{align}
\quad \quad \quad where $\A^{k+1}=\Big\{i:|\beta_i^{k+1}+\nu^{m} d_i^{k+1}|\geq ||\beta^{k+1}+\nu^{m} d^{k+1}||_{T,\infty}\Big\}$.\\
\textbf{Step 3} Compute
$$\A^{k+1}=\Big\{i:|\beta_i^{k+1}+\tau^{k+1} d_i^{k+1}|\geq ||\beta^{k+1}+\tau^{k+1} d^{k+1}||_{T,\infty}\Big\}, \quad \I^{k+1}=(\A^{k+1})^{c}.$$ If $\A^k=\A^{k+1}$, then stop, otherwise, go to \textbf{Step 1}.
\end{framed}

\begin{remark}
The merit function in the line search step \eqref{line} is the loss function $F$ which restricts on
the first $T$ largest absolute values of $\beta$. This merit function is quite different from other merit functions for either gradient type method or Newton type method.
As we will see in the proof later, this line search ensures the original objective function $F(\beta)$ decreases in a certain sense.
%This not only reduces the calculation amount of the algorithm, but also ensures that the original objective function $F(\beta)$ is in a downward trend, which will be explained in detail in the following theoretical analysis.
\end{remark}

%\begin{remark}
%It is worth explaining the way of solving the sub-problem $\beta^{k+1}_{\A^k}\in \argmin\limits_{\beta_{\A^k}}\widetilde{F}(\beta_{\A^k})$ in the subsequent numerical calculations in detail. The explicit expression of the solution is directly used for linear regression problems. While for logistic regression problems, we adopt the quasi-Newton method to solve the corresponding sub-problem.
%\end{remark}
%\section{Adaptive SDARL}\label{ada}
%\setlength\parindent{2em}
\subsection{ASDARL algorithm}
To apply SDARL algorithm, one need to estimate the sparse level of true parameter $\beta^*$ in advance.  For many practical problems which the true sparsity level is not known,  we propose an adaptive version ASDARL, which regards $T$ as a tuning parameter. Similar as \cite{HJK2020,HJL2018}, let $T$ increase continuously from 0 to $Q= n/log(n)$ \cite{FL2008}, then we can get a set of solutions paths: $\{\hat{\beta}(T): T=0,1,\cdots,Q\}$, where $\hat{\beta}(0)=\mathbf{0}$. Afterwards we can use either cross validation or HBIC \cite{WKL2013} to select a $\hat{T}$ and use $\hat{\beta}(\hat{T})$ as the final estimation of $\beta^{*}$. In summary, we give ASDARL algorithm in following.

\begin{framed}\label{algo2}
\noindent
{\bf ASDARL Algorithm:}
\vskip 1.0mm \hrule \vskip 1mm
\noindent  Initialize $\beta^0, d^0$, integers $\alpha$, $Q$. Set $k=1$. \\
\textbf{for} $k=1,2,\cdots,$ \textbf{do}

Run SDARL Algorithm with $T=\alpha k$ and with initial value $(\beta^{k-1}, d^{k-1})$. Denote the output by

$(\beta^{k},d^k)$.

\textbf{if} $T>Q$ \textbf{then}

stop

\textbf{else}

$k=k+1$.

\textbf{end if}\\
\textbf{end for}
\end{framed}

\section{Convergence Analysis}\label{conana}
\setlength\parindent{2em}

In this section, we will conduct the theoretical analysis of SDARL algorithm. Firstly, given a positive constant $s$, we assume function $F$ satisfies the following $m_{s}$-restricted strong convexity (RSC) and $M_{s}$-restricted strong smoothness (RSS) which are commonly used to analyze nonconvex problems.

\begin{assumption}\label{ass2}
There exist constants $0<m_{s}\leq M_{s}$, such that
\begin{align*}
\frac{m_{s}}{2}\|\beta_1-\beta_2\|^2\leq F(\beta_1)-F(\beta_2)-\langle \nabla F(\beta_2),\beta_1-\beta_2\rangle\leq\frac{M_{s}}{2}\|\beta_1-\beta_2\|^2, \quad \forall \|\beta_1-\beta_2\|\leq s.
\end{align*}
\end{assumption}
\begin{remark}
Assumption \ref{ass2} are the restricted strong convexity and smoothness conditions that is needed in bounding the estimation error in high- dimensional models  \cite{M2019}. In the case of linear model $y=X \beta+\epsilon$, where $X\in \mathbb{R}^{n\times p}$ is the design matrix with $\sqrt{n}$-normalized columns, $y\in \mathbb{R}^{n}$ is a response vector, and $\varepsilon \in \mathbb{R}^{n}$ is the additive Gaussian noise. The $m_{s}$-RSC can be established when $X_{s}$ is full column rank, and $M_{s}$ can be controlled by the maximum eigenvalue of the matrix $X^{\top}X$.
\end{remark}

%Later, it will be explained in detail that under the above weak conditions, the algorithm SDARL in this paper can guarantee global convergence from the perspective of optimization.
The enforceability of the termination condition $\A^k=\A^{k+1}$ has been verified in \cite{HJK2020,HJL2018}. Below we can verify the feasibility of the algorithm as long as we verify the well-defined nature of the specific line search. %For ease of understanding, we introduce a new vector $\theta^k=\beta^k-\tau^k \nabla F(\beta^k)$.

\begin{theorem}\label{welldef}
Let $\sigma \in (0,1/2), K\leq T$ in SDARL algorithm, then the data-driven line search rule (\ref{line}) is well-defined.
\end{theorem}

\begin{proof}
See Appendix \ref{thwel}.
\end{proof}

\begin{remark}
The proof process of Theorem \ref{welldef} shows that the establishment of the line search in this paper is very natural and has nothing to do with the parameters of restricted strong convexity and smoothness of the function $F$.
Although a search method of step size is also proposed in \cite{ZPX2021}, it is still difficult to determine the proper value because $M_s$ is generally not easy to compute.
\end{remark}

Moreover, from the calculation format of $\beta^{k+1}$, we can get
$$F(\beta^{k+1})-F(\beta^k)\leq F(\theta^{k}|_{\A^{k}})-F(\beta^{k}) \leq -\sigma \tau^k \|\nabla_{\A^{k}\backslash \A^{k-1}} F(\beta^{k})\|^2.$$
Therefore, the line search guarantees the function values of iterative sequence $\{\beta^k\}$ are in a downward trend which exactly is the state we expect. Then we establish the $\ell_2$ estimation error between the estimator $\beta^k$ and the target regression coefficient $\beta^*$.

\begin{theorem}\label{le4}
Let $\sigma \in (0,1/2), K\leq T$ in SDARL algorithm, then before algorithm terminates, for all $k\geq0$, we have
\begin{align}\label{form3}
\|\beta^k-\beta^*\|\leq \frac{2}{m_{K+T}}\|\nabla F(\beta^*)\| +\sqrt{1+\frac{M_{K+T}}{m_{K+T}}}(\sqrt{\eta_{\tau}})^k\|\beta^0-\beta^*\|,
\end{align}
where $\eta_{\tau}=1-\frac{2 \sigma \tau^k m_{K+T}}{K+1}\in (0,1)$.
\end{theorem}
\begin{proof}
See Appendix \ref{thel2}.
\end{proof}

\begin{remark}
As $\beta^*$ is the target regression coefficient, then the first term $\frac{2}{m_{K+T}}\|\nabla F(\beta^*)\|$ at the right of inequality (\ref{form3}) reflects the noise level which is very small and can't reduce. In linear and generalized linear regression with random design, $\|\nabla F(\beta^*)\|\leq \mathcal{O}(\sqrt{\frac{\log p}{n}})$ holds with high probability \cite{HJK2020,HJL2018}. Here $\eta_{\tau}\in (0,1)$ holds automatically due to the step size determined by the line search. Moreover, the second term $\sqrt{1+\frac{M_{K+T}}{m_{K+T}}}(\sqrt{\eta_{\tau}})^k\|\beta^0-\beta^*\|$ is related to iteration step $k$, and because of $\eta_{\tau}\in (0,1)$, we know that the value of the second term is getting smaller and smaller as $k$ increases. When $k> \log_{\frac{1}{\eta}}\frac{(1+\frac{M_{K+T}}{m_{K+T}})\|\beta^0-\beta^*\|^2}{\delta^2}$, the second term can tend to noise level $\delta$.
\end{remark}

\begin{remark}
With the step size determined by line search, we bound the $\ell_2$ error of iteration sequence $\beta^k$ of SDARL and the target regression coefficient $\beta^*$ without any restrictions on the parameters of restricted strong convexity and smoothness for loss function $F$. However, the $\ell_\infty$ error bound in \cite[Theorem 1]{HJK2020} needs the parameter of restricted strong smoothness of function $F$ to be less than $\frac{1}{T}$, which is not easy to verify in practical problems.

\end{remark}

\begin{remark}
It is worth noting that the sequence of objective function value converges linearly which can be obtained from (\ref{conresult}).
\end{remark}

The following theorem establishes the support recovery property of SDARL algorithm.
\begin{theorem}\label{le5}
Assuming $\sigma \in (0,1/2), K\leq T$, $\beta^0=\textbf{0}$ and $\|\beta_{\A^*}^{*}\|_{\min}\geq \frac{3}{m_{K+T}}\|\nabla F(\beta^*)\|$ in SDARL algorithm, then we have
\begin{align}\label{form4}
\A^*\subseteq \A^k,
\end{align}
if $k>\log_{\frac{1}{\eta_{\tau}}}9(1+\frac{M_{K+T}}{m_{K+T}})\frac{\|\beta^*\|^2}{\|\beta_{\A^*}^{*}\|_{\min}^2}$.
\end{theorem}

\begin{proof}
See Appendix \ref{thsupp}.
\end{proof}
\begin{remark}
The condition $\|\beta_{\A^*}^{*}\|_{\min}\geq \frac{3}{m_{K+T}}\|\nabla F(\beta^*)\|$ is required for the target $\beta^*$ to be detectable \cite{M2019}.
\end{remark}

\section{Numerical results}\label{mumres}
\setlength\parindent{2em}

In this section, we provide several numerical examples to highlight the effectiveness and superiority of SDARL and ASDARL algorithms. These examples are linear and logistic regression problems. First we compare SDARL with GSDAR to show the advantage of the line search technique. Then we verify the superiority of the proposed algorithms by comparing with LASSO, MCP and SCAD methods which are implemented by the R package ``ncvreg" \cite{BH2011}. We use 10-fold cross validation to select a $\hat{\beta}$ from its output for comparing.

In terms of numerical comparison, we consider some commonly used indicators, such as the average relative error $\text{ARE}=\frac{1}{100}\sum \frac{||\hat{\beta}-\beta^*||}{||\beta^*||}$, average positive discovery rate $\text{APDR}=\frac{1}{100}\sum \frac{|\hat{\A}\bigcap \A^*|}{|\A^*|}$, average false discovery rate $\text{AFDR}=\frac{1}{100}\sum \frac{|\hat{\A}\bigcap \A^{*c}|}{|\hat{\A}|}$ and average combined discovery rate $\text{ACDR}=\text{APDR}+(1-\text{AFDR})$ \cite{HJK2020, LC2014}. The selection consistency of a algorithm means that $APDR=1, AFDR=0$ and $ACDR=2$. In addition, since the purpose of logistic regression problem is to classify, so we also consider average classification accuracy rate (ACAR) for comparison. We uniformly set the parameters in line search as $\nu=0.9$, $\sigma=0.1$, and other parameters will be given according to specific issues.

\subsection{Linear regression}

For the linear regression problems, we here consider the function $F(\beta)$ as the least square estimation function, i.e., $F(\beta)=\frac{1}{2n}\|X \beta+\beta_0-y\|^2$, where $X\in \mathbb{R}^{n\times p}$ is the design matrix with $\sqrt{n}$-normalized columns, $\beta_0$ is an $n$-dimensional intercept item with each component being 1, $y$ is the response variable. We next use an illustrative example to analyze the effectiveness of the proposed line search.

\subsubsection{An illustrative example}\label{lineexa}

In this simulation example, we first generate an $n\times p$ matrix $\bar{X}$ whose rows are drawn independently from $\mathcal{N}(0,\Sigma)$ with $\Sigma_{jk}=\rho^{|j-k|}, 1\leq j,k \leq p$, and then obtain coefficient matrix $X$ by normalizing its columns to the $\sqrt{n}$ length. In order to generate the target regression coefficient $\beta^*\in \mathbb{R}^{p}$, we randomly select a subset of $\{1,\cdots,p\}$ to form the active set $\A^*$ with $|\A^*|=K<n$. Let $R=m_2 / m_1$, where $m_2=\|\beta^*_{\A^*}\|_{max}$ and $m_1=\|\beta^*_{\A^*}\|_{min}=1$. Then the $K$ nonzero coefficients in $\beta^*$ are uniformly distributed in $(m_1,m_2)$. The response variable is generated by $y=X\beta^*+\beta_0+\varepsilon$ where $\varepsilon \in \mathbb{R}^{n}$ is the additive Gaussian noise and generated independently from $\mathcal{N}(0,\sigma_1^2)$. We consider the problem setting of $n=500$, $p=1000$, $\rho=0.2$, $\sigma_1=1$, $R=100$.

Firstly, we analyze the advantage of the line search designed in this paper by comparing the SDARL with its version of a fixed step size $\tau=1$, i.e., SDAR in \cite{HJL2018}. Based on 100 independent replications, the calculation results of the relative error, positive discovery rate, false discovery rate and combined discovery rate of this two algorithms are revealed in Figure \ref{fig:1}. It can be seen directly that, the SDARL algorithm has higher regression accuracy, more accurate variable selection and more stable performance in comparison with SDAR. This phenomenon is sufficient to illustrate the necessity and effectiveness of the line search proposed in this paper for linear regression problems.

\begin{figure}
\centering
\includegraphics[width=0.45\textwidth,height=4cm]{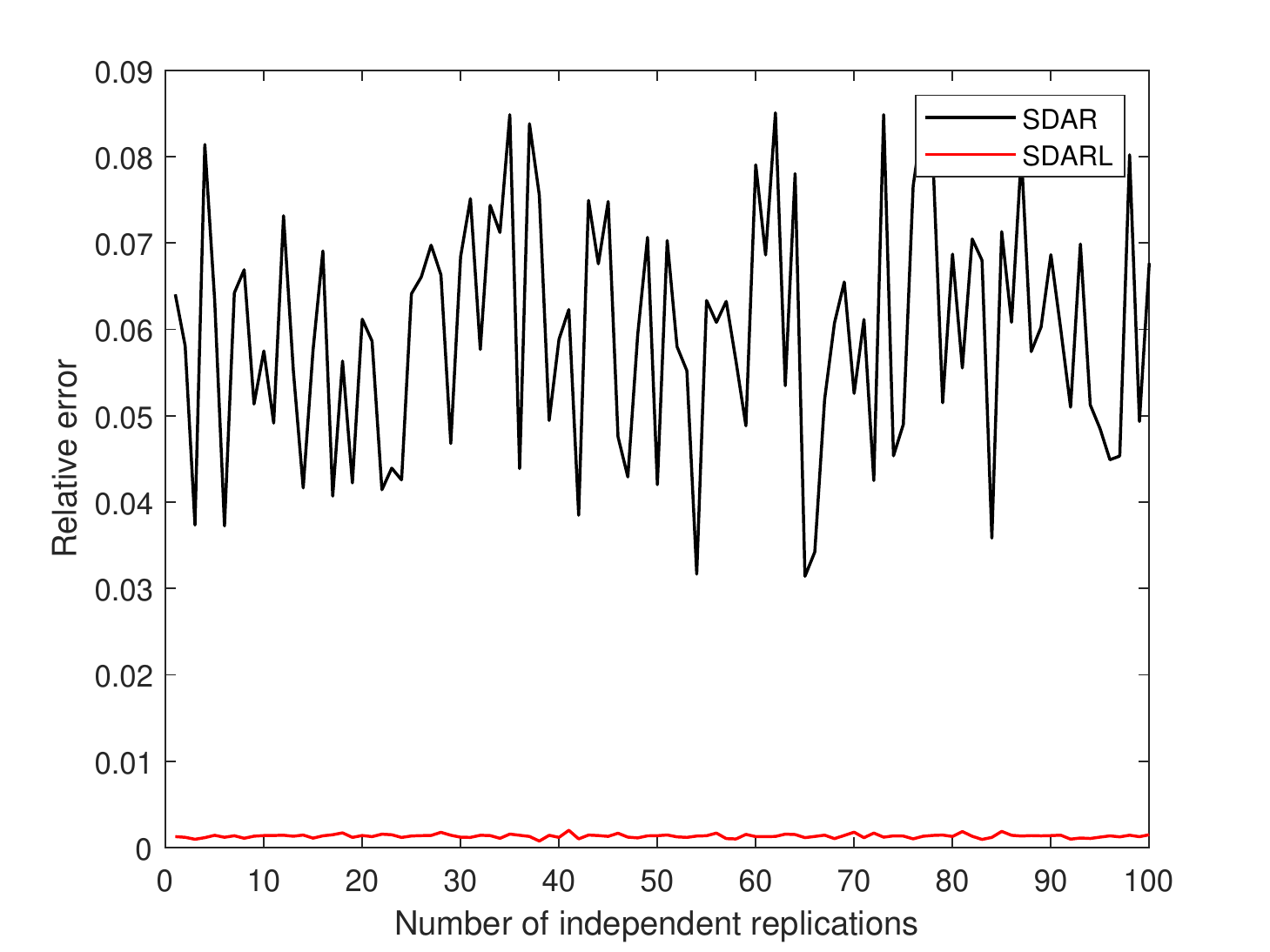}\hspace{-.4cm}
\includegraphics[width=0.45\textwidth,height=4cm]{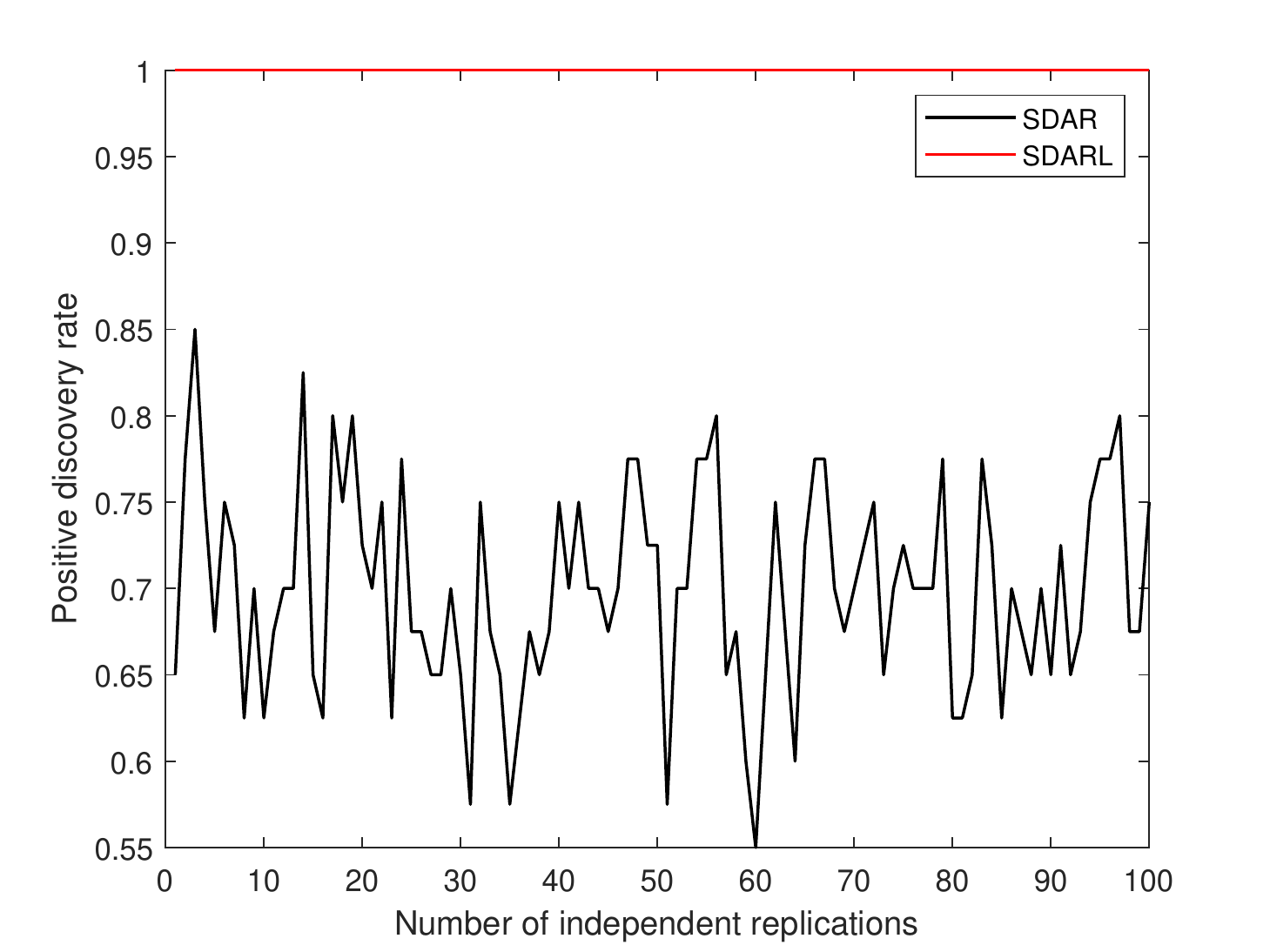}\\
\includegraphics[width=0.45\textwidth,height=4cm]{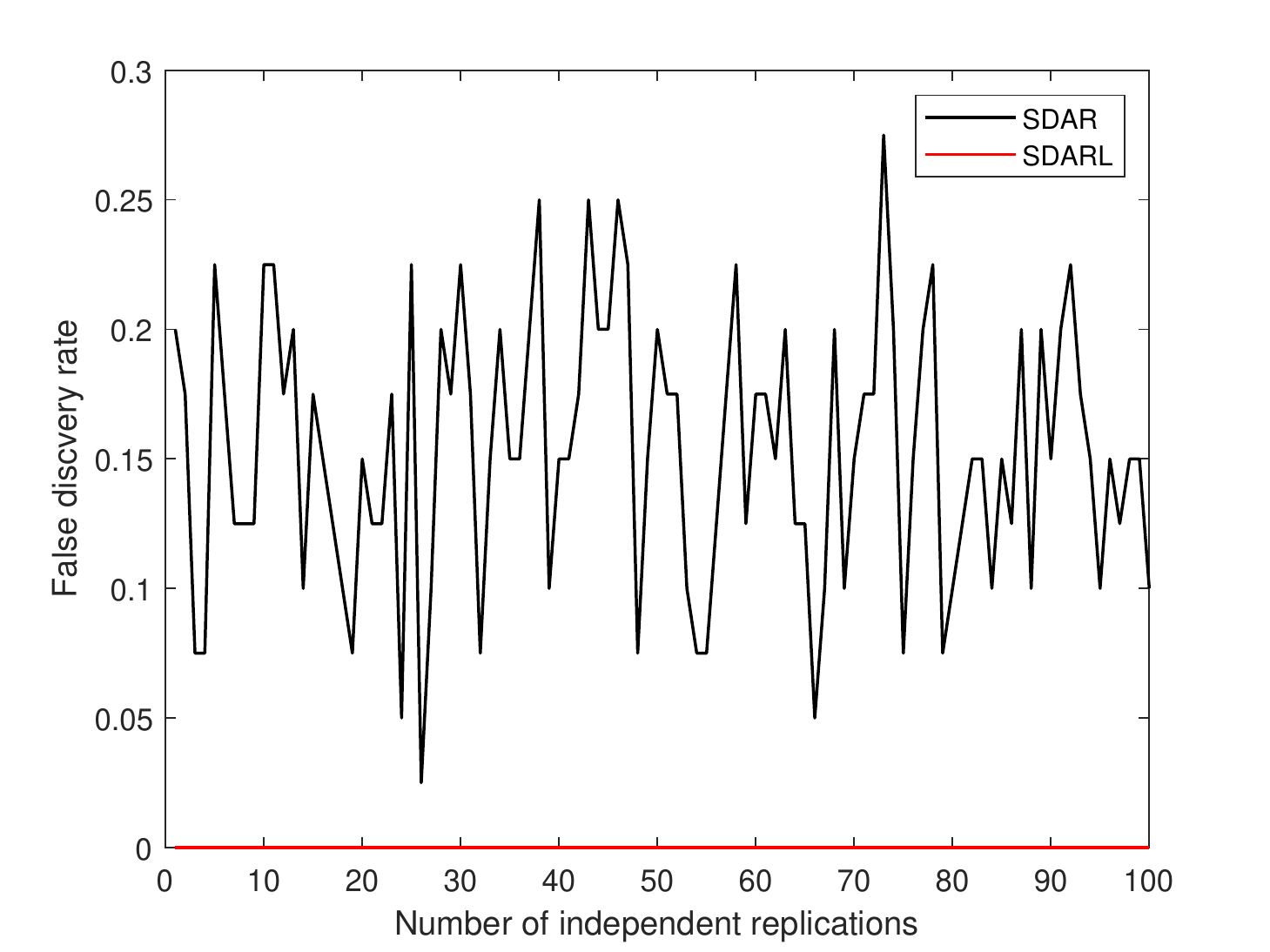}\hspace{-.4cm}
\includegraphics[width=0.45\textwidth,height=4cm]{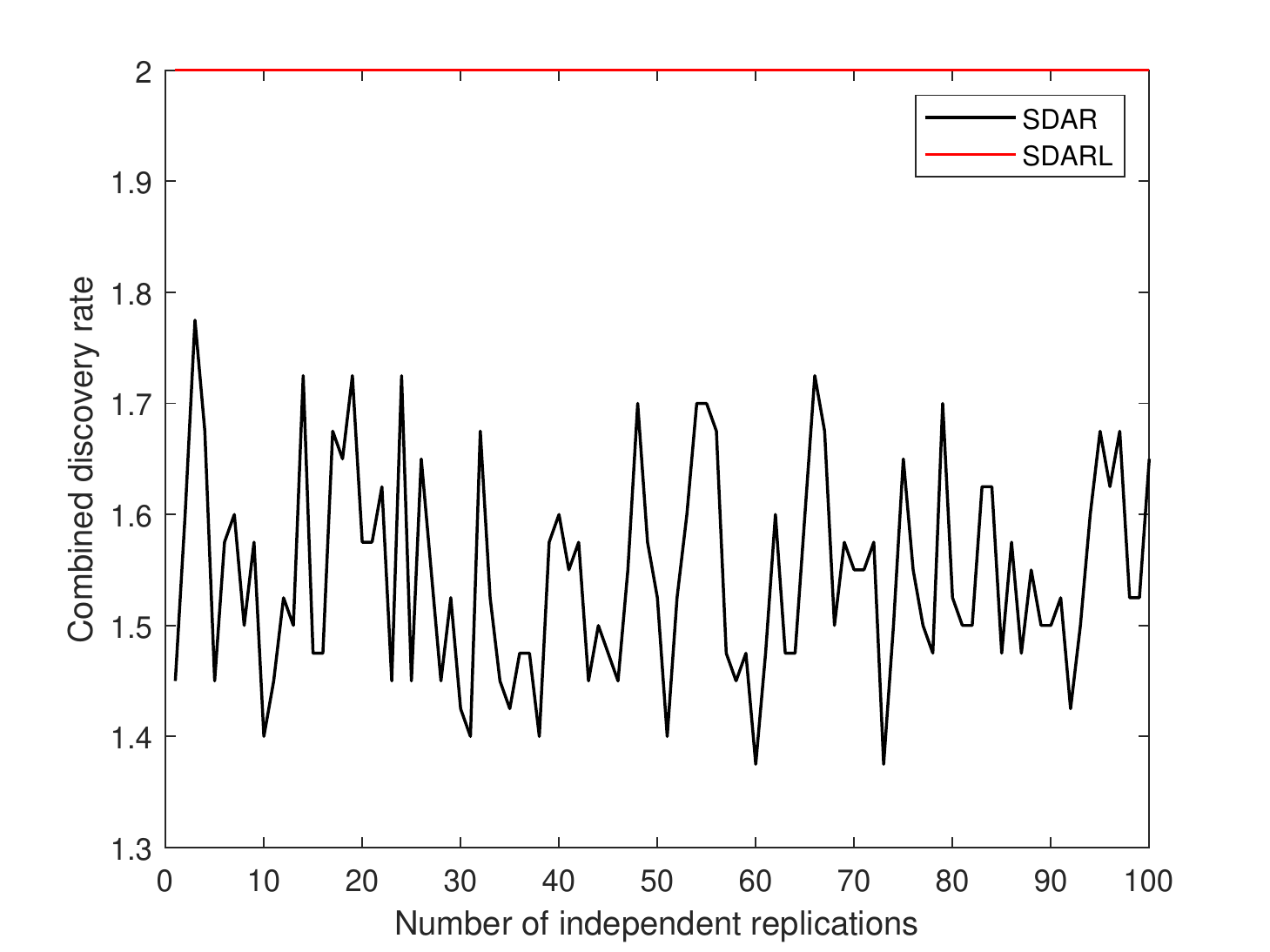}
\caption{{\small The comparison between SDARL and its version with a fixed step size $\tau=1$ in linear regression problems}}
\label{fig:1}
\end{figure}

In addition, we show the effectiveness of SDARL algorithm based on the average relative error and average number of iterations of 100 independent experiments with different sparsity levels $K=5:5:50$. We also consider the influence of the correlation level $\rho=0.2:0.3:0.8$, and the specific results are shown in Figure \ref{fig:2}. Firstly, from the values of the average relative error, it can be seen that the estimated coefficient $\hat{\beta}$ and the target coefficient $\beta^*$ have a very high degree of fusion. Moreover, the average number of iterations of SDARL is almost on the rise as the sparsity level increase from $5$ to $50$ for every $\rho$. But the average number of iterations only reaches about 5 when $\rho=0.2$ and $0.5$, and does not exceed 10 when $\rho=0.8$. This phenomenon fully illustrates the convergence speed of SDARL algorithm is fast.
%\begin{figure}
%\centering
%\includegraphics[width=0.5\textwidth]{linexp.eps}\hspace{-.2cm}
%\includegraphics[width=0.5\textwidth]{iter1.eps}
%\caption{{\small From left to right: comparison between target regression coefficient and estimated result, the average number of iterations of SDARL for linear regression as $K$ increases.}}
%\label{fig:2}
%\end{figure}

\begin{figure}
\centering
\includegraphics[width=0.5\textwidth,height=5cm]{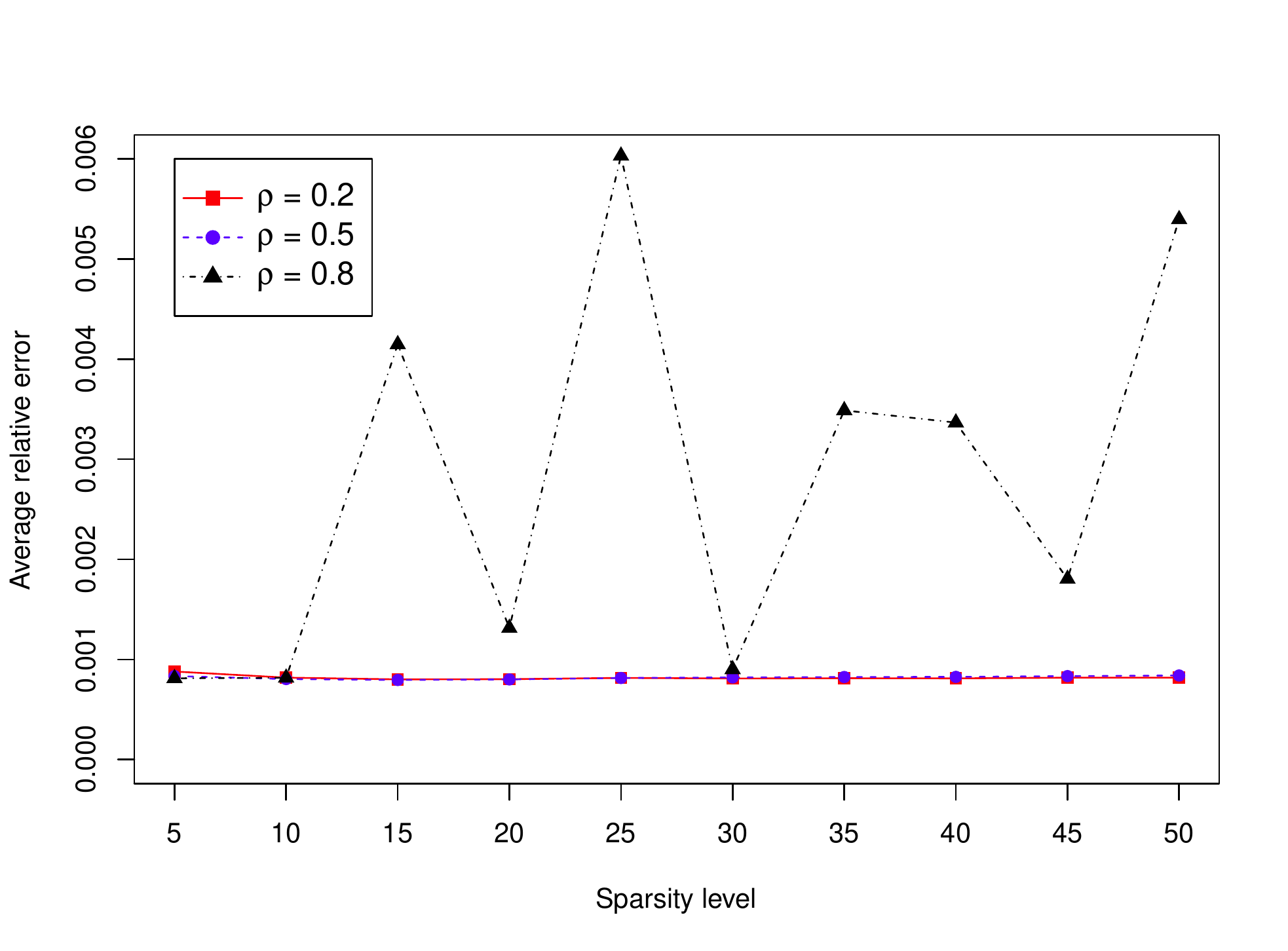}\hspace{-.2cm}
\includegraphics[width=0.5\textwidth,height=5cm]{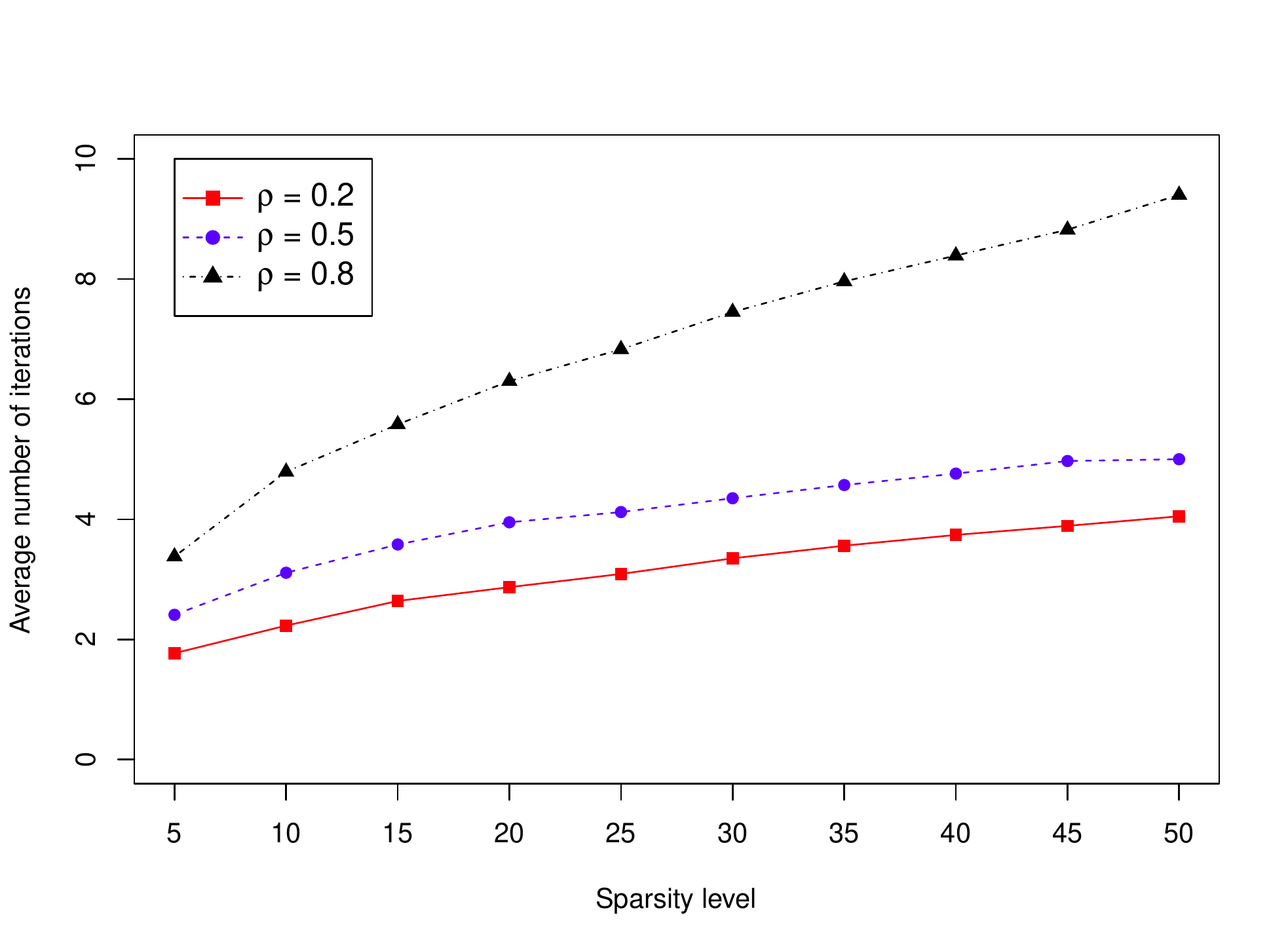}
\caption{{\small From left to right: the average relative error and average number of iterations of SDARL for linear regression as sparsity level $K$ increases.}}
\label{fig:2}
\end{figure}

\subsubsection{Influence of the model parameters}

In this part, we consider the influence of model parameters $\{n,p,K,\rho\}$ on APDR, AFDR and ACDR of ASDARL, LASSO, MCP and SCAD methods. We use the same method of data generation as the previous section. All simulation experiments in this subsection are based on 10 independent replications. The specific values of the parameters are shown below the corresponding figure. The simulation results are given in Figure \ref{fig:3}- Figure \ref{fig:6}. It can be seen that ASDARL can always have the largest values on APDR and ACDR, and have the least values on AFDR with the changing of each considered parameter. This phenomena fully illustrates that ASDARL algorithm is generally more accurate, more efficient and more stable than LASSO, MCP and SCAD.

\begin{figure}
\centering
\includegraphics[width=0.34\textwidth,height=4cm]{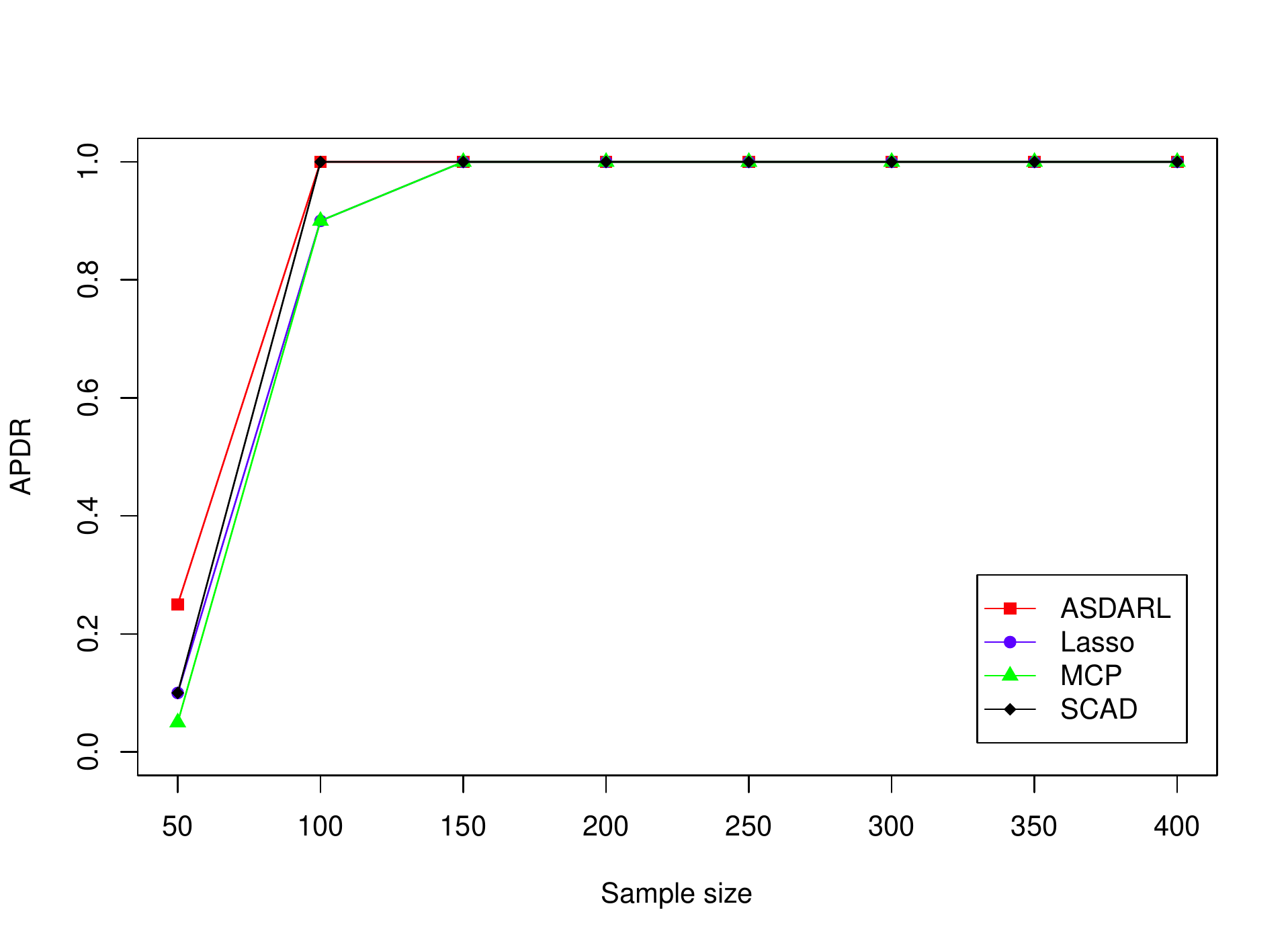}\hspace{-.3cm}
\includegraphics[width=0.34\textwidth,height=4cm]{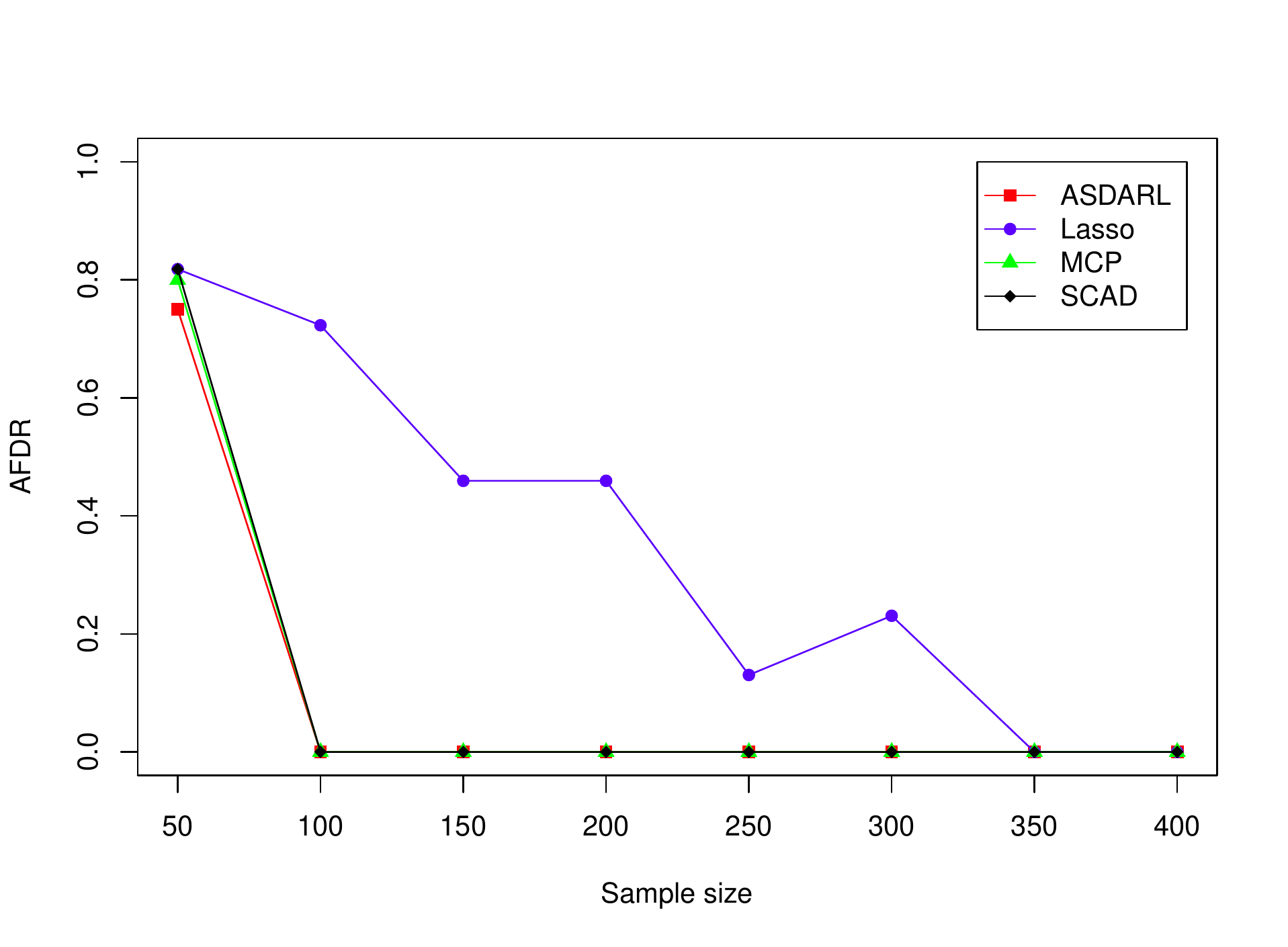}\hspace{-.3cm}
\includegraphics[width=0.34\textwidth,height=4cm]{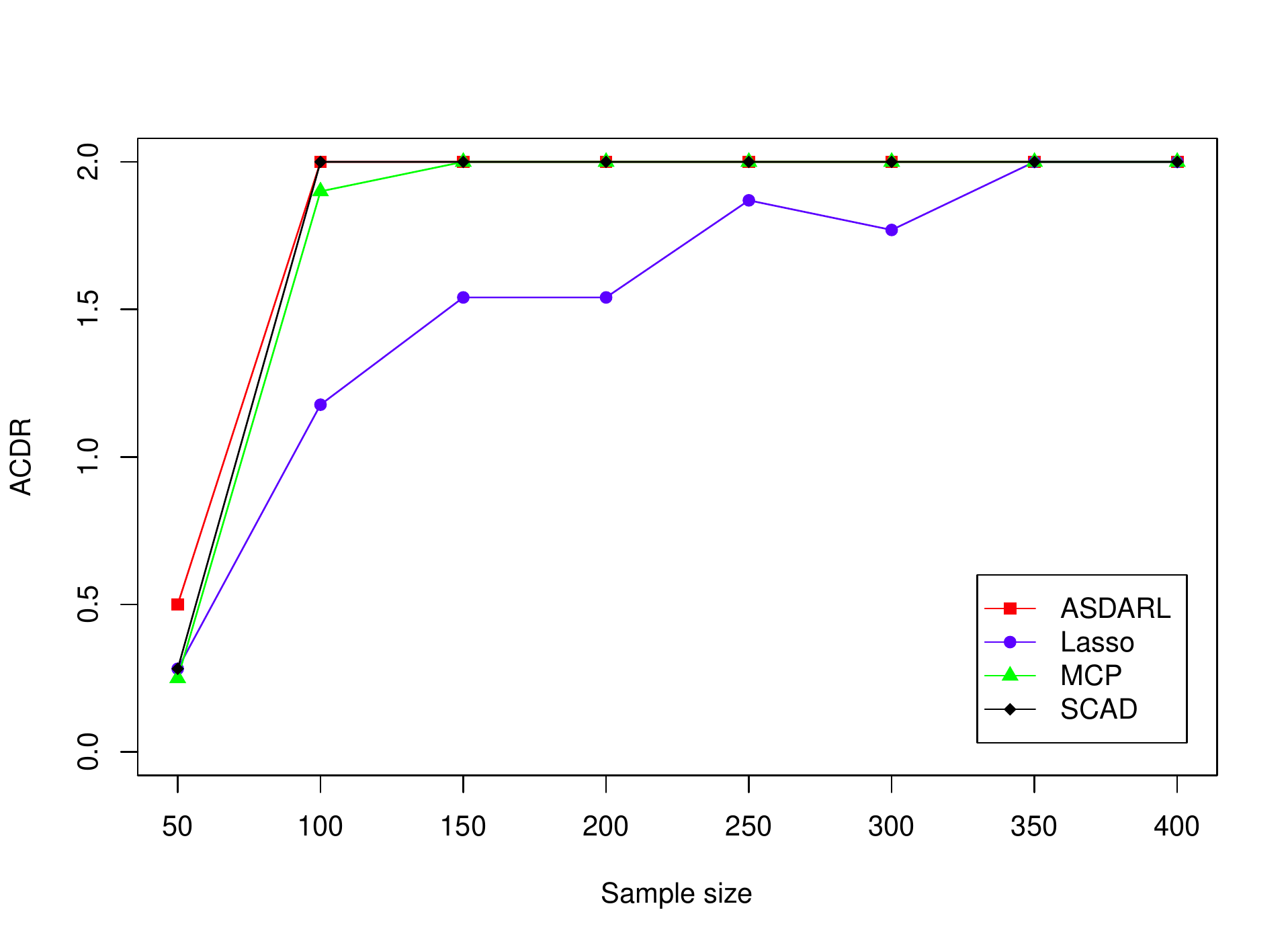}
\caption{{\small Numerical results (ARDR, AFDR, ACDR) of the influence of sample size in linear regression problems with $n=50:50:400$, $p=1000$, $K=20$, $R=5$, $\rho=0.2$.}}
\label{fig:3}
\end{figure}

\begin{figure}
\centering
\includegraphics[width=0.34\textwidth,height=4cm]{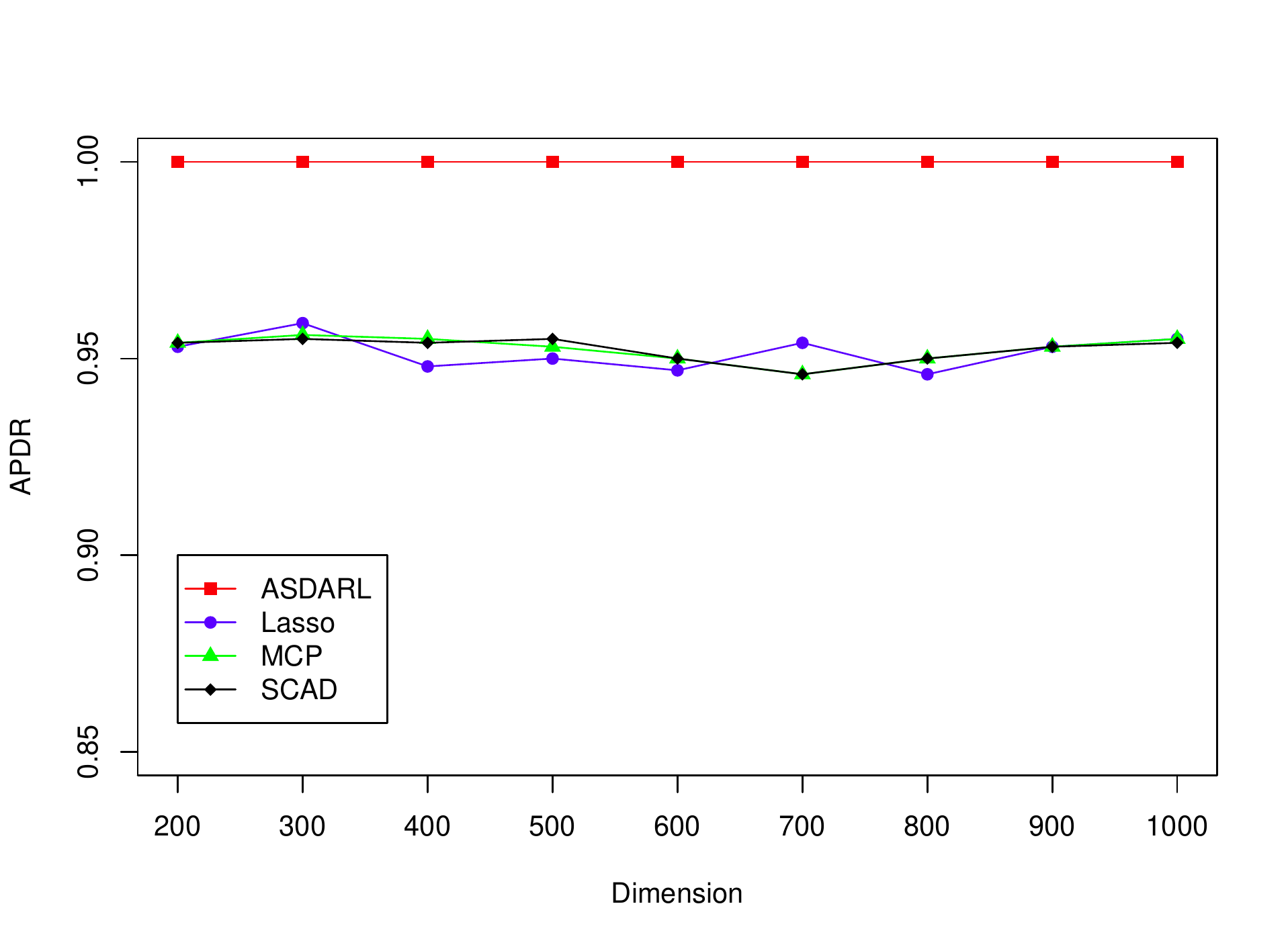}\hspace{-.3cm}
\includegraphics[width=0.34\textwidth,height=4cm]{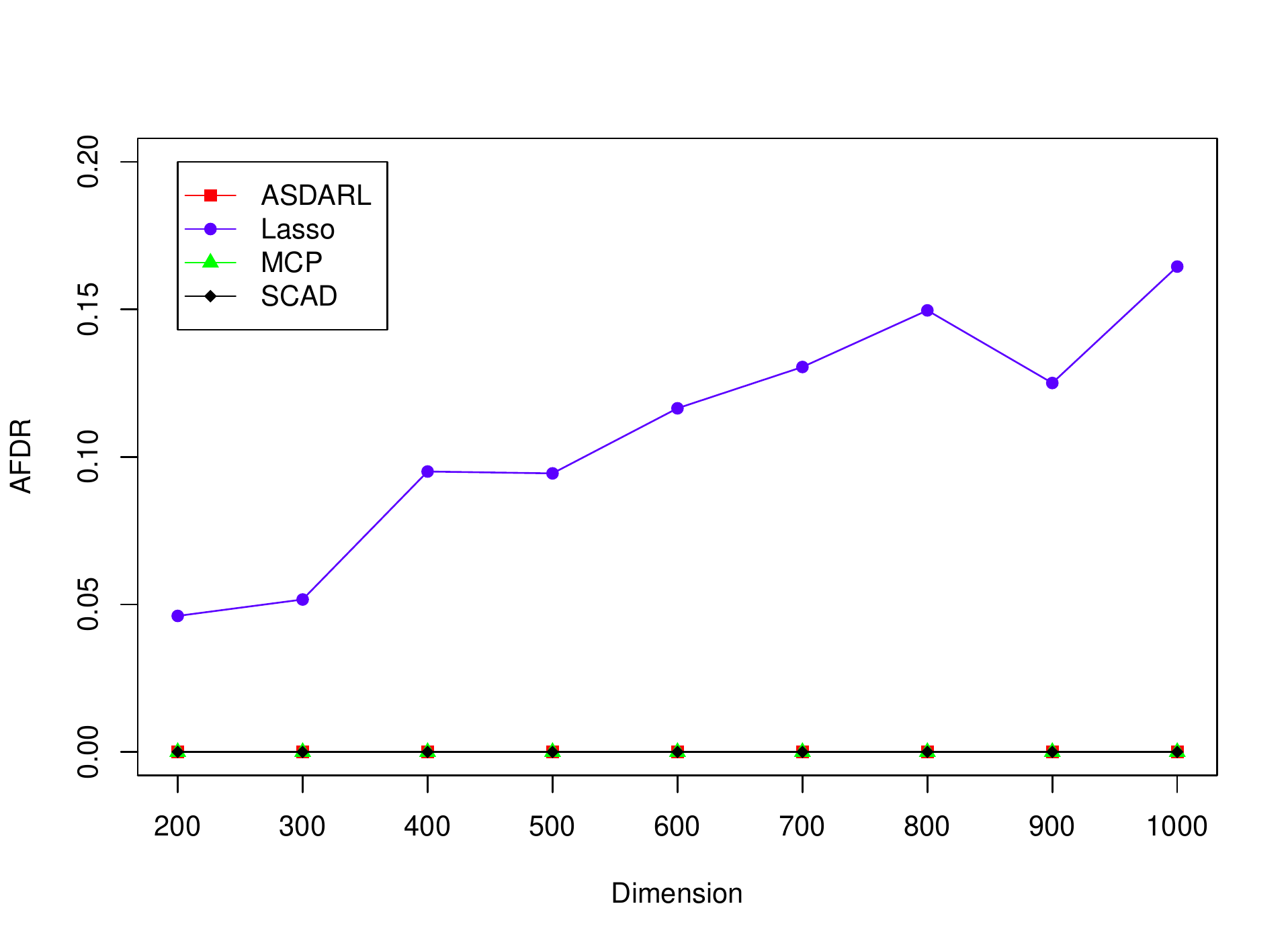}\hspace{-.3cm}
\includegraphics[width=0.34\textwidth,height=4cm]{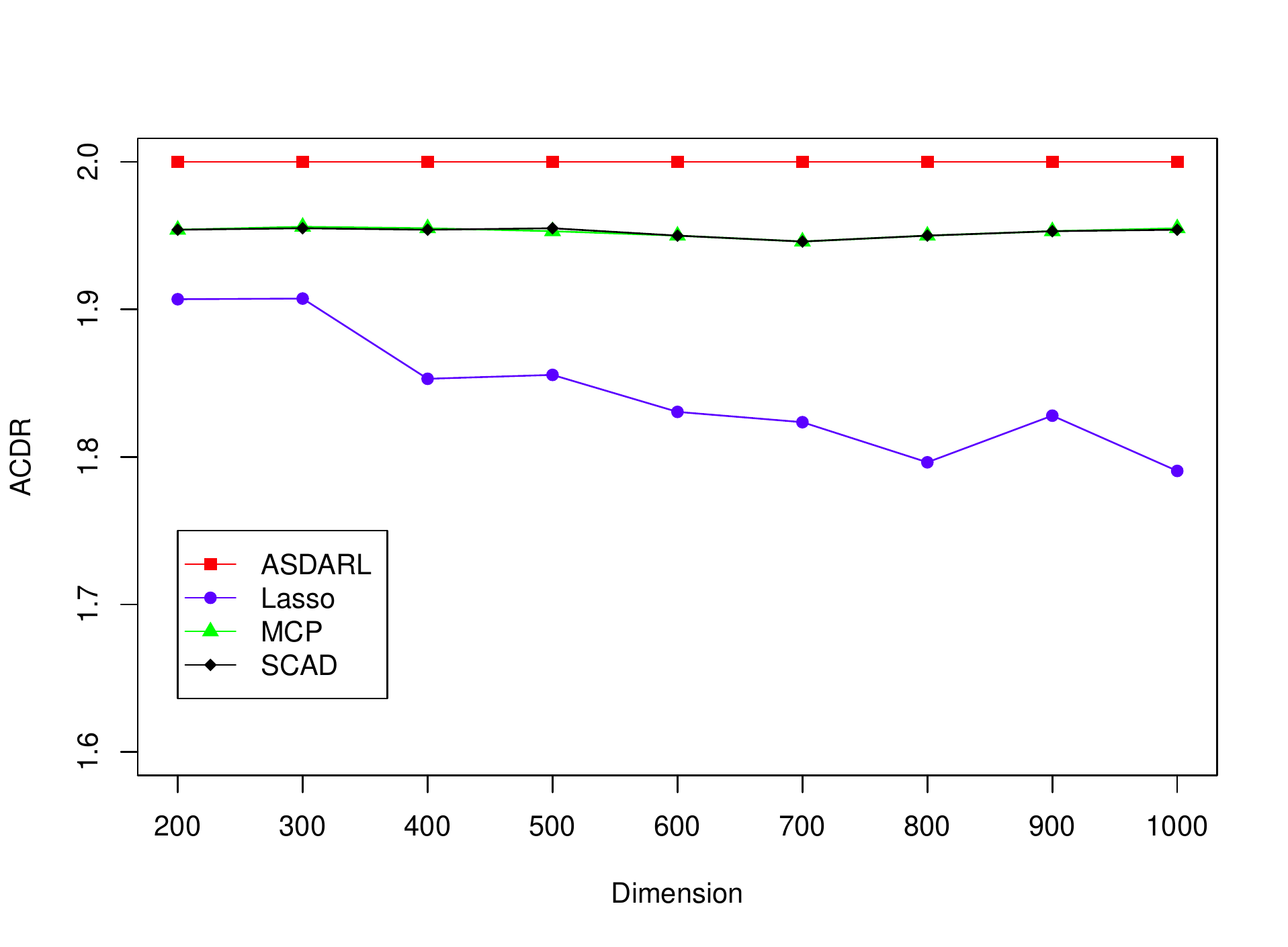}
\caption{{\small Numerical results (ARDR, AFDR, ACDR) of the influence of dimension in linear regression problems with $n=100$, $p=200:100:1000$, $K=10$, $R=100$, $\rho=0.2$.}}
\label{fig:4}
\end{figure}

\begin{figure}
\centering
\includegraphics[width=0.34\textwidth,height=4cm]{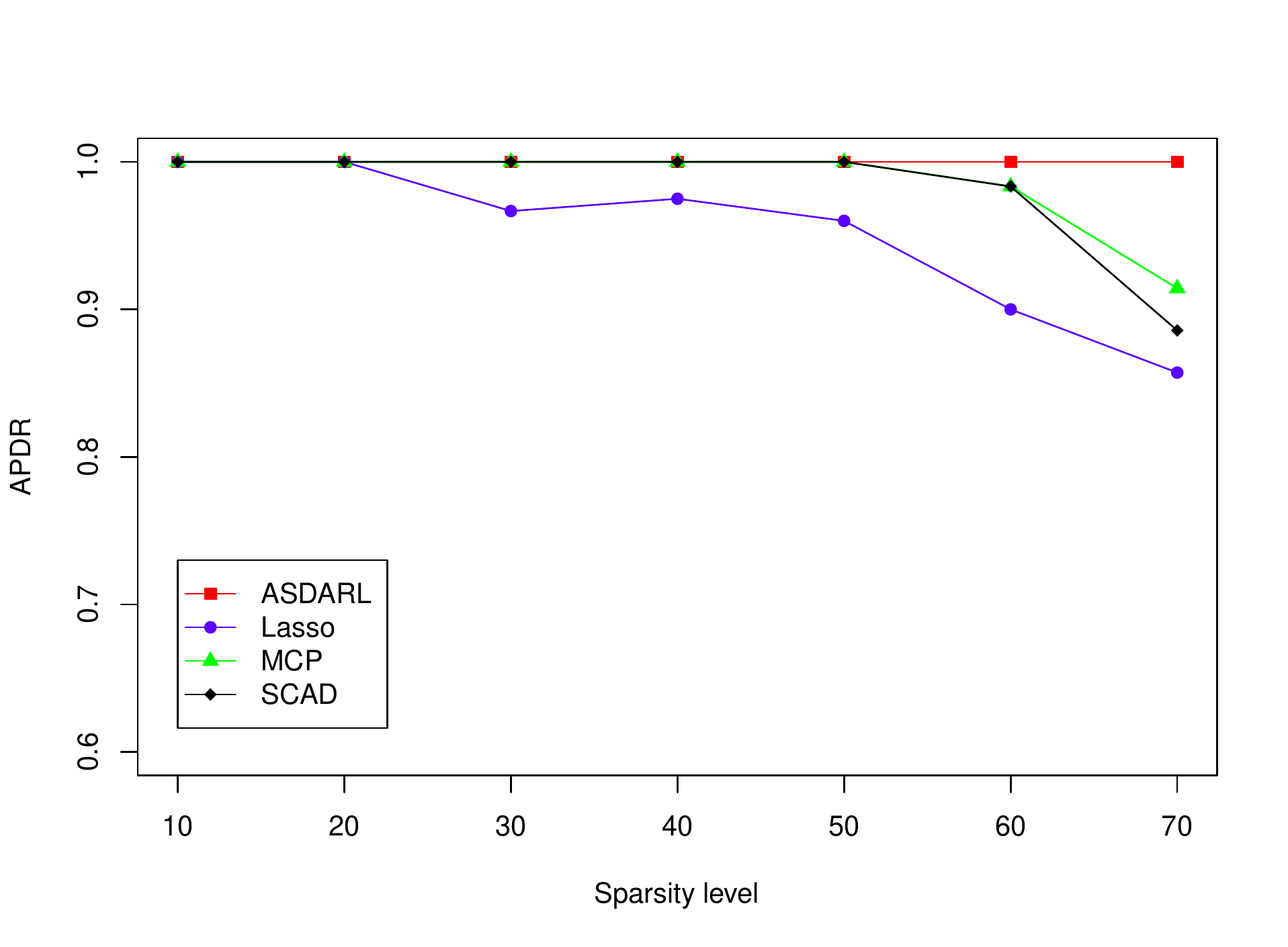}\hspace{-.3cm}
\includegraphics[width=0.34\textwidth,height=4cm]{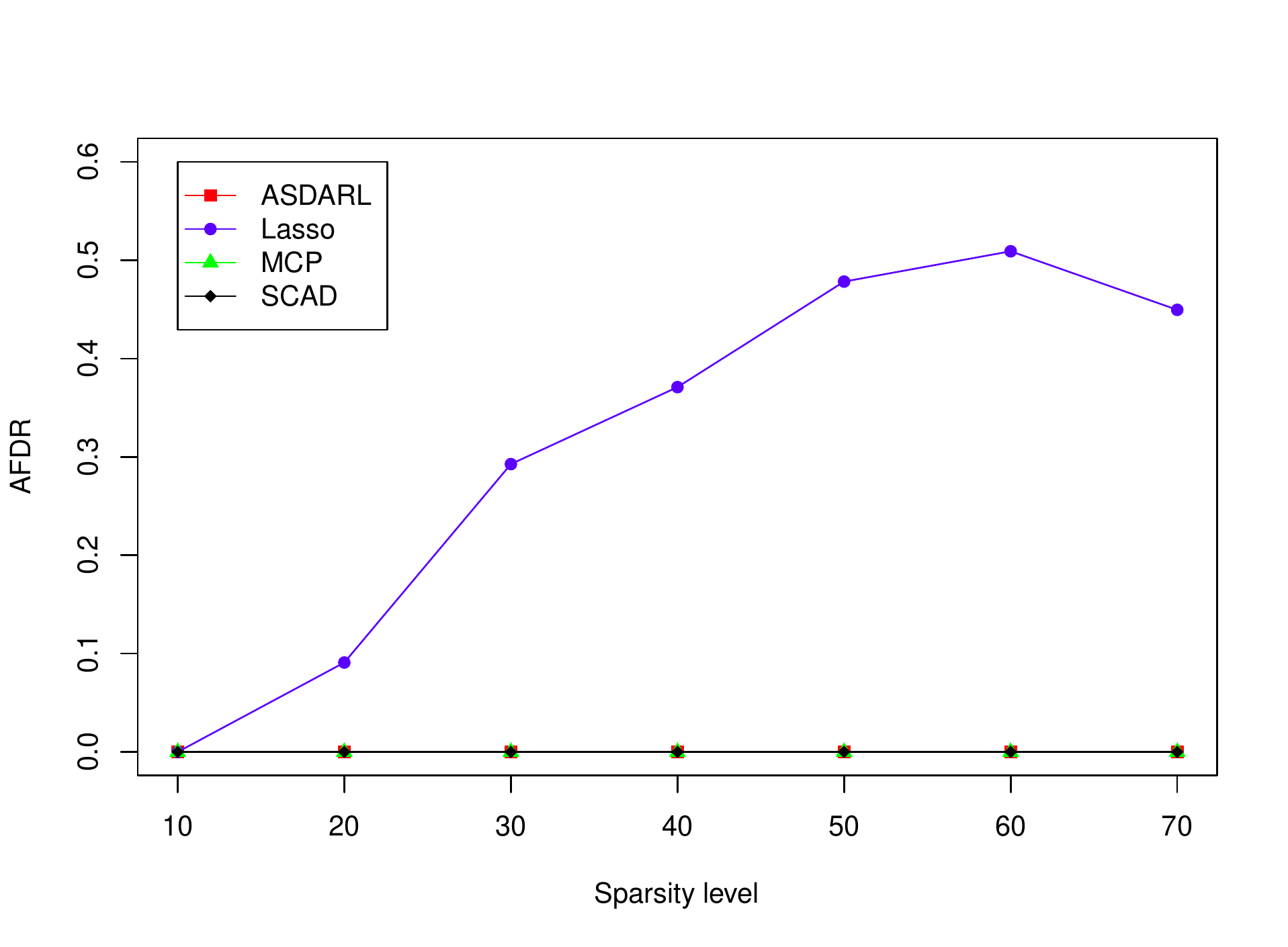}\hspace{-.3cm}
\includegraphics[width=0.34\textwidth,height=4cm]{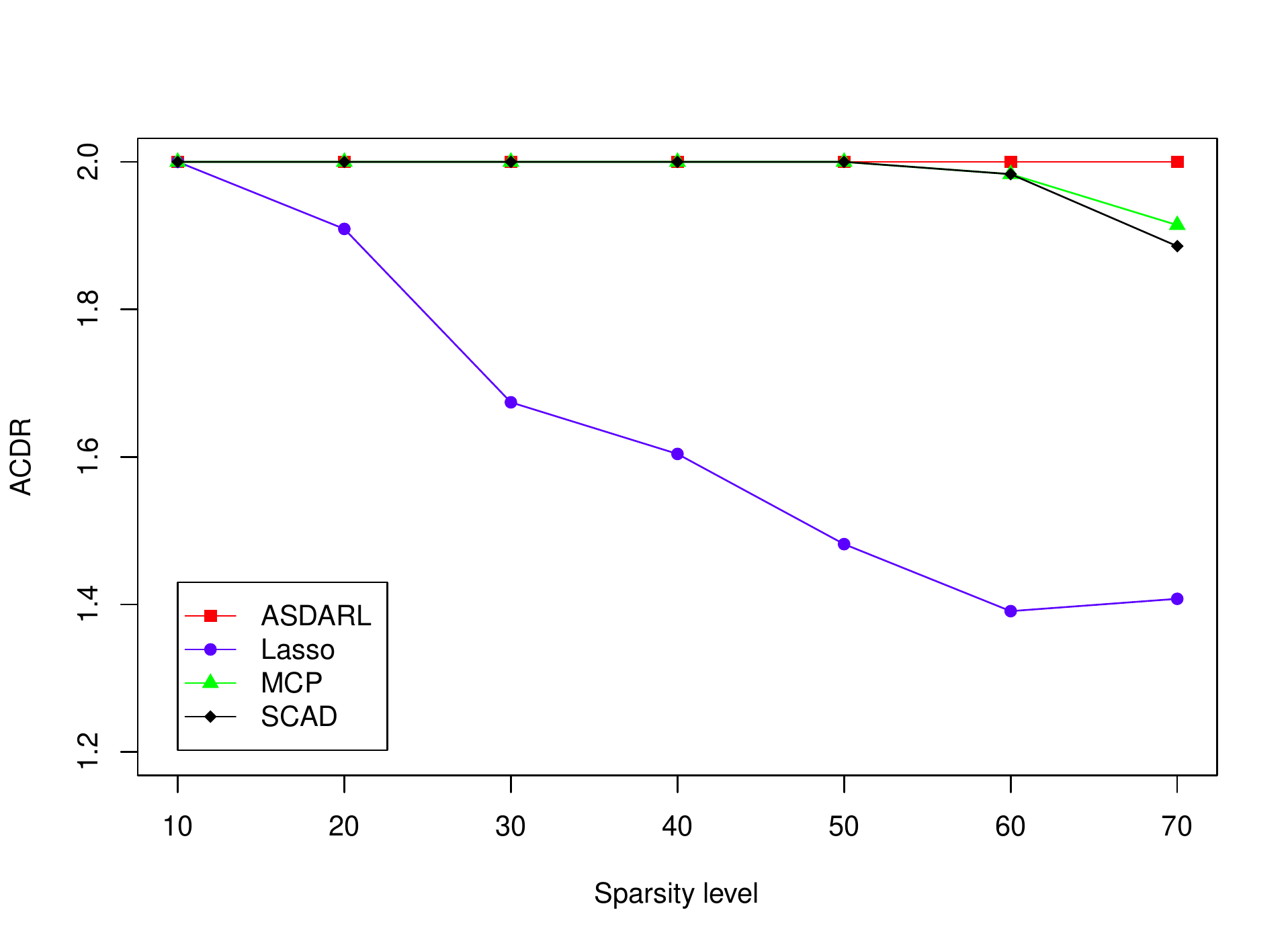}
\caption{{\small Numerical results (ARDR, AFDR, ACDR) of the influence of sparsity level in linear regression problems with $n=200$, $p=500$, $K=10:10:70$, $R=10$, $\rho=0.2$.}}
\label{fig:5}
\end{figure}

\begin{figure}
\centering
\includegraphics[width=0.34\textwidth,height=4cm]{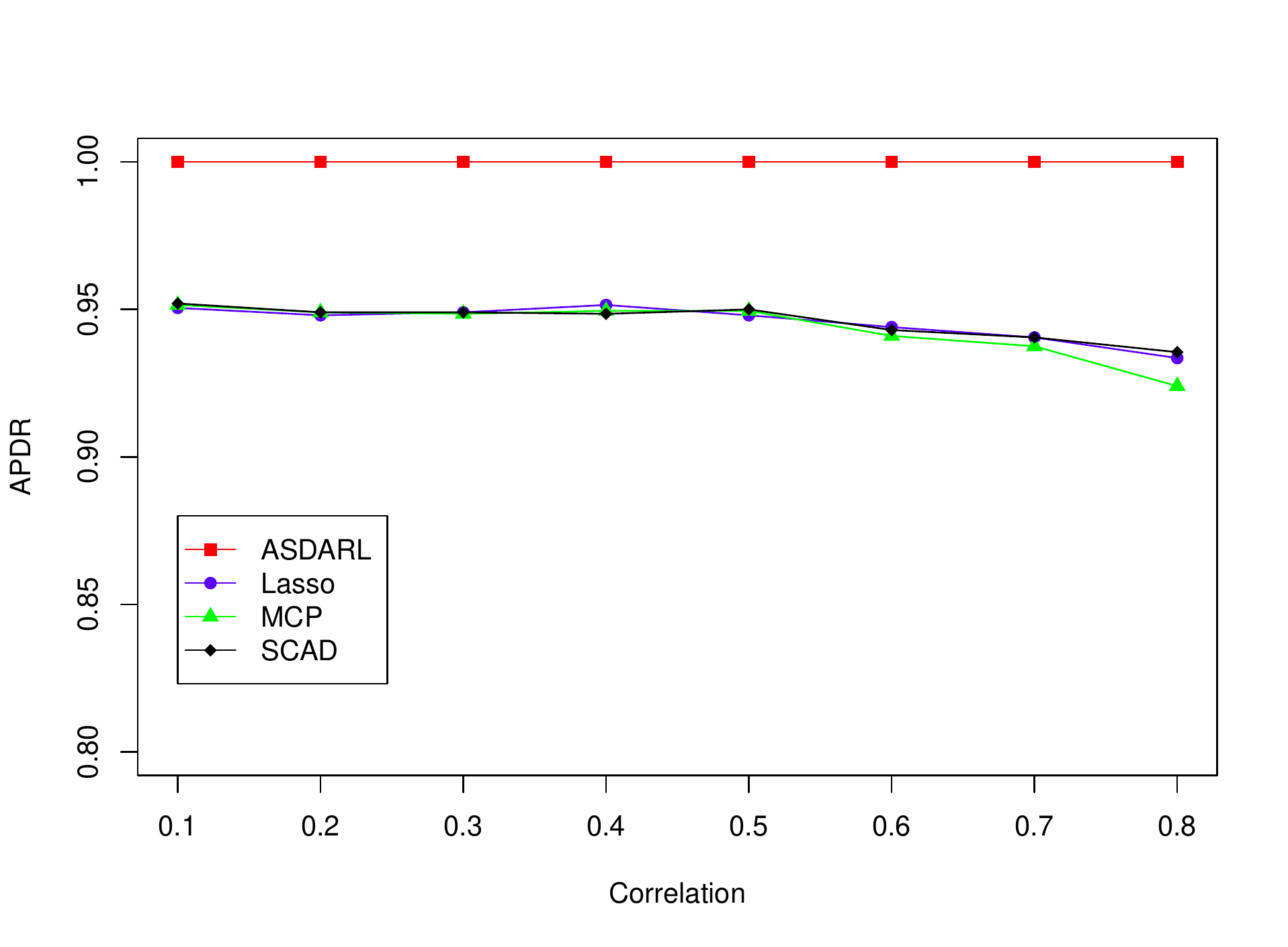}\hspace{-.3cm}
\includegraphics[width=0.34\textwidth,height=4cm]{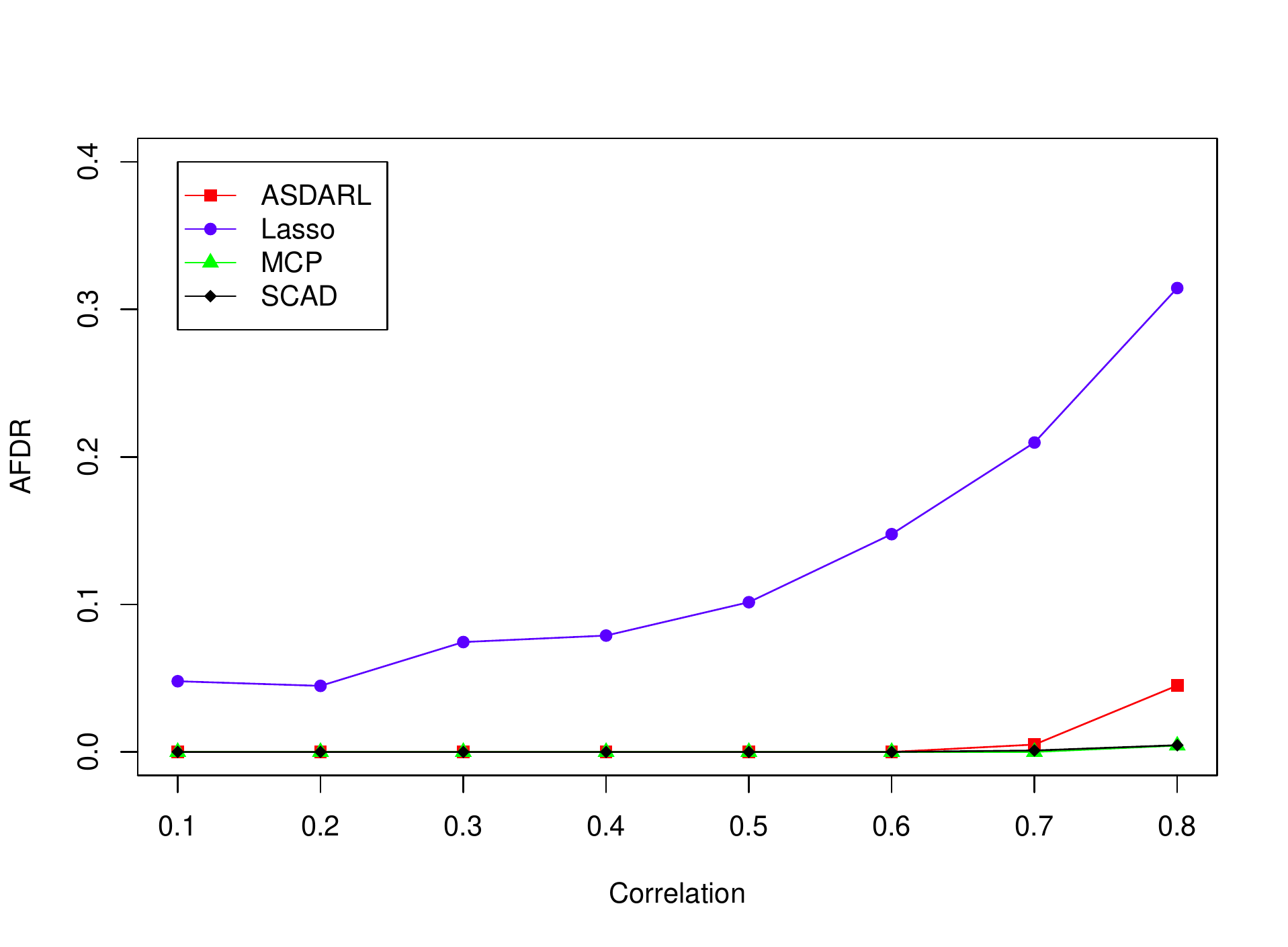}\hspace{-.3cm}
\includegraphics[width=0.34\textwidth,height=4cm]{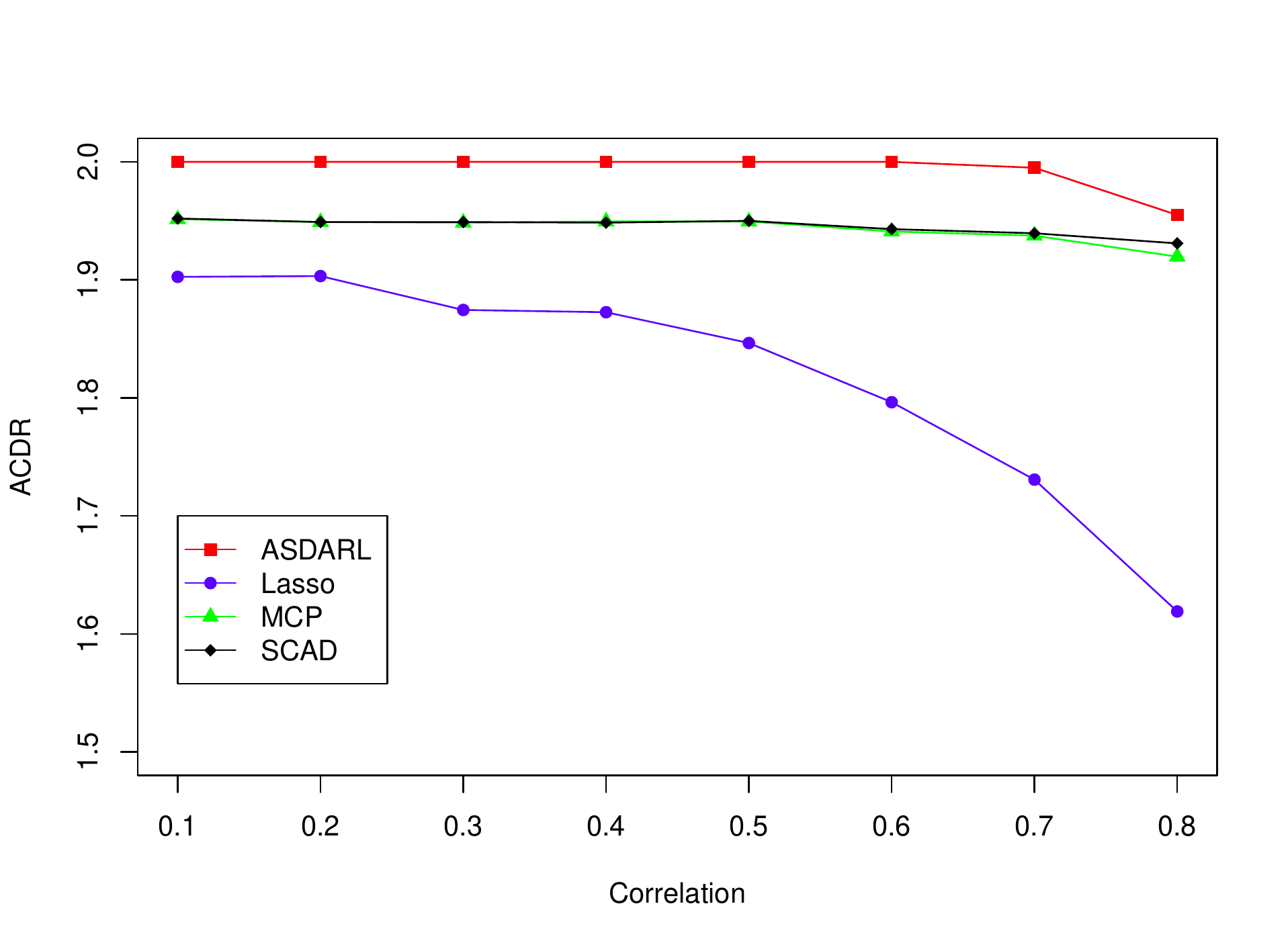}
\caption{{\small Numerical results (ARDR, AFDR, ACDR) of the influence of correlation in linear regression problems with $n=200$, $p=500$, $K=20$, $R=100$, $\rho=0.1:0.1:0.8$.}}
\label{fig:6}
\end{figure}

\subsubsection{Numerical comparison}\label{linecom}

In this subsection, we use random synthetic data to compare the accuracy and efficiency of SDARL, ASDARL, LASSO, MCP and SCAD methods. For the design matrix $X$, we first generate an $n\times p$ random Gaussian matrix $\bar{X}$ whose entries are i.i.d. $\mathcal{N}(0,1)$, and then normalize its columns to the $\sqrt{n}$ length. Then $X$ is generated with $X_1=\bar{X}_1$, $X_p=\bar{X}_p$ and $X_{j}=\bar{X}_j+\rho (\bar{X}_{j+1}+\bar{X}_{j-1}), j=2, \cdots, p-1$. The $K$ nonzero elements of underlying regression coefficient $\beta^*$ are uniformly distributed in $[m_1, m_2]$, where $m_1=5 \sqrt{2\log(p)/n}$ and $m_2=R*m_1$. The $K$ nonzero elements are randomly assigned to $K$ components of $\beta^*$. Then the response variable is generated by $y=X\beta^*+\beta_0+\varepsilon$. Here we consider the problem setting of $n = 800$, $p = 5000$, $K = 100$, $\sigma_1=1$, $R=100$. We run ASDARL with $\alpha=50$. We also take the influence of matrix correlation coefficient $\rho$ into account and set $\rho=0.2:0.3:0.8$. Based on 100 independent replications, we obtain the specific values of ARE, CPU time ($\text{Time(s)}$), APDR, AFDR and ACDR in Table \ref{table1}. The standard deviations of \text{ARE} and \text{Time(s)} are shown in the corresponding parentheses. For convenience, we mark the numbers in boldface to indicate the best performers (the tables below are the same).

\begin{table}[h]
\renewcommand\arraystretch{1.5}
\centering {
\setlength{\tabcolsep}{4mm}{
\begin{tabular}{|c||c|c|c|c|c|}
\Xhline{1.2pt}
$\rho$ & Algorithm & ARE  & Time(s)& (APDR, AFDR, ACDR)\\
\Xhline{0.8pt}
\multirow{5}{*}{0.2} & LASSO & 1.41e-1 (1.59e-2) & 6.35 (1.56e-1) & (0.9334, 0.2640, 1.6694) \\
& MCP & 3.64e-2 (7.26e-3)  & 6.51 (1.39e-1) & (0.9363, \textbf{0.0000}, 1.9363) \\
& SCAD & 5.03e-2 (9.62e-3) & 6.54 (1.48e-1) & (0.9356, \textbf{0.0000}, 1.9356) \\
& SDARL & \textbf{8.62e-4 (7.90e-5)} & \textbf{4.35 (4.72e-1)} & (\textbf{1.0000}, \textbf{0.0000}, \textbf{2.0000}) \\
& ASDARL & \textbf{8.62e-4 (7.90e-5)}  & 17.60 (1.52+0) & (\textbf{1.0000}, \textbf{0.0000}, \textbf{2.0000}) \\
\Xhline{0.8pt}
\multirow{5}{*}{0.5}  & LASSO & 1.61e-1 (2.27e-2) & 6.61 (1.66e-1)  & (0.9223, 0.3298, 1.5925) \\
& MCP & 1.56e-1 (1.24e-1)  & 6.45 (1.87e-1) & (0.9066, 0.0371, 1.8695) \\
& SCAD &  1.15e-1 (8.88e-2) & 6.79 (2.51e-1) & (0.9191, \textbf{0.0216}, 1.8975) \\
& SDARL & 7.50e-2 (9.92e-2) & \textbf{6.31 (7.42e-1)} & (0.9721, 0.0279, \textbf{1.9442}) \\
& ASDARL & \textbf{4.14e-3 (2.89e-2)}  & 22.89 (1.90+0) & (\textbf{0.9997}, 0.1502, 1.8495) \\
\Xhline{0.8pt}
\multirow{5}{*}{0.8}  & LASSO & 1.64e-1 (2.19e-2) & \textbf{6.82 (3.07e-1)}  & (0.9192, 0.3420, 1.5772) \\
& MCP & 5.53e-2 (2.68e-2)  & 6.56 (2.08e-2) & (0.9207, 0.0068, 1.9139) \\
& SCAD & 6.18e-2 (1.89e-2)   & 7.17 (4.52e-1) & (0.9268, 0.0056, 1.9212) \\
& SDARL & 8.21e-3 (2.66e-2)  & 7.46 (7.85e-1) & (0.9952, \textbf{0.0048}, \textbf{1.9904}) \\
& ASDARL & \textbf{7.81e-4 (3.47e-4)}  & 25.72 (2.36e+0) & (\textbf{1.0000}, 0.0700, 1.9300) \\
\Xhline{1.2pt}
\end{tabular}}
\small\caption{Numerical comparison in linear regression with $n = 800$, $p = 5000$, $K = 100$, $R=100$ and $\rho=0.2:0.3:0.8$.}
\label{table1}
}
\end{table}

From Table \ref{table1}, we can conclude that ASDARL can always get the smallest values of $\text{ARE}$ for these correlation coefficients $\rho$, and the relative errors of SDARL are also significantly smaller than those of LASSO, MCP and SCAD. In comparison, the SDARL algorithm has a faster calculation speed. Since ASDARL needs to adjust the parameter $T$, so it consumes more calculation time. The standard deviations in parentheses also illustrate the stability of the algorithms in this paper. Weighing the three values in the last column of Table \ref{table1}, we can find that the proposed algorithms have better performance in model selection. In summary, compared with the mentioned algorithms, SDARL and ASDARL algorithms have obvious numerical advantages for linear regression problems.

\subsection{Logistic regression}

In this section, we make some simulations and real data analysis in logistic regression model to illustrate the performance of SDARL and ASDARL algorithms. Here, we regard $F(\beta)$ as the negative logarithmic likelihood function, i.e., $F(\beta)=\frac{1}{n}\sum_{i=1}^{n}[\ln(1+e^{X_i^{\top}\beta})-y_i X_i^{\top}\beta]$. We firstly analyze the advantages of line search mentioned in this paper.

\subsubsection{An illustrative example}

In the example, we generate matrix $X$ and the underlying regression coefficient $\beta^*$ in the same way as described in Section \ref{lineexa}. Then the response variable is generated according to $y_i \sim \text{Binomial}(1,p_i)$, where $p_i=\frac{1}{1+\exp(-X_i^{\top}\beta^*)}, i=1, \cdots, n$. Since logistic regression model aims to classify, we randomly choose $80\%$ of the samples as the training set and the rest for the test set. Here we set $n=300, p=5000, K=10, \rho=0.2$. Based on 100 independent replications, we compare SDARL algorithm with its version of a fixed step size $\tau=1$, i.e., GSDAR in \cite{HJK2020} by classification accuracy rate, positive discovery rate, false discovery rate and combined discovery rate. The specific calculation results are shown in Figure \ref{fig:7}. It is clear that the classification accuracy rates of the SDARL algorithm are always much higher than those of the GSDAR algorithm. In addition, for SDARL, the values of positive discovery rate are always closer to 1, the values of false discovery rate are closer to 0, and the values of combined discovery rate are always closer to 2, while GSDAR are far inferior. These results show that the line search proposed in this paper is very necessary and effective for logistic regression problems.

\begin{figure}[h]
\centering
\includegraphics[width=0.45\textwidth,height=4cm]{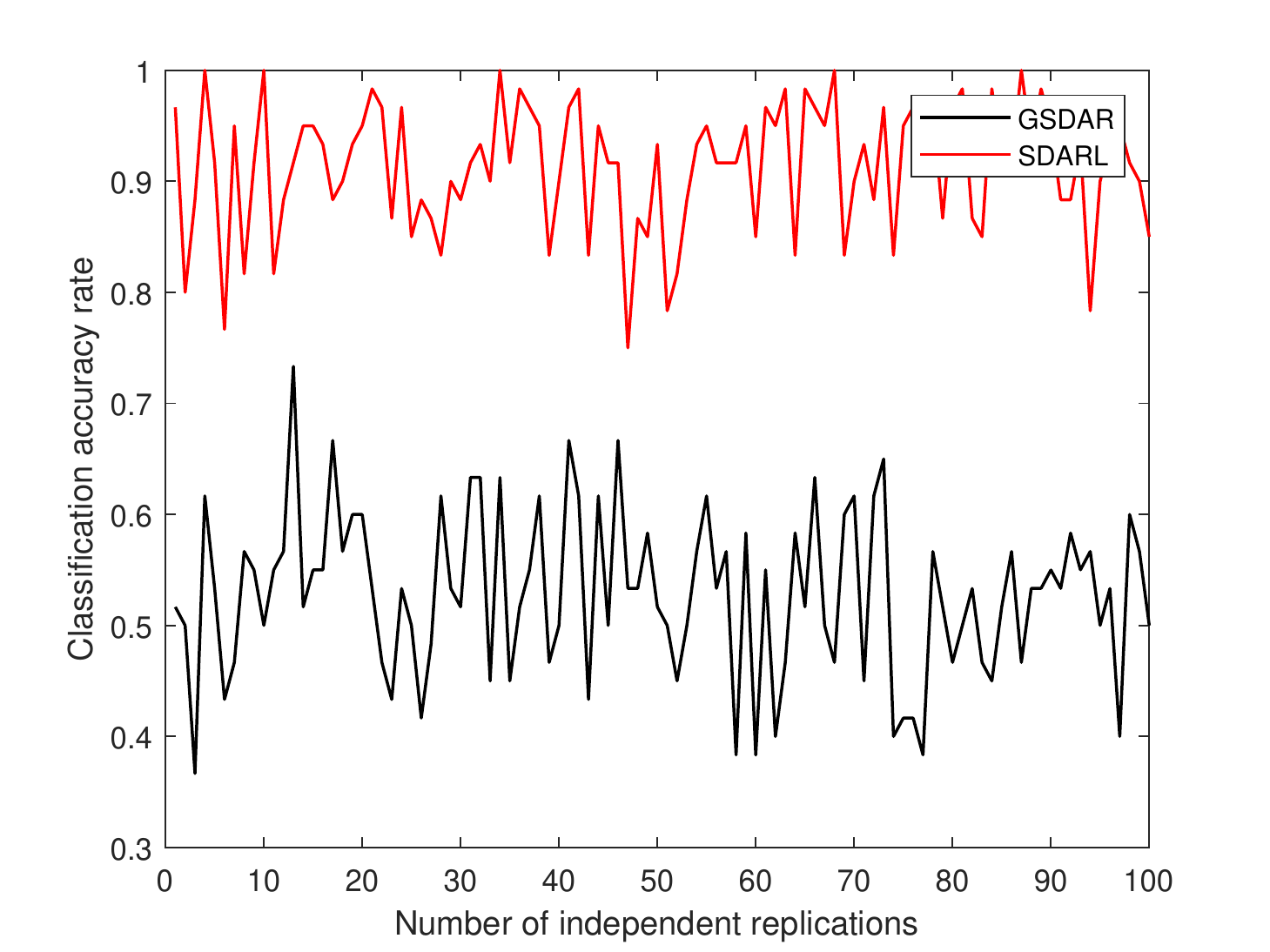}\hspace{-.4cm}
\includegraphics[width=0.45\textwidth,height=4cm]{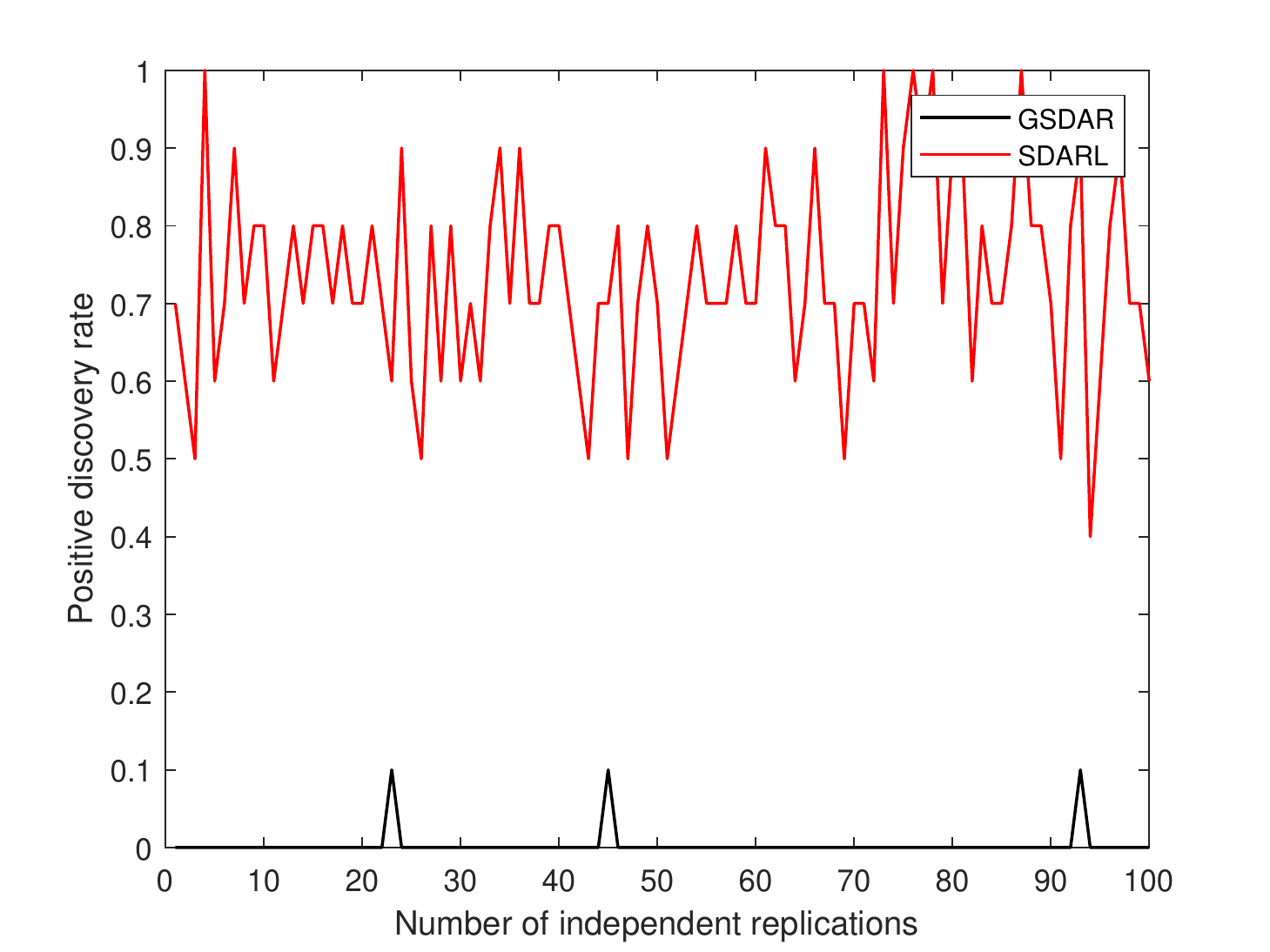}\\
\includegraphics[width=0.45\textwidth,height=4cm]{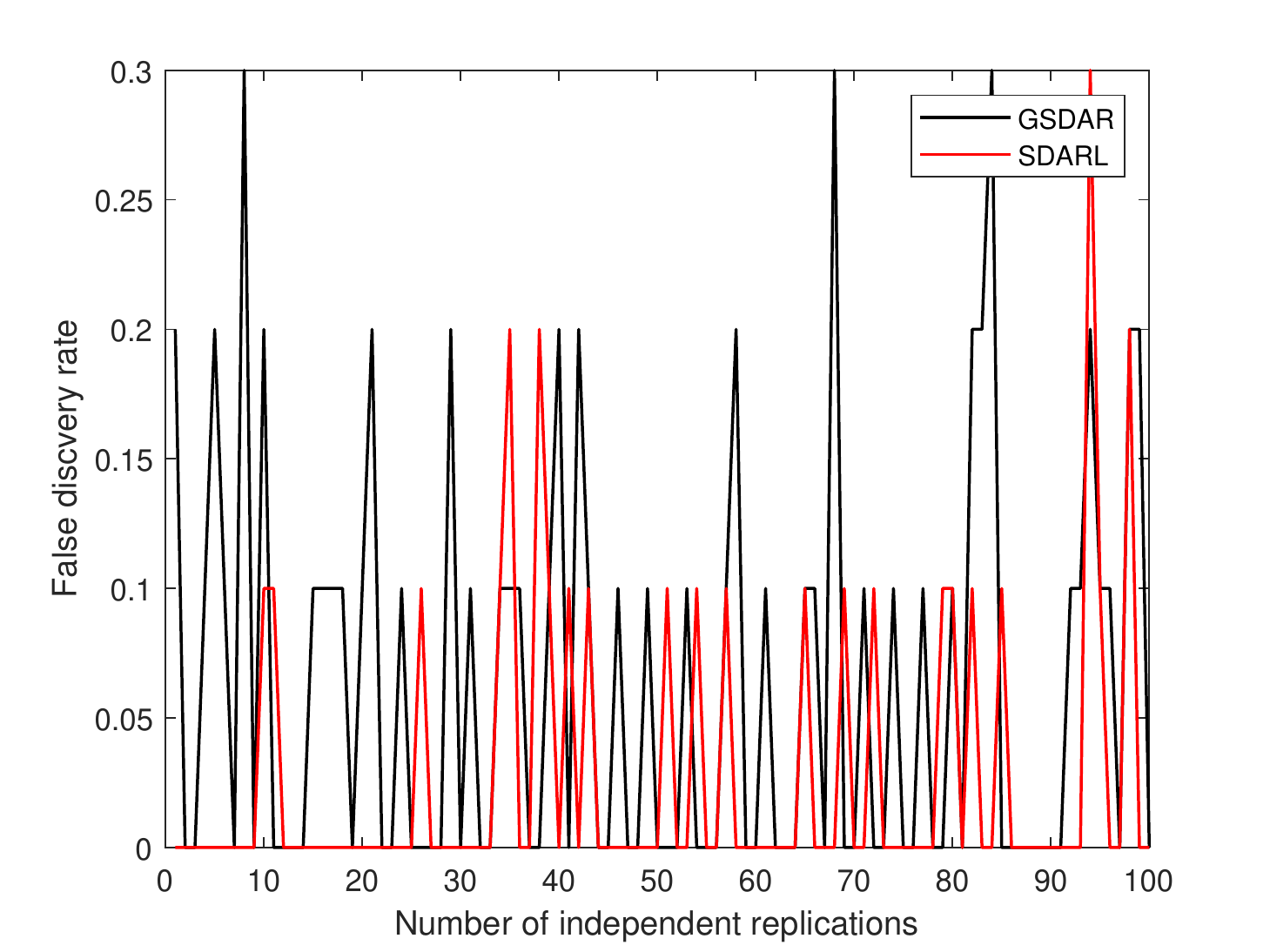}\hspace{-.4cm}
\includegraphics[width=0.45\textwidth,height=4cm]{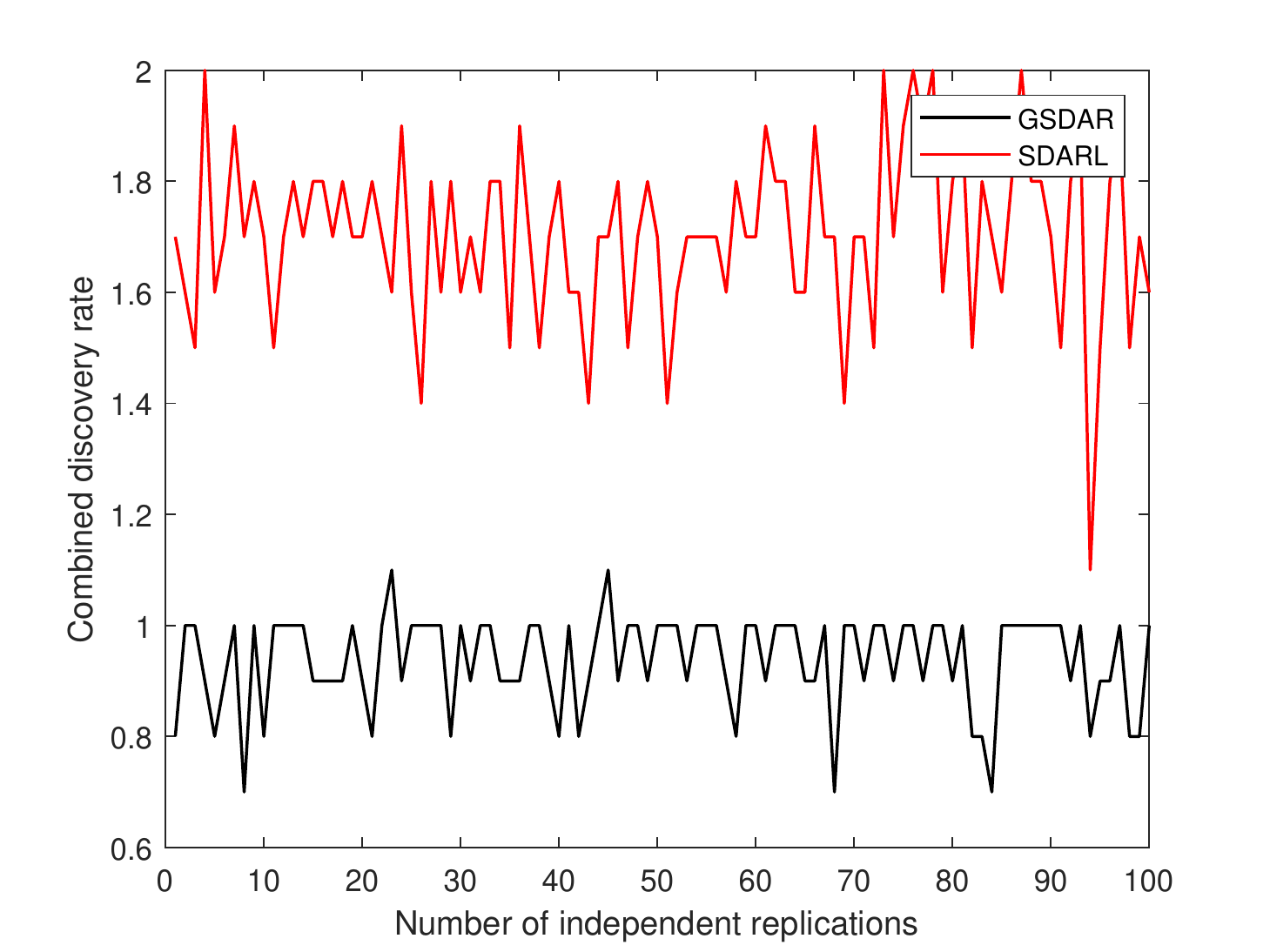}
\caption{{\small The comparison between SDARL and its version with a fixed step size $\tau= 1$ in logistic regression problems}}
\label{fig:7}
\end{figure}

In addition, in order to test the effectiveness of SDARL algorithm for logistic regression problems, we give the average classification accuracy rate and average number of iterations with different sparsity levels $K=5:5:50$ in Figure \ref{fig:8}. Here we also consider the influence of matrix correlation and take $\rho=0.2:0.3:0.8$. We can see that, as the sparsity level $K$ increases, the average classification accuracy rates of Algorithm SDARL are all above 80\%, and the average numbers of iterations are all below 11 for each $\rho$. These are sufficient to illustrate the effectiveness and rapid convergence of SDARL for logistic regression problems.

\begin{figure}
\centering
\includegraphics[width=0.5\textwidth,height=5cm]{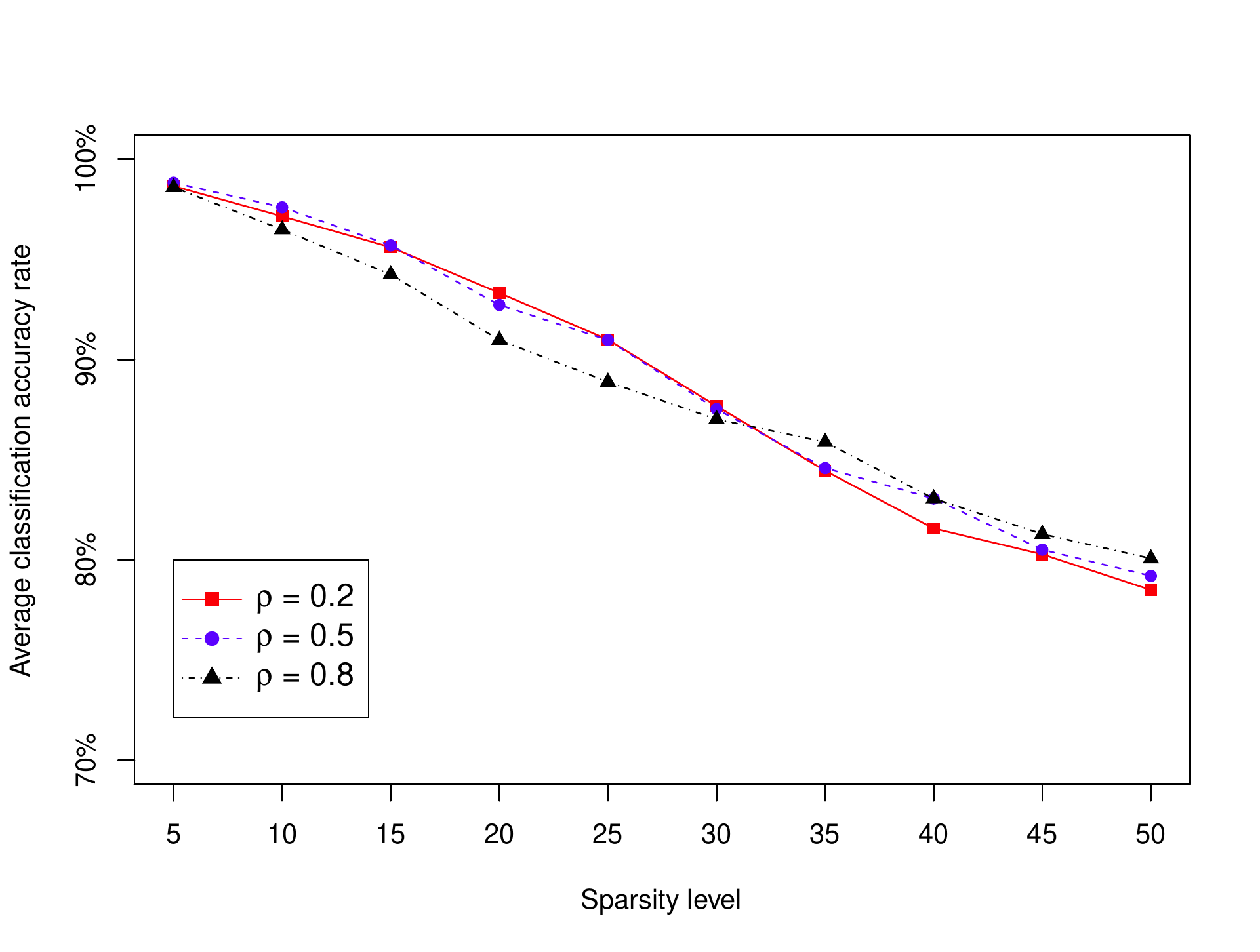}\hspace{-.2cm}
\includegraphics[width=0.5\textwidth,height=5cm]{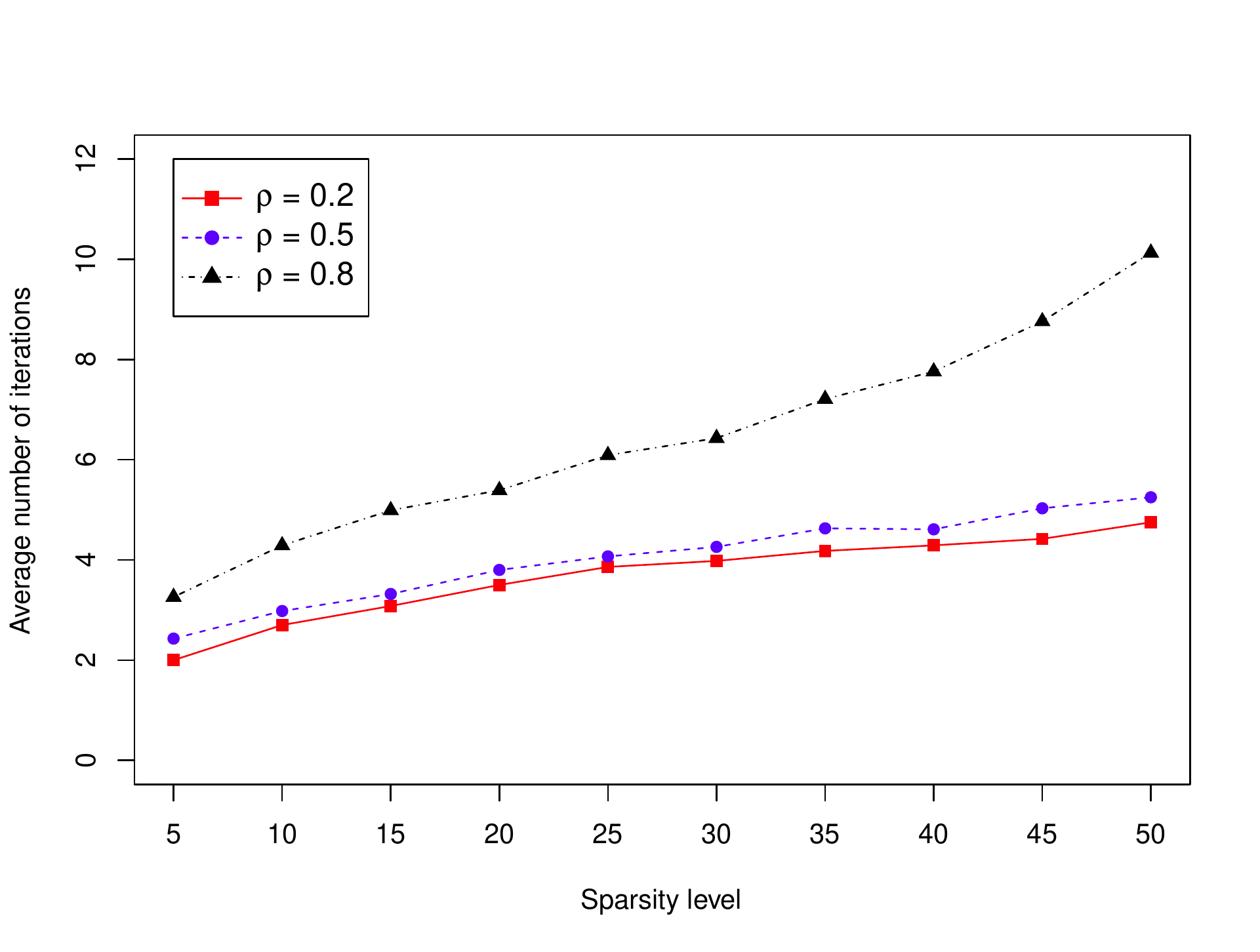}
\caption{{\small From left to right: the average classification accuracy rate and average number of iterations of SDARL for logistic regression as sparsity level $K$ increases.}}
\label{fig:8}
\end{figure}

\subsubsection{Influence of the model parameters}

In this part, we also consider the influence of the model parameters $\{n,p,K,\rho\}$ on the performance of ASDARL, LASSO, MCP and SCAD methods for logistic regression problems. We generate the data as above subsection and show the simulation results in Figure \ref{fig:9}- Figure \ref{fig:12} based on 10 independent replications. We can see that the values of APDR for LASSO, MCP and SCAD are higher than that of SDARL, but the gap is not big. However, the values of AFDR and ACDR for ASDARL algorithm are closer to the requirements of variable selection. Therefore, ASDARL can simultaneously select the relevant variables and avoid the irrelevant variables, thus reduce the complexity of the model.

\begin{figure}
\centering
\includegraphics[width=0.34\textwidth,height=4cm]{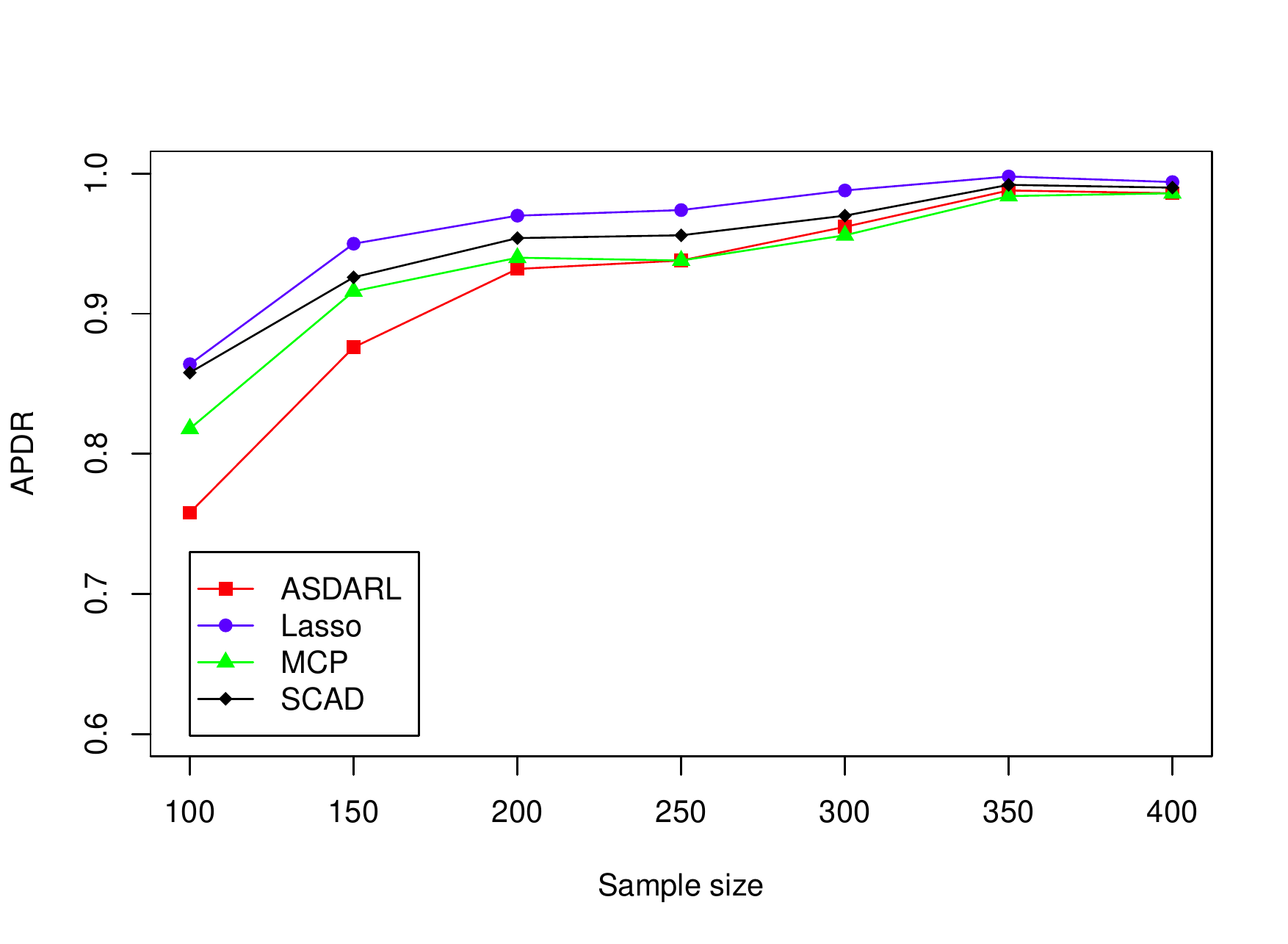}\hspace{-.3cm}
\includegraphics[width=0.34\textwidth,height=4cm]{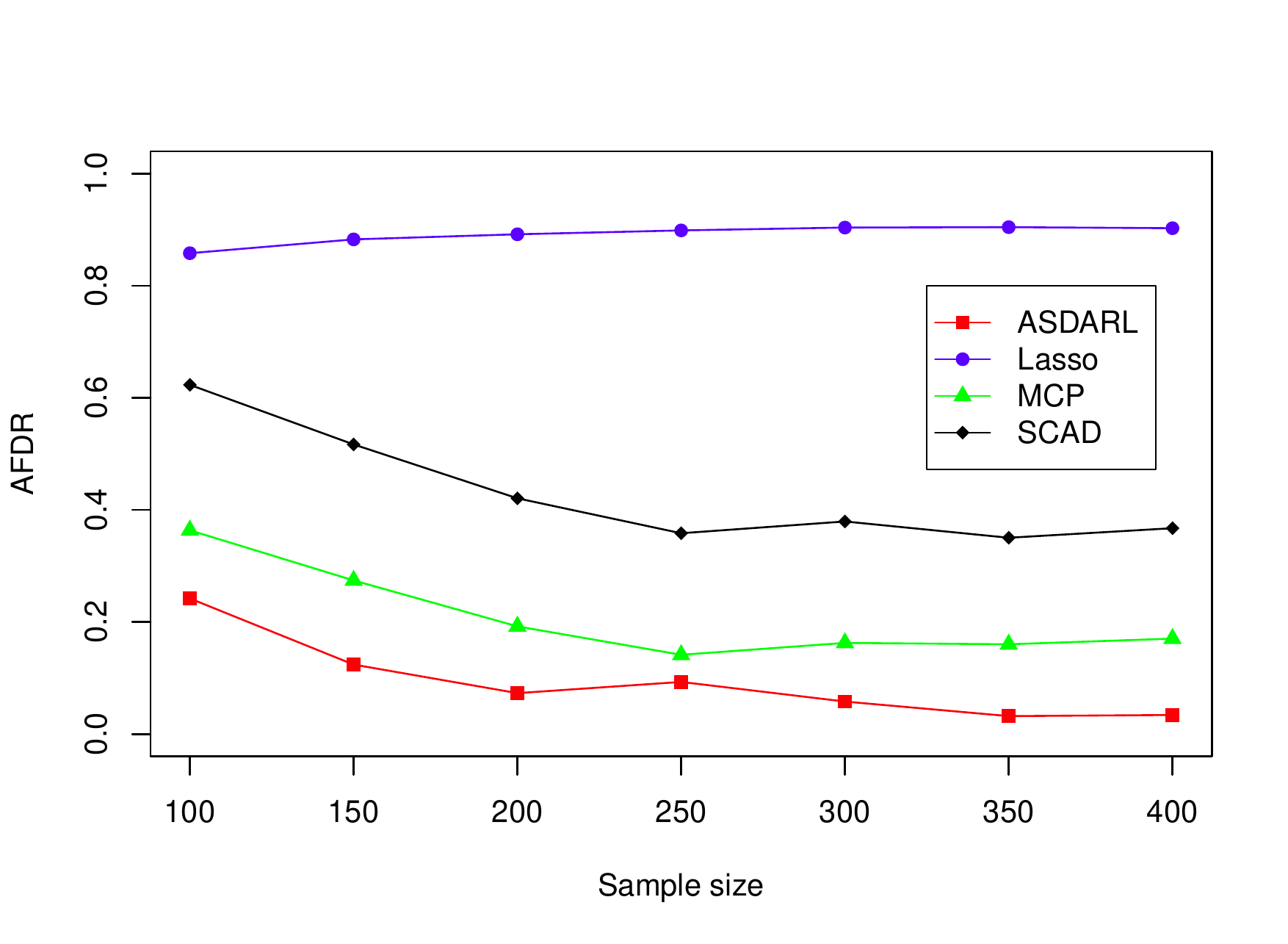}\hspace{-.3cm}
\includegraphics[width=0.34\textwidth,height=4cm]{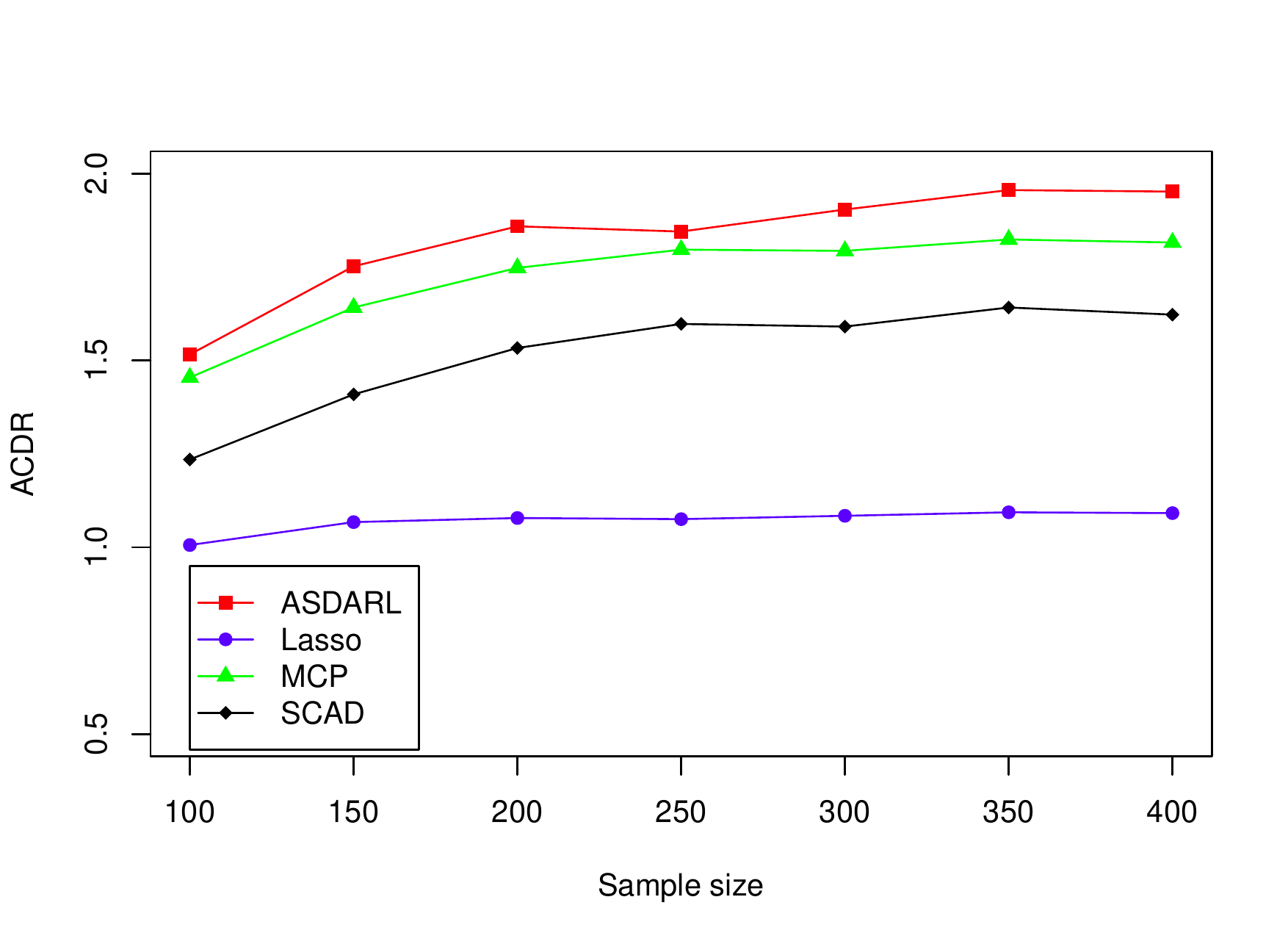}
\caption{{\small Numerical results (ARDR, AFDR, ACDR) of the influence of sample size in logistic regression problems with $n=100:50:400$, $p=500$, $K=5$, $R=10$, $\rho=0.2$.}}
\label{fig:9}
\end{figure}

\begin{figure}
\centering
\includegraphics[width=0.34\textwidth,height=4cm]{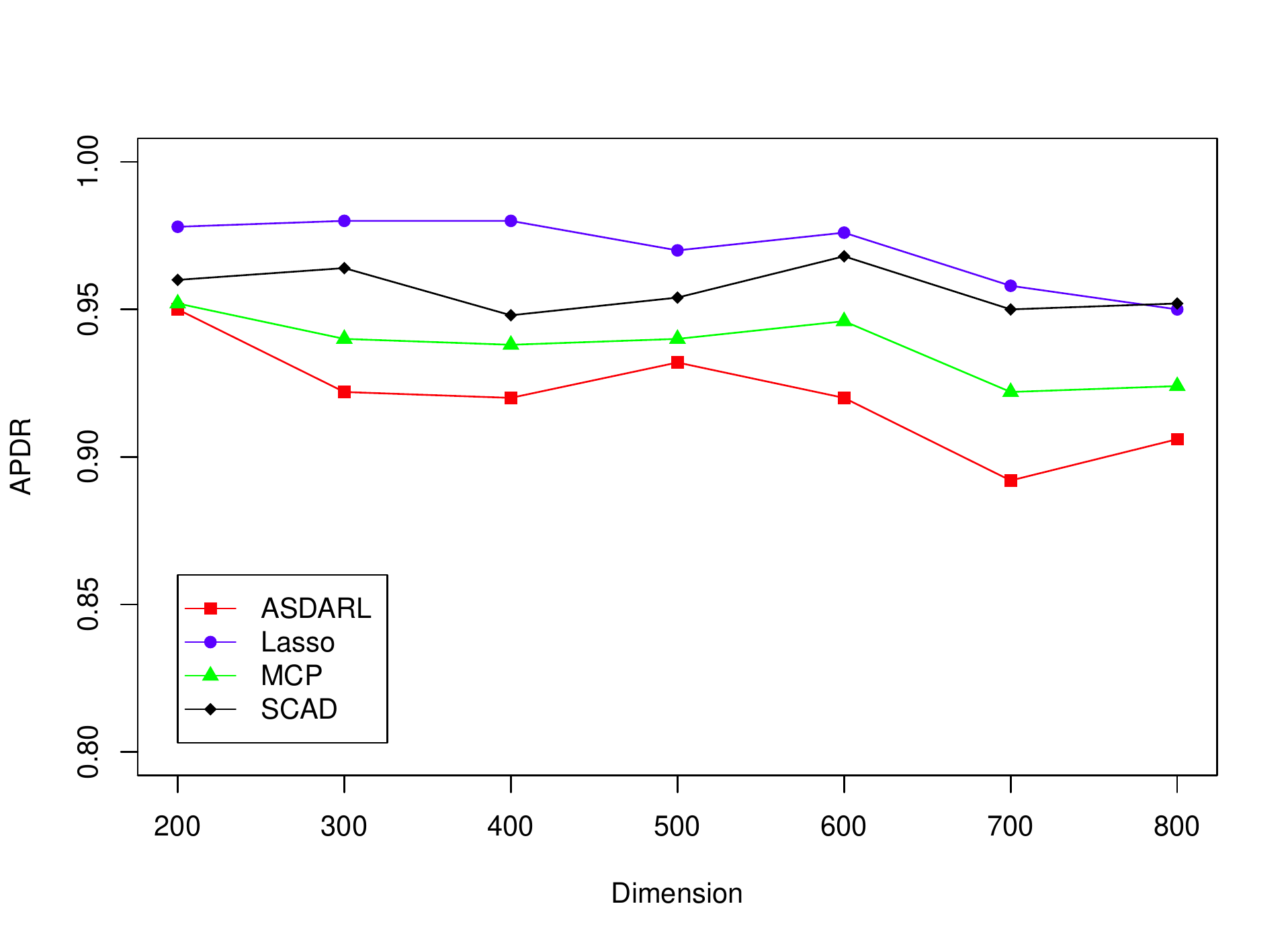}\hspace{-.3cm}
\includegraphics[width=0.34\textwidth,height=4cm]{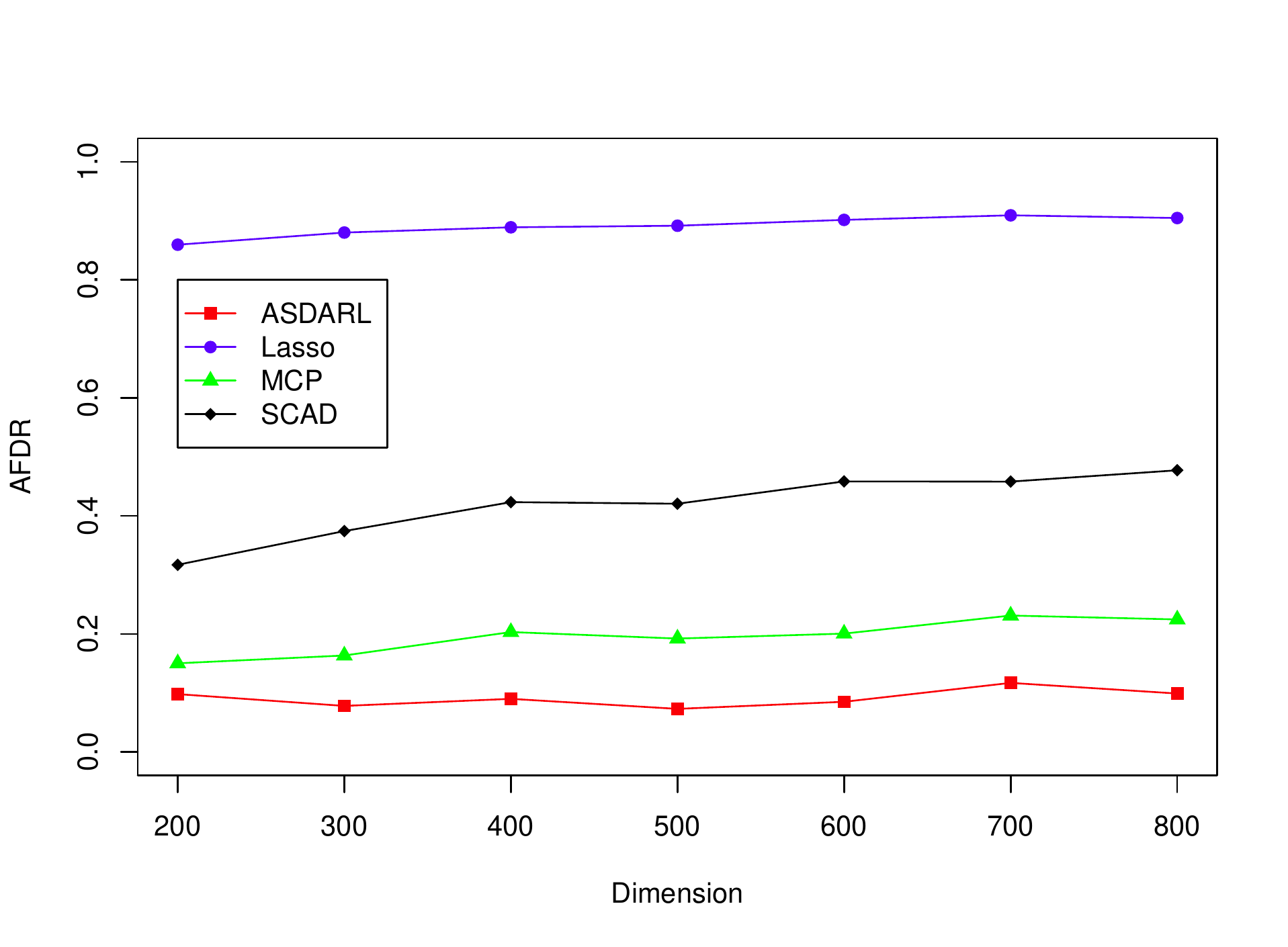}\hspace{-.3cm}
\includegraphics[width=0.34\textwidth,height=4cm]{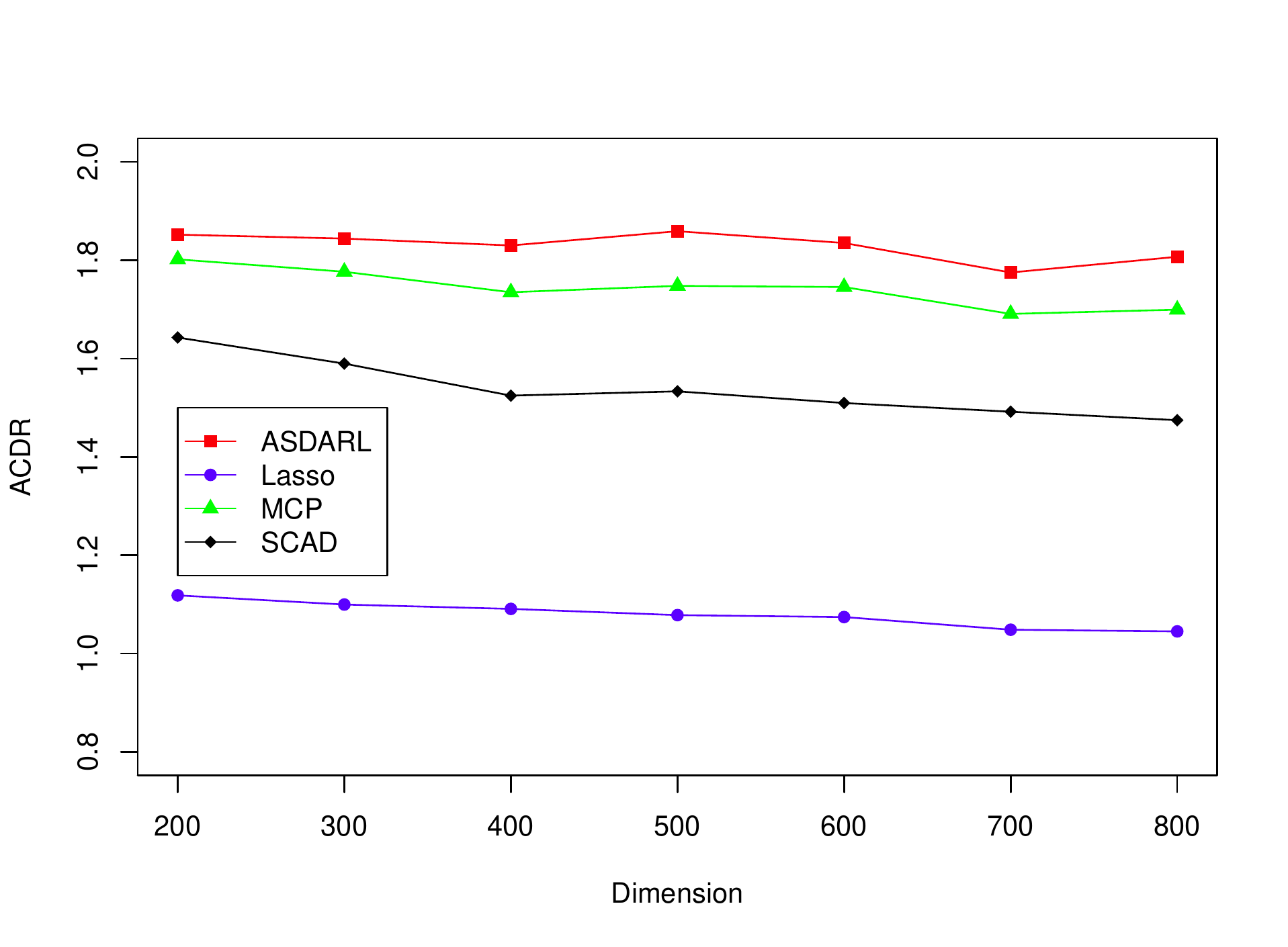}
\caption{{\small Numerical results (ARDR, AFDR, ACDR) of the influence of ambient dimension in logistic regression problems with $n=200$, $p=200:100:800$, $K=5$, $R=10$, $\rho=0.2$.}}
\label{fig:10}
\end{figure}

\begin{figure}
\centering
\includegraphics[width=0.34\textwidth,height=4cm]{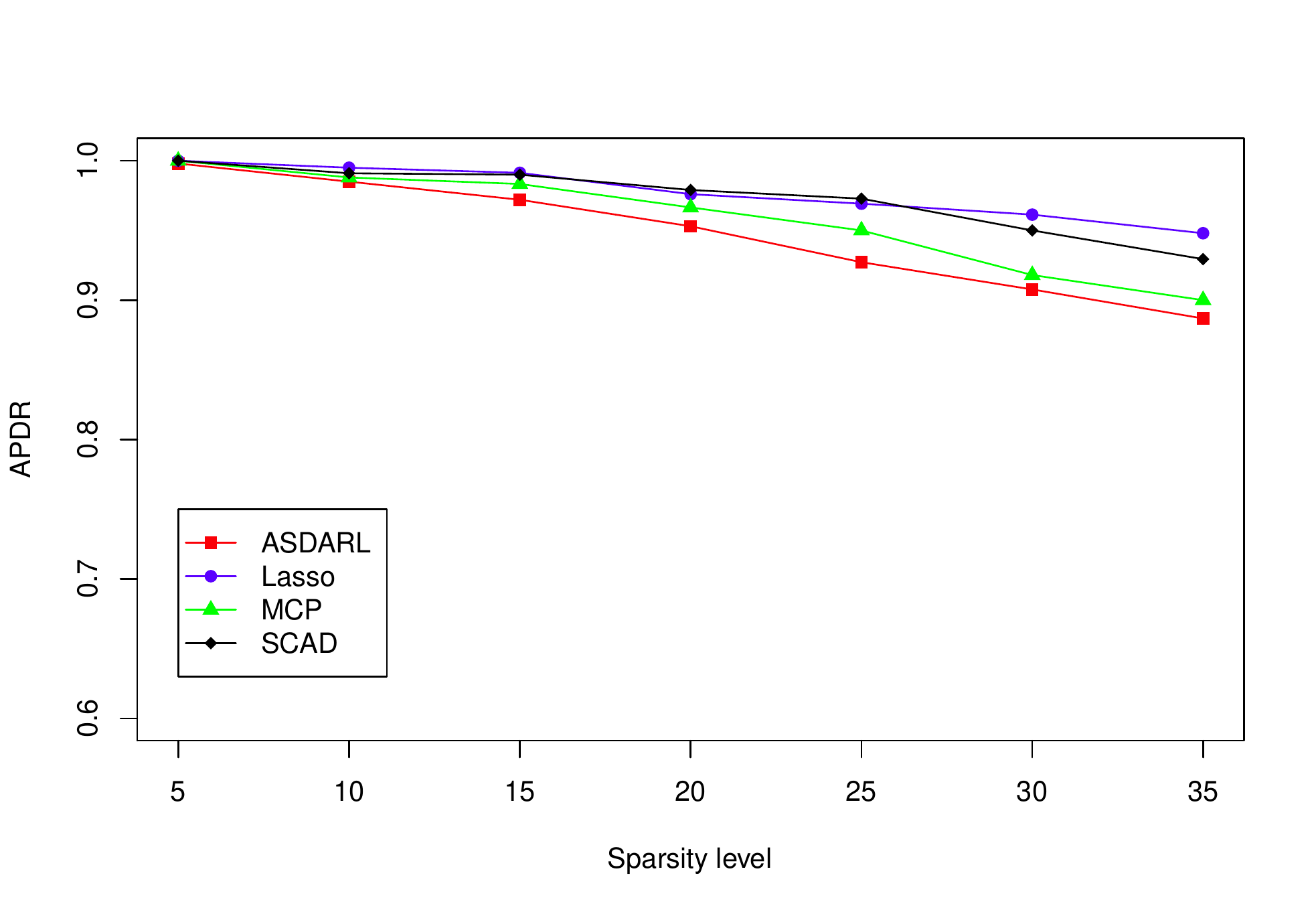}\hspace{-.3cm}
\includegraphics[width=0.34\textwidth,height=4cm]{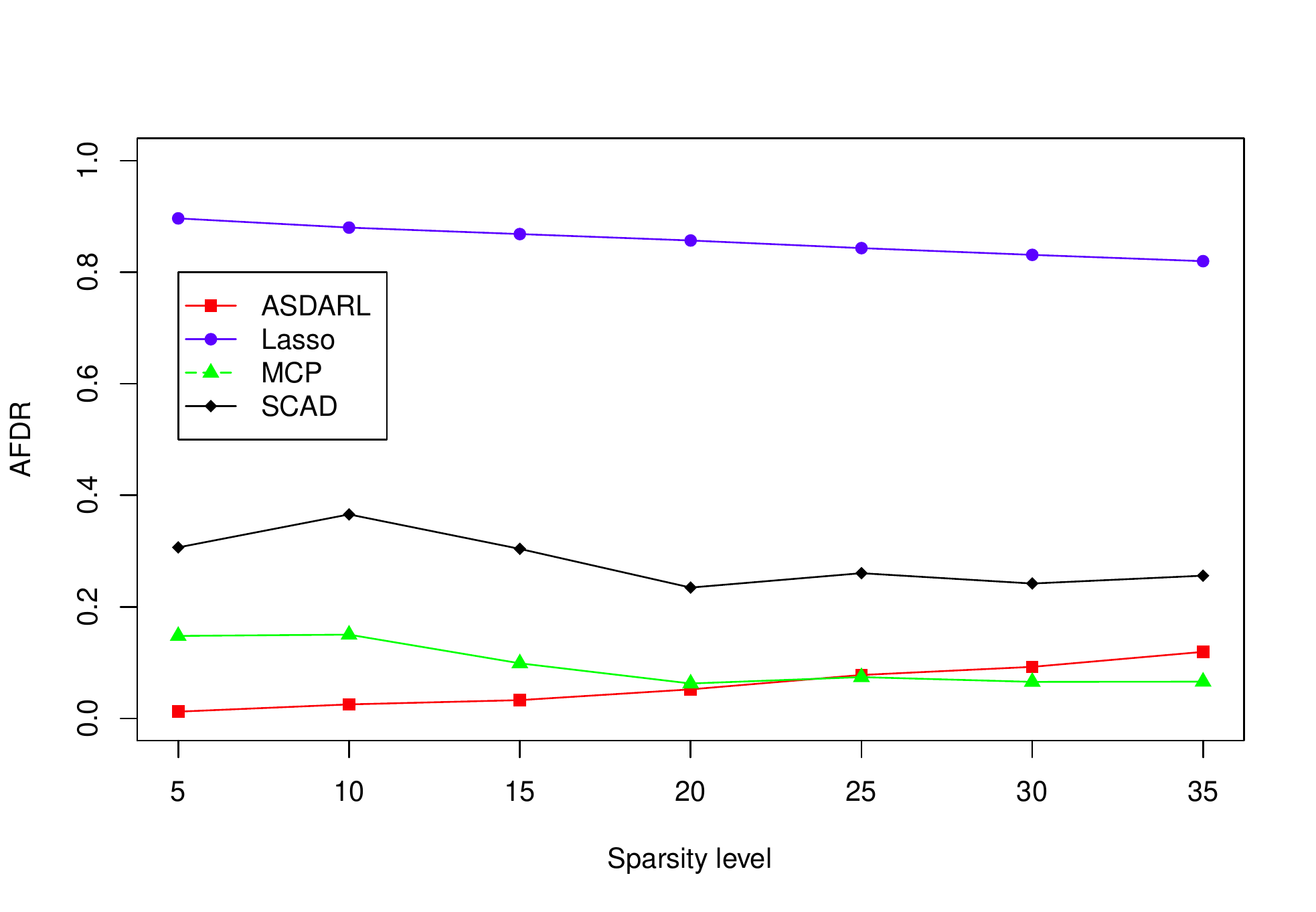}\hspace{-.3cm}
\includegraphics[width=0.34\textwidth,height=4cm]{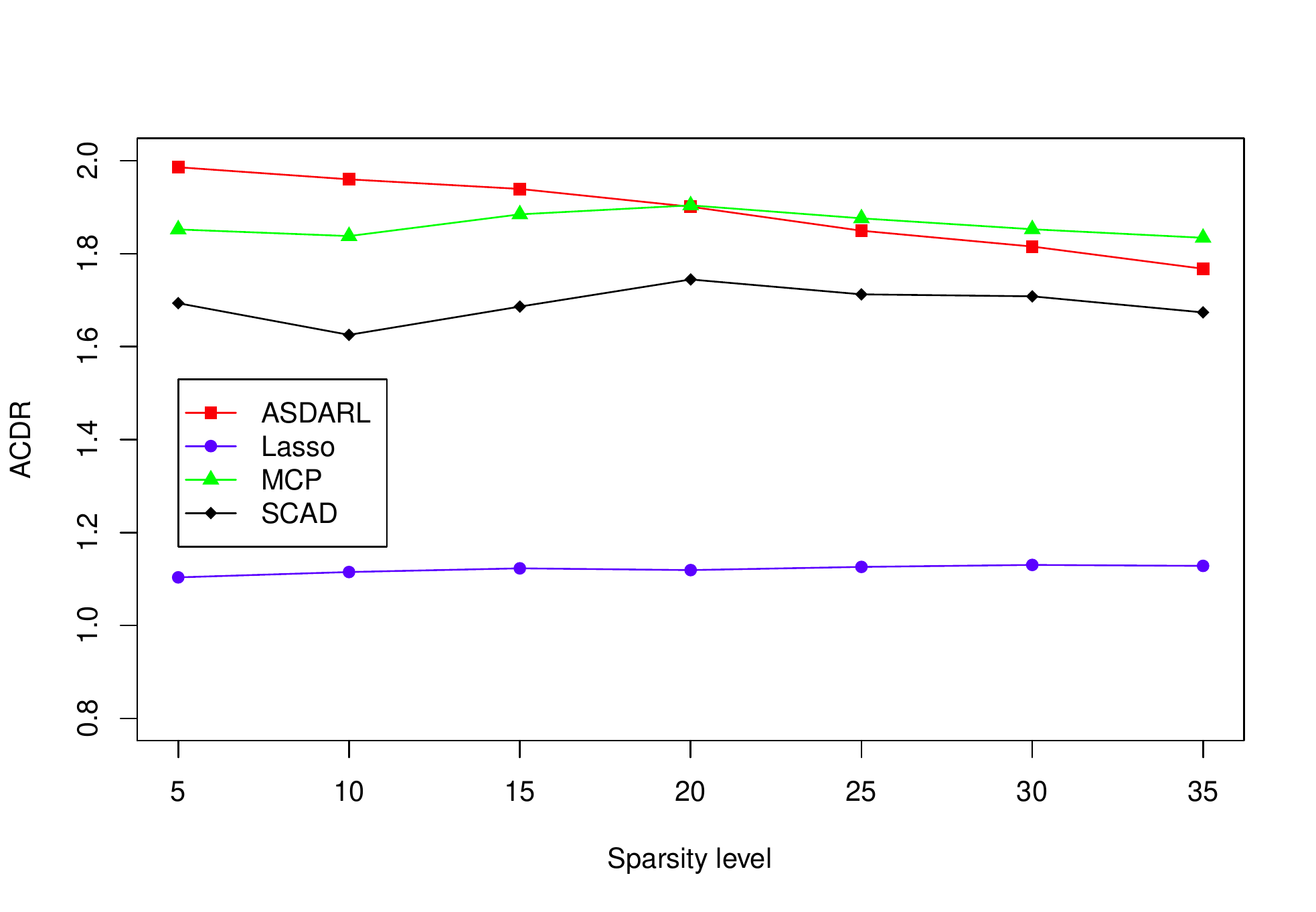}
\caption{{\small Numerical results (ARDR, AFDR, ACDR) of the influence of sparsity level in logistic regression problems with $n=800$, $p=1000$, $K=5:5:35$, $R=10$, $\rho=0.2$.}}
\label{fig:11}
\end{figure}

\begin{figure}
\centering
\includegraphics[width=0.34\textwidth,height=4cm]{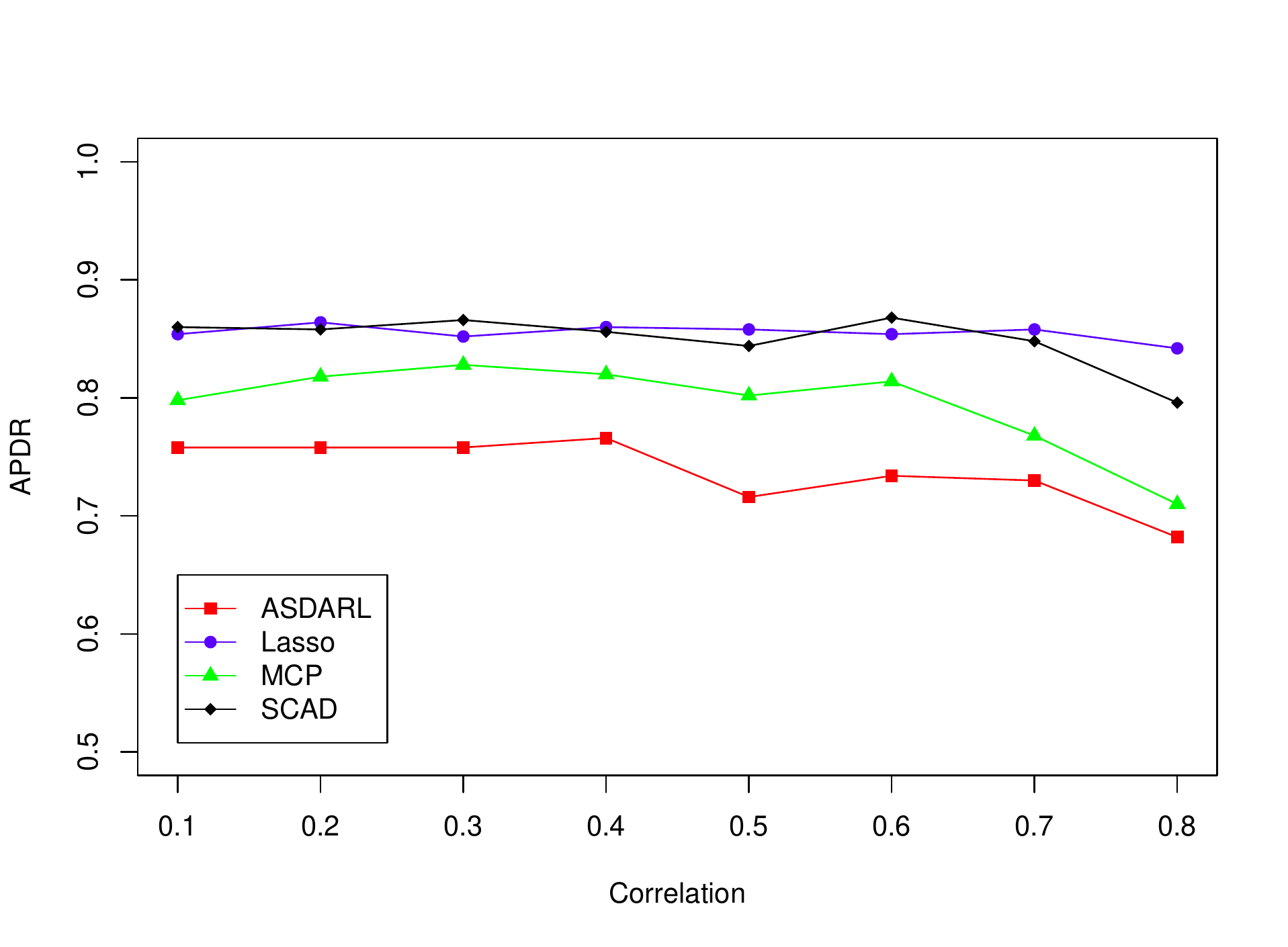}\hspace{-.3cm}
\includegraphics[width=0.34\textwidth,height=4cm]{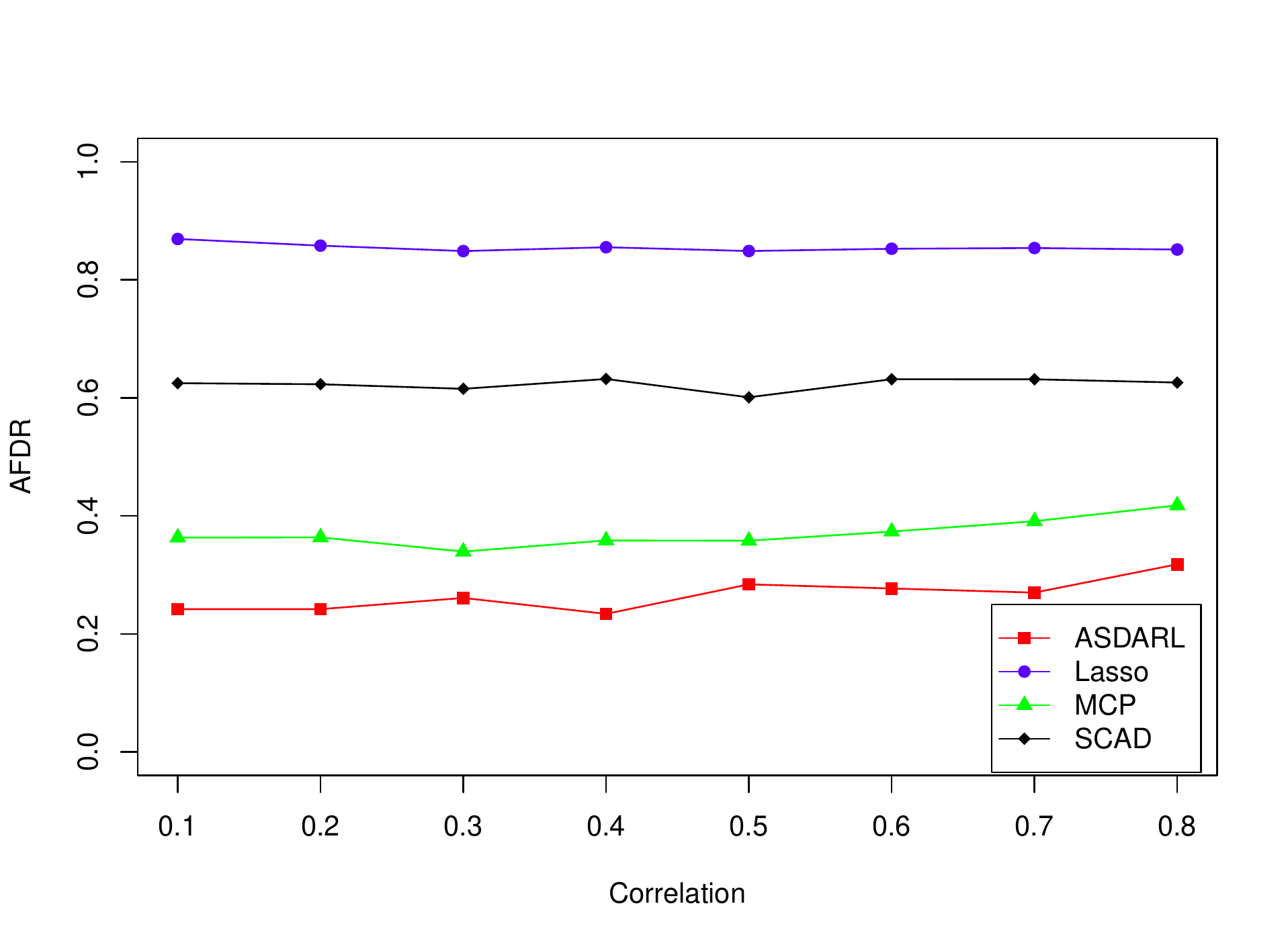}\hspace{-.3cm}
\includegraphics[width=0.34\textwidth,height=4cm]{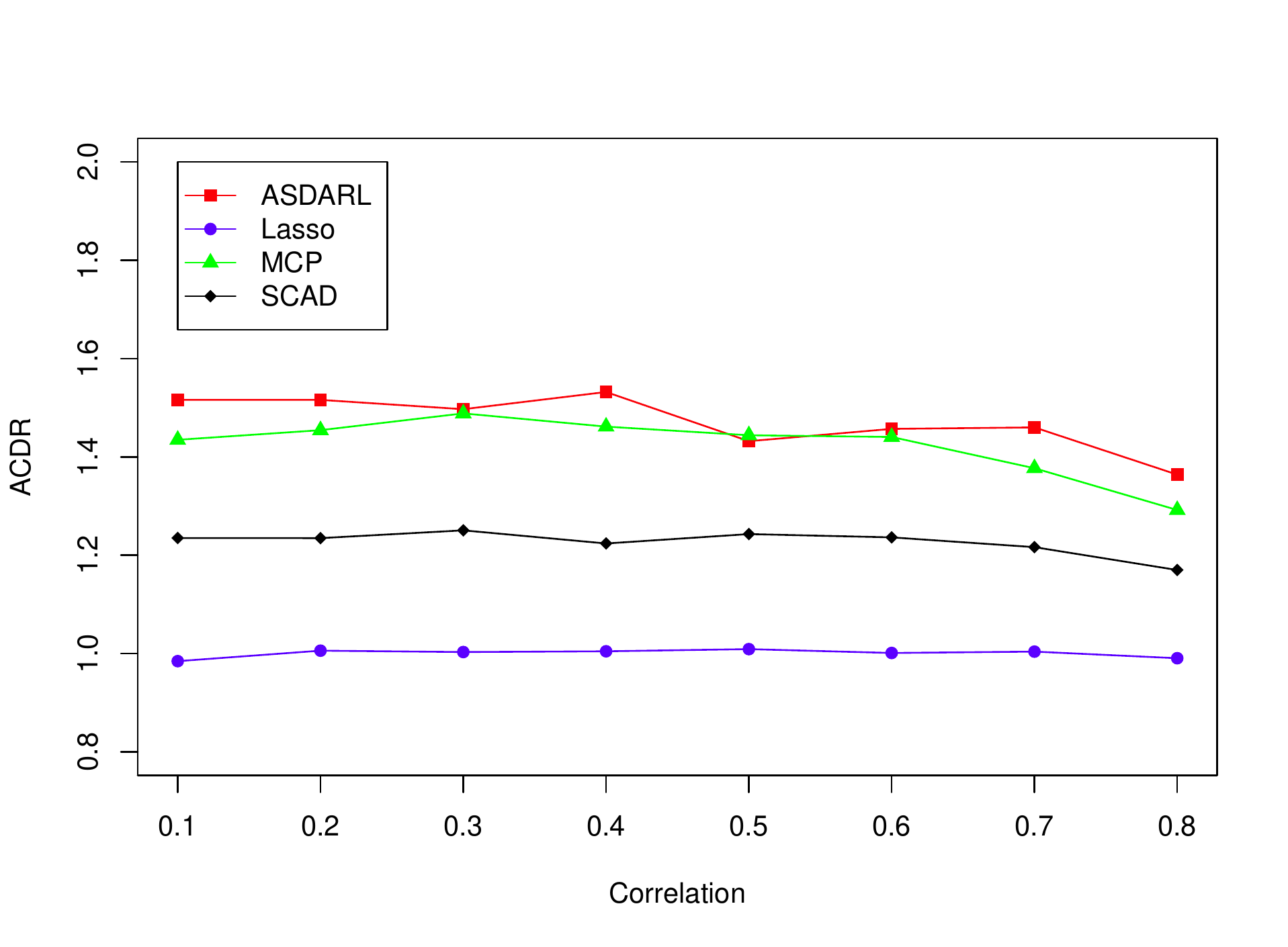}
\caption{{\small Numerical results (ARDR, AFDR, ACDR) of the influence of correlation in logistic regression problems with $n=500$, $p=1000$, $K=10$, $R=10$, $\rho=0.1:0.1:0.8$.}}
\label{fig:12}
\end{figure}

\subsubsection{Numerical comparison with synthetic data}

Based on 100 independent replications, we here compare the performance of LASSO, MCP and SCAD with the proposed methods in terms of ACAR, ARE, Time(s), APDR, AFDR and ACDR with synthetic data. We generate matrix $X$ and the underlying regression coefficient $\beta^*$ in the same way as described in Section \ref{linecom}. We give the numerical results in Table \ref{table2} with problem setting $p=5000$, $n=300$, $K=10$, $R=100$ for three kinds of correlation $\rho=0.2:0.3:0.8$. We can see that the five algorithms considered in this paper can effectively solve the simulation problems. A closer look shows that SDARL and its adaptive version ASDARL have exactly the same results except for the CPU time. This is because the ASDARL algorithm needs to run with different $T$, this will inevitably increase the CPU time. Then the above phenomenon in the table is reasonable. In addition, from the perspective of comparing, it can be seen that MCP always has a relatively higher classification accuracy rate with different $\rho$. SDARL has obvious advantages in average relative error and CPU time. From the results in the last column, we can see that the five algorithms have similar ability to select relevant variables, while MCP, SDARL and ASDARL tend to select fewer irrelevant variables than LASSO and SCAD.

\begin{table}[h]\tiny
\renewcommand\arraystretch{1.5}
\centering {
\setlength{\tabcolsep}{4mm}{
\begin{tabular}{|c||c|c|c|c|c|c|}
\Xhline{1.2pt}
$\rho$ & Algorithm & ACAR & ARE  & Time(s)& (APDR, AFDR, ACDR)\\
\Xhline{0.8pt}
\multirow{5}{*}{0.2} & LASSO & 86.91\% (4.71e-2) & 9.89e-1 (2.96e-3) & 2.21 (1.26e-1) & (0.7790, 0.9161, 0.8629) \\
& MCP & \textbf{93.88\% (3.87e-2)} & 9.48e-1 (2.57e-2)  & 3.68 (5.38e-1) & (0.779, \textbf{0.2114}, \textbf{1.5676}) \\
& SCAD & 93.68\% (4.08e-2)& 9.42e-1 (4.59e-2) & 4.28 (9.17e-1) & (\textbf{0.8030}, 0.5193, 1.2837) \\
& SDARL & 92.08\% (4.85e-2) & \textbf{6.46e-1 (2.07e-1)} & \textbf{0.56 (1.84e-1)} & (0.7180, 0.2820, 1.4360) \\
& ASDARL & 92.08\% (4.85e-2) & \textbf{6.46e-1 (2.07e-1)} & 1.61 (5.64e-1) & (0.7180, 0.2820, 1.4360) \\
\Xhline{0.8pt}
\multirow{5}{*}{0.5}  & LASSO & 87.43\% (5.01e-2)&  9.91e-1 (2.52e-3) & 2.19 (1.13e-1)  & (0.7670, 0.9148, 0.8522) \\
& MCP & \textbf{94.46\% (3.56e-2)} & 9.54e-1 (2.81e-2)  & 3.77 (5.62e-1) & (0.765, \textbf{0.2374}, \textbf{1.5276}) \\
& SCAD & 94.05\% (3.85e-2)  & 9.54e-1 (2.39e-2) & 4.41 (1.09e+0) & (\textbf{0.7870}, 0.5190, 1.268) \\
& SDARL & 92.53\% (4.65e-2) & \textbf{7.30e-1 (1.61e-1)} & \textbf{0.73 (2.19e-1)} & (0.7090, 0.2910, 1.4180) \\
& ASDARL & 92.53\% (4.65e-2) & \textbf{7.30e-1 (1.61e-1)} & 1.89 (6.44e-1) & (0.7090, 0.2910, 1.4180) \\
\Xhline{0.8pt}
\multirow{5}{*}{0.8}  & LASSO & 86.31\% (4.95e-2) & 9.93e-1 (2.01e-3) & 2.31 (1.29e-1)  & (0.7550, 0.9129, 0.8421) \\
& MCP & \textbf{93.55\% (4.45e-2)} & 9.64e-1 (2.31e-2)  & 3.95 (6.41e-1) & (0.7610, \textbf{0.2234}, \textbf{1.5376}) \\
& SCAD & 93.31\% (4.54e-2) & 9.62e-1 (2.50e-2)   & 4.29 (1.04e+0) & (\textbf{0.7980}, 0.5208, 1.2772) \\
& SDARL & 91.60\% (5.60e-2) & \textbf{8.38e-1 (1.06e-1)}  & \textbf{0.96 (2.89e-1)} & (0.6870, 0.3130, 1.3740) \\
& ASDARL & 91.60\% (5.60e-2) & \textbf{8.38e-1 (1.06e-1)}  & 2.07 (5.68e-1) & (0.6870, 0.3130, 1.3740) \\
\Xhline{1.2pt}
\end{tabular}}
\scriptsize\caption{Numerical comparison in logistic regression with $n = 300$, $p = 5000$, $K = 10$, $R=100$ and $\rho=0.2:0.3:0.8$.}
\label{table2}
}
\end{table}

\subsubsection{Numerical comparison with real data}

In this section, we carry out numerical comparisons with five real data sets: duke breast-cancer, colon-cancer, leukemia, madelon and splice. These data sets can be downloaded from \url{https://www.csie.ntu.edu.tw/~cjlin/libsvmtools/datasets/}. We first delete the vacant items, and then replace the values $-1$ with $0$ in response variable $y$. We compare SDARL/ASDARL with LASSO, MCP and SCAD in terms of classification accuracy rate ($\text{CAR}$), CPU time ($\text{Time(s)}$) and the number of selected variables ($\text{NSV}$). Let $T = \gamma n/\log(n)$ in SDARL, where $\gamma$ is a positive and finite constant. When the data set has no testing data, we get the classification accuracy rate through the training set itself. The numerical results are showed in Table \ref{table3}. Below each data name, we give the data information in the form of (features dimensionality, training size, testing size). Numerical results show that the proposed algorithms in this paper are comparable to LASSO, MCP and SCAD for the considered real data.
%All data have been normalized such that the mean is $0$ and variance is $1$.

\begin{table}[h]\scriptsize
\renewcommand\arraystretch{1.5}
\centering {
\setlength{\tabcolsep}{4mm}{
\begin{tabular}{|c||c|c|c|c|c|}
\Xhline{1.2pt}
Data name & Algorithm & CAR  & Time(s)& NSV\\
\Xhline{0.8pt}
\multirow{5}{*}{\makecell[c]{Duke breast-cancer \\ (7219, 38, 4)}} & LASSO & 75\%  & 1.32 & 18 \\
& MCP & 25\%   & 1.81 & 6 \\
& SCAD & 75\%  & 1.42 & 13 \\
& SDARL & \textbf{1} & \textbf{0.02} & 9 \\
& ASDARL & \textbf{1}  & 0.03 & 9 \\
\Xhline{0.8pt}
\multirow{5}{*}{\makecell[c]{Colon-cancer \\ (2000, 62, 0)}} & LASSO & 88.71\%  & 0.53 & 8 \\
& MCP & 85.48\%   & 1.01 & 2 \\
& SCAD & 90.32\%  & 0.72 & 10 \\
& SDARL & \textbf{1} & \textbf{0.05} & 6 \\
& ASDARL & \textbf{1}  & 0.06 & 5 \\
\Xhline{0.8pt}
\multirow{5}{*}{\makecell[c]{Leukemia \\ (7219, 38, 34)}} & LASSO & 91.17\%  & 0.93 & 13 \\
& MCP & \textbf{94.11\%}   & 1.73 & 4 \\
& SCAD & \textbf{94.11\%}  & 1.44 & 10 \\
& SDARL & 91.17\% & \textbf{0.02} & 12 \\
& ASDARL & 91.17\% & 0.03 & 12 \\
\Xhline{0.8pt}
\multirow{5}{*}{\makecell[c]{Madelon \\ (500, 2000, 600)}} & LASSO & \textbf{54.50\%}  & 36.79 & 34 \\
& MCP & 53.16\%   & 12.61 & 14 \\
& SCAD & 53.50\%  & 12.54 & 15 \\
& SDARL & 52.16\% & \textbf{1.01} & 52 \\
& ASDARL & \textbf{54.50\%}  & 21.66 & 14 \\
\Xhline{0.8pt}
\multirow{5}{*}{\makecell[c]{Splice \\ (60, 1000, 2175)}} & LASSO & 56.87\%  & 1.25 & 39 \\
& MCP & 68.78\%   & 1.33 & 20 \\
& SCAD & 62.75\%  & 1.16 & 32 \\
& SDARL & 52.64\% & \textbf{0.29} & 28 \\
& ASDARL & \textbf{74.71\%}  & 1.28 & 11 \\
\Xhline{1.2pt}
\end{tabular}}
\scriptsize\caption{Numerical results of the real data.}
\label{table3}
}
\end{table}

\section{Conclusion}\label{con}

This paper mainly focuses on high-dimensional data analysis with sparse assumption and considers regression models with $\ell_0$-penalty. Based on support detection using primal and dual information and root finding, we propose a data-driven line search rule to adaptively update the step size, and refer to the algorithm as SDARL for brevity. In addition, we also propose the ASDARL algorithm, an adaptive version of SDARL, to deal with the situation that the true sparsity level is unknown in advance. Theoretically, without any restrictions on the parameters of restricted strong convexity and smoothness for the loss function, we establish the $\ell_2$ error bound between the iteration sequence and the target regression coefficient. This weakens the conditions of theoretical analysis for existing literature, so that the numerical calculation does not need to manually adjust the step size, but can directly follow the line search, thereby expanding the application range of SDARL/ASDARL algorithm. We also analyze the necessity and effectiveness of the novel line search method in this paper from both theoretical and numerical aspects. In addition, the numerical comparisons with LASSO, MCP and SCAD in linear and logistic regression problems illustrate the stability and effectiveness of the SDARL/ASDARL algorithm.

\section*{Acknowledgements}
The work of Xiliang Lu is partially supported by the National Key Research and Development Program of China (No. 2020YFA0714200), the National Science Foundation of China (No. 11871385) and the Open Research Fund of KLATASDS2005. The work of Yuling Jiao is supported in part by the National Science Foundation of China (No. 11871474) and the research fund of KLATASDSMOE of China.

%%%%%%%%%%%%%%%%%%%%%%%%%%%%%%%%%%%%%%%%%%%%%%%%%%%%%%%%%%%%%%%%%%%%%%%%%%%%%%%%%%%%%%%%%%%%%%%%%%%%
\bigskip
%\newpage
\begin{appendix}
\setcounter{equation}{0}
\setcounter{subsection}{0}
\renewcommand{\theequation}{A.\arabic{equation}}
\renewcommand{\thesubsection}{A.\arabic{subsection}}
\renewcommand{\thelemma}{\arabic{lemma}}
\begin{flushleft}
%\centerline{\textbf{\Large APPENDIX}}
\textbf{\Large APPENDIX}
\end{flushleft}
%\appendix

\subsection{Proof of Lemma \ref{le1}}\label{leproof}
\begin{proof}
%See Appendix A.1.
The proof is similar to that of \cite[Lemma 1]{HJK2020}, thus is omitted here.
\end{proof}

\subsection{Proof of Theorem \ref{welldef}}\label{thwel}

\begin{proof}
Let $\theta^k=\beta^k-\tau^k \nabla F(\beta^k)$, we have
\begin{align*}
F(\theta^{k+1}|_{\A^{k+1}})-F(\beta^{k+1})\leq \langle \nabla F(\beta^{k+1}), \theta^{k+1}|_{\A^{k+1}}-\beta^{k+1}\rangle+\frac{M_{2T}}{2}\|\theta^{k+1}|_{\A^{k+1}}-\beta^{k+1}\|^2.
\end{align*}
Then on the one hand,
\begin{align*}
&\langle \nabla F(\beta^{k+1}), \theta^{k+1}|_{\A^{k+1}}-\beta^{k+1} \rangle\\&=\langle\nabla F(\beta^{k+1}), \theta^{k+1}|_{\A^{k+1}} \rangle\\
&=\langle\nabla_{\A^{k+1}} F(\beta^{k+1}), \theta^{k+1}_{\A^{k+1}} \rangle\\
&=\langle\nabla_{\A^{k+1}\backslash \A^k} F(\beta^{k+1}), \theta^{k+1}_{\A^{k+1}\backslash \A^k} \rangle.
\end{align*}
Because $\theta^{k+1}_{\A^{k+1}\backslash \A^k}=\beta_{\A^{k+1}\backslash \A^k}^{k+1}-\tau^{k+1} \nabla_{\A^{k+1}\backslash \A^k} F(\beta^{k+1})=-\tau^{k+1}  \nabla_{\A^{k+1}\backslash \A^k} F(\beta^{k+1})$, so we have
\begin{align}\label{eq3}
\langle \nabla F(\beta^{k+1}), \theta^{k+1}|_{\A^{k+1}}-\beta^{k+1} \rangle=-\tau^{k+1} \|\nabla_{\A^{k+1}\backslash \A^k} F(\beta^{k+1})\|^2.
\end{align}

On the other hand,
\begin{align*}
&\|\theta^{k+1}|_{\A^{k+1}}-\beta^{k+1}\|^2=\|\theta^{k+1}|_{\A^{k+1}}-\beta^{k+1}|_{\A^k}\|^2\\
&=\|\theta^{k+1}|_{\A^{k+1}\backslash \A^k}+\theta^{k+1}|_{\A^{k+1}\bigcap \A^k}-\beta^{k+1}|_{\A^{k+1}\bigcap \A^k}-\beta^{k+1}|_{\A^k\backslash \A^{k+1}}\|^2\\
&=\|\theta^{k+1}_{\A^{k+1}\backslash \A^k}\|^2+\|\theta^{k+1}_{\A^{k+1}\bigcap \A^k}-\beta^{k+1}_{\A^{k+1}\bigcap \A^k}\|^2+\|\beta^{k+1}_{\A^k\backslash \A^{k+1}}\|^2.
\end{align*}
Because
\begin{align*}
&\|\theta^{k+1}_{\A^{k+1}\bigcap \A^k}-\beta^{k+1}_{\A^{k+1}\bigcap \A^k}\|^2\\
&=\|\beta^{k+1}_{\A^{k+1}\bigcap \A^k}-\tau^{k+1} \nabla _{\A^{k+1}\bigcap \A^k}F(\beta^{k+1})-\beta^{k+1}_{\A^{k+1}\bigcap \A^k}\|^2\\
&=(\tau^{k+1})^2 \|\nabla _{\A^{k+1}\bigcap \A^k}F(\beta^{k+1})\|^2\\
&=0,
\end{align*}
so we can get $\|\theta^{k+1}|_{\A^{k+1}}-\beta^{k+1}\|^2=\|\theta^{k+1}_{\A^{k+1}\backslash \A^k}\|^2+\|\beta^{k+1}_{\A^k\backslash \A^{k+1}}\|^2$.
And because $|\A^k\backslash \A^{k+1}|=|\A^{k+1}\backslash \A^k|$ and $\theta^{k+1}_{\A^{k}\backslash \A^{k+1}}=\beta_{\A^{k}\backslash \A^{k+1}}^{k+1}$, then $\|\beta^{k+1}_{\A^k\backslash \A^{k+1}}\|^2=\|\theta^{k+1}_{\A^{k}\backslash \A^{k+1}}\|^2\leq\|\theta^{k+1}_{\A^{k+1}\backslash \A^k}\|^2$. Then
\begin{align}\label{eq4}
\|\theta^{k+1}|_{\A^{k+1}}-\beta^{k+1}\|^2\leq 2 \|\theta^{k+1}_{\A^{k+1}\backslash \A^k}\|^2.
\end{align}
Summing up (\ref{eq3}) and (\ref{eq4}), we have
\begin{align*}
F(\theta^{k+1}|_{\A^{k+1}})-F(\beta^{k+1})&\leq -\tau^{k+1} \|\nabla_{\A^{k+1}\backslash \A^k} F(\beta^{k+1})\|^2+M_{2T} \|\theta^{k+1}_{\A^{k+1}\backslash \A^k}\|^2\\
&=-\tau^{k+1} \|\nabla_{\A^{k+1}\backslash \A^k} F(\beta^{k+1})\|^2+(\tau^{k+1})^2 M_{2T}\|\nabla_{\A^{k+1}\backslash \A^k} F(\beta^{k+1})\|^2\\
&=-\tau^{k+1} (1-\tau^{k+1} M_{2T})\|\nabla_{\A^{k+1}\backslash \A^k} F(\beta^{k+1})\|^2.
\end{align*}

After simple calculation, we can get that the following formula holds when $0<\tau^{k+1}\leq\frac{1-\sigma}{M_{2T}}$
$$F((\beta^{k+1}+\tau^{k+1} d^{k+1})|_{\A^{k+1}})-F(\beta^{k+1})\leq -\sigma \tau^{k+1}||\nabla_{\A^{k+1}\backslash \A^{k}}F(\beta^{k+1})||^2.$$
Therefore, instead of running endlessly, the line search in SDARL algorithm will definitely end in the continuous decrease of the step size.
In addition, the obtained step size must be greater than $\frac{\nu(1-\sigma)}{M_{2T}}$. So far, we have proved the well-defined character of the novel data driven line search rule (\ref{line}).
\end{proof}

The proof of Theorem \ref{le4} is built on the following lemma.

\begin{lemma}\label{lere}
Assumption \ref{ass2} holds and $\|\beta^*\|_0=K\leq T$, then
\begin{align}\label{re1}
\|\nabla_{\A^k\backslash \A^{k-1}}F(\beta^k)\|^2\geq \frac{2 m_{K+T} |\A^k\backslash \A^{k-1}|}{|\A^k\backslash \A^{k-1}|+|\A^*\backslash \A^{k-1}|} [F(\beta^k)-F(\beta^*)].
\end{align}
\end{lemma}

\begin{proof}
Obviously, this lemma holds when $\A^k=\A^{k-1}$ or $F(\beta^k)\leq F(\beta^*)$. Therefore, it is assumed that $\A^k\neq \A^{k-1}$ and $F(\beta^k)> F(\beta^*)$. From Assumption \ref{ass2}, we have
$$F(\beta^*)-F(\beta^k)-\langle\nabla F(\beta^k), \beta^*-\beta^k \rangle\geq \frac{m_{K+T}}{2}||\beta^*-\beta^k||^2.$$
On the one hand,
\begin{align}\label{eq1}
-\langle\nabla F(\beta^k), \beta^*-\beta^k \rangle
&=\langle \nabla F(\beta^k), -\beta^* \rangle\\ \notag
&\geq \frac{m_{K+T}}{2}||\beta^*-\beta^k||^2+F(\beta^k)-F(\beta^*)\\ \notag
&\geq \sqrt{2 m_{K+T}}||\beta^*-\beta^k||\cdot \sqrt{F(\beta^k)-F(\beta^*)}.
\end{align}
Then on the other hand, from the definitions of $\A^k$ and $\A^*$, we get $supp(\nabla F(\beta^k))\bigcap supp(\beta^*)=$ $\A^*\backslash \A^{k-1}$.
Thus, we have
\begin{align}\label{formF}
\langle \nabla F(\beta^k), -\beta^* \rangle=\langle \nabla_{\A^*\backslash \A^{k-1}}F(\beta^k),-\beta_{\A^*\backslash \A^{k-1}}^*\rangle
\leq ||\nabla_{\A^*\backslash \A^{k-1}}F(\beta^k)||\cdot ||\beta_{\A^*\backslash \A^{k-1}}^*||.
\end{align}
Because $\beta^k_{\A^k \backslash \A^{k-1}}=\beta^k_{\A^* \backslash \A^{k-1}}=\text{0}$, so we have $$\min\{|(\nabla_{\A^k \backslash \A^{k-1}}F(\beta^k))_i|, i=1,\cdots,|\A^k \backslash \A^{k-1}|\}\geq \max\{|(\nabla_{\A^* \backslash \A^{k-1}}F(\beta^k))_j|, j=1,\cdots, |\A^* \backslash \A^{k-1}|\}.$$ Then
$$(|\A^k \backslash \A^{k-1}|+|\A^* \backslash \A^{k-1}|)\|\nabla_{\A^k\backslash \A^{k-1}}F(\beta^k)\|^2 \geq |\A^k \backslash \A^{k-1}|\|\nabla_{\A^*\backslash \A^{k-1}}F(\beta^k)\|^2,$$
which yields $$||\nabla_{\A^*\backslash \A^{k-1}}F(\beta^k)||^2 \leq \frac{|\A^k\backslash \A^{k-1}|+|\A^*\backslash \A^{k-1}|}{|\A^k\backslash \A^{k-1}|} \cdot ||\nabla_{\A^k\backslash \A^{k-1}}F(\beta^k)||^2,$$
i.e., $||\nabla_{\A^*\backslash \A^{k-1}}F(\beta^k)|| \leq \sqrt{\frac{|\A^k\backslash \A^{k-1}|+|\A^*\backslash \A^{k-1}|}{|\A^k\backslash \A^{k-1}|}}\|\nabla_{\A^k\backslash \A^{k-1}}F(\beta^k)\|$. So (\ref{formF}) can launch
\begin{align}\label{eq2}
\langle \nabla F(\beta^k), -\beta^* \rangle &\leq \sqrt{\frac{|\A^k\backslash \A^{k-1}|+|\A^*\backslash \A^{k-1}|}{|\A^k\backslash \A^{k-1}|}} ||\nabla_{\A^k\backslash \A^{k-1}}F(\beta^k)|| \cdot ||(\beta^*-\beta^k)_{\A^*\backslash \A^{k-1}}^*||\\ \notag
&\leq \sqrt{\frac{|\A^k\backslash \A^{k-1}|+|\A^*\backslash \A^{k-1}|}{|\A^k\backslash \A^{k-1}|}}\cdot ||\nabla_{\A^k\backslash \A^{k-1}}F(\beta^k)||\cdot ||\beta^*-\beta^k||.
\end{align}
Combining (\ref{eq1}) and (\ref{eq2}),
$$\sqrt{2 m_{K+T}}||\beta^*-\beta^k||\cdot \sqrt{F(\beta^k)-F(\beta^*)}\leq \sqrt{\frac{|\A^k\backslash \A^{k-1}|+|\A^*\backslash \A^{k-1}|}{|\A^k\backslash \A^{k-1}|}}\cdot \|\nabla_{\A^k\backslash \A^{k-1}}F(\beta^k)\|\cdot \|\beta^*-\beta^k\|,$$
i.e., $$\|\nabla_{\A^k\backslash \A^{k-1}}F(\beta^k)\|^2\geq \frac{2 m_{K+T} |\A^k\backslash \A^{k-1}|}{|\A^k\backslash \A^{k-1}|+|\A^*\backslash \A^{k-1}|} [F(\beta^k)-F(\beta^*)].$$
\end{proof}

\subsection{Proof of Theorem \ref{le4}}\label{thel2}

\begin{proof}
Obviously, (\ref{form3}) holds if $\|\beta^k-\beta^*\|< \frac{2}{m_{K+T}}\|\nabla F(\beta^*)\|$, then we consider $\|\beta^k-\beta^*\|\geq$ $\frac{2}{m_{K+T}}\|\nabla F(\beta^*)\|$ below. From the known Assumption \ref{ass2}, we have
$$F(\beta^k)-F(\beta^*)\geq \langle \nabla F(\beta^*), \beta^k-\beta^* \rangle+\frac{m_{K+T}}{2}\|\beta^k-\beta^*\|^2\geq -\|\nabla F(\beta^*)\|\cdot \|\beta^k-\beta^*\|+\frac{m_{K+T}}{2}\|\beta^k-\beta^*\|^2.$$
Regrouping,
$$\frac{m_{K+T}}{2}\|\beta^k-\beta^*\|^2-\|\nabla F(\beta^*)\|\cdot \|\beta^k-\beta^*\|-[F(\beta^k)-F(\beta^*)]\leq0,$$
which is an univariate quadratic inequality about $\|\beta^k-\beta^*\|$. Therefore, we have
\begin{align}\label{ine1}
\|\beta^k-\beta^*\|& \leq  \frac{\|\nabla F(\beta^*)\|+\sqrt{\|\nabla F(\beta^*)\|^2+2m_{K+T}[F(\beta^k)-F(\beta^*)]}}{m_{K+T}}\notag \\
&\leq \frac{2\|\nabla F(\beta^*)\|}{m_{K+T}}+\sqrt{\frac{2\max\{F(\beta^k)-F(\beta^*),0\}}{m_{K+T}}}.
\end{align}

Then reviewing and analyzing the SDARL algorithm, we can get $\tau^k M_{2T}\leq 1-\sigma$. From $m_{K+T}\leq M_{2T}$, we have $$0<\tau^k m_{K+T}\leq \tau^k M_{2T}\leq 1-\sigma<1.$$ Combining the line search with Lemma \ref{lere}, then
$$F(\beta^{k+1})-F(\beta^k)\leq -\frac{2\sigma \tau^k m_{K+T} |\A^k \backslash \A^{k-1}|}{|\A^k \backslash \A^{k-1}|+|\A^* \backslash \A^{k-1}|} [F(\beta^k)-F(\beta^*)].$$
We know $\frac{|\A^k \backslash \A^{k-1}|}{|\A^k \backslash \A^{k-1}|+|\A^*\backslash \A^{k-1}|}\geq \frac{1}{K+1}$ on account of $\frac{|\A^*\backslash \A^{k-1}|}{|\A^k \backslash \A^{k-1}|}\leq K$. Then we have
$$F(\beta^{k+1})-F(\beta^k)\leq - \frac{2 \sigma \tau^k m_{K+T}}{K+1}[F(\beta^k)-F(\beta^*)].$$
Reorganizing,
\begin{align}\label{conresult}
F(\beta^{k+1})-F(\beta^*)\leq \eta_{\tau}[F(\beta^k)-F(\beta^*)], \quad \eta_{\tau}:=1-\frac{2 \sigma \tau^k m_{K+T}}{K+1}\in (0,1).
\end{align}

Then we have
$$F(\beta^k)-F(\beta^*)\leq \eta_{\tau} [F(\beta^{k-1})-F(\beta^*)]\leq \eta_{\tau}^k [F(\beta^{0})-F(\beta^*)].$$
While
\begin{align*}
F(\beta^0)-F(\beta^*)
&\leq \langle \nabla F(\beta^*), \beta^0-\beta^* \rangle+\frac{M_{K+T}}{2}\|\beta^0-\beta^*\|^2\\
&\leq \|\nabla F(\beta^*)\|\cdot \|\beta^0-\beta^*\|+\frac{M_{K+T}}{2}\|\beta^0-\beta^*\|^2\\
&=\|\beta^0-\beta^*\|\cdot\big[\|\nabla F(\beta^*)\|+\frac{M_{K+T}}{2}\|\beta^0-\beta^*\|\big].
\end{align*}
Then we get
$$
F(\beta^k)-F(\beta^*)
\leq \eta_{\tau}^k \|\beta^0-\beta^*\|\cdot\big[\|\nabla F(\beta^*)\|+\frac{M_{K+T}}{2}\|\beta^0-\beta^*\|\big].
$$
Because of the presupposition $\|\beta^k-\beta^*\|\geq \frac{2}{m_{K+T}}\|\nabla F(\beta^*)\|$ for all $k\geq0$, we have
$$F(\beta^k)-F(\beta^*)\leq \eta_{\tau}^k \|\beta^0-\beta^*\|\cdot \frac{m_{K+T}+M_{K+T}}{2}\|\beta^0-\beta^*\|=\frac{m_{K+T}+M_{K+T}}{2}\eta_{\tau}^k \|\beta^0-\beta^*\|^2.$$
Combining with (\ref{ine1}), we can get
\begin{align*}
\|\beta^k-\beta^*\|\leq \frac{2}{m_{K+T}}\|\nabla F(\beta^*)\|+\sqrt{1+\frac{M_{K+T}}{m_{K+T}}}(\sqrt{\eta_{\tau}})^k\|\beta^0-\beta^*\|.
\end{align*}
\end{proof}

\subsection{Proof of Theorem \ref{le5}}\label{thsupp}
\begin{proof}
Considering the Theorem \ref{le4} and the preset conditions in Theorem \ref{le5}, we have
\begin{align*}
\|\beta^k-\beta^*\| &\leq \frac{2}{m_{K+T}}\|\nabla F(\beta^*)\| +\sqrt{1+\frac{M_{K+T}}{m_{K+T}}}(\sqrt{\eta_{\tau}})^k\|\beta^*\|\\
&\leq \frac{2}{3}\|\beta_{\A^*}^{*}\|_{\min}+\sqrt{1+\frac{M_{K+T}}{m_{K+T}}}(\sqrt{\eta_{\tau}})^k\|\beta^*\|\\
& <\|\beta_{\A^*}^{*}\|_{\min},
\end{align*}
if $k>\log_{\frac{1}{\eta_{\tau}}}9(1+\frac{M_{K+T}}{m_{K+T}})\frac{\|\beta^*\|^2}{\|\beta_{\A^*}^{*}\|_{\min}^2}$. This implies that $\A^*\subseteq \A^k$.
\end{proof}
\end{appendix}

\bibliographystyle{Chicago}
\bibliography{GLR}

\begin{thebibliography}{}

\bibitem[\protect\citeauthoryear{Agarwal, Negahban, and Wainwright}{Agarwal
  et~al.}{2012}]{ANW2012}
Agarwal, A., S.~N. Negahban, and M.~J. Wainwright (2012).
\newblock {Fast global convergence of gradient methods for high-dimensional
  statistical recovery}.
\newblock {\em The Annals of Statistics\/}~{\em 40\/}(5), 2452--2482.

\bibitem[\protect\citeauthoryear{Bahmani, Raj, and Boufounos}{Bahmani
  et~al.}{2013}]{BRB2013}
Bahmani, S., B.~Raj, and P.~T. Boufounos (2013).
\newblock {Greedy sparsity-constrained optimization}.
\newblock {\em Journal of Machine Learning Research\/}~{\em 14\/}(3), 807--841.

\bibitem[\protect\citeauthoryear{Bishop}{Bishop}{2006}]{B2006}
Bishop, C.~M. (2006).
\newblock {Pattern recognition and machine learning}.
\newblock {\em Springer\/}.

\bibitem[\protect\citeauthoryear{Boyd and Chu}{Boyd and Chu}{2011}]{BPC2011}
Boyd, S., a. P.~N. and E.~Chu (2011).
\newblock {Distributed optimization and statistical learning via the
  alternating direction method of multipliers}.
\newblock {\em Now Publishers Inc\/}.

\bibitem[\protect\citeauthoryear{Breheny and Huang}{Breheny and
  Huang}{2011}]{BH2011}
Breheny, P. and J.~Huang (2011).
\newblock {Coordinate descent algorithms for nonconvex penalized regression,
  with applications to biological feature selection}.
\newblock {\em The Annals of Applied Statistics\/}~{\em 5\/}(1), 232--253.

\bibitem[\protect\citeauthoryear{Cai and Luo}{Cai and Luo}{2011}]{CLL2011}
Cai, T.T., a. L. W.~D. and X.~Luo (2011).
\newblock {A constrained $\ell_1$ minimization approach to sparse precision
  matrix estimation}.
\newblock {\em Journal of the American Statistical Association\/}~{\em
  106\/}(494), 594--607.

\bibitem[\protect\citeauthoryear{Chen and Saunders}{Chen and
  Saunders}{2001}]{CDS2001}
Chen, S.~S., a. D. D.~L. and M.~A. Saunders (2001).
\newblock {Atomic decomposition by basis pursuit}.
\newblock {\em SIAM Review\/}~{\em 43\/}(1), 129--159.

\bibitem[\protect\citeauthoryear{Chen and Gu}{Chen and Gu}{2017}]{CQ2017}
Chen, J.~H. and Q.~Q. Gu (2017).
\newblock {Fast Newton hard thresholding pursuit for sparsity constrained
  nonconvex optimization}.
\newblock {\em Proceedings of the 23rd ACM SIGKDD International Conference on
  Knowledge Discovery and Data Mining\/}, 757--766.

\bibitem[\protect\citeauthoryear{Fan and Li}{Fan and Li}{2001}]{FL2001}
Fan, J.~Q. and R.~Z. Li (2001).
\newblock {Variable selection via nonconvave penalized likelihood and its
  Oracle properties}.
\newblock {\em Journal of the American Statistical Association\/}~{\em
  96\/}(456), 1348--1360.

\bibitem[\protect\citeauthoryear{Fan and Lv}{Fan and Lv}{2008}]{FL2008}
Fan, J.~Q. and J.~C. Lv (2008).
\newblock {Sure Independence Screening for Ultra-high Dimensional Feature
  Space}.
\newblock {\em Journal of the Royal Statistical Society: Series B (Statistical
  Methodology)\/}~{\em 70\/}(5), 849--911.

\bibitem[\protect\citeauthoryear{Fan, Jiao, and Lu}{Fan et~al.}{2014}]{FJL2014}
Fan, Q.~B., Y.~L. Jiao, and X.~L. Lu (2014).
\newblock {A primal dual active set algorithm with continuation for compressed
  sensing}.
\newblock {\em IEEE Transactions on Signal Processing\/}~{\em 62\/}(23),
  6276--6285.

\bibitem[\protect\citeauthoryear{Friedman, Hastie, H$\ddot{o}$fling, and
  Tibshirani}{Friedman et~al.}{2007}]{FHHT2007}
Friedman, J., T.~Hastie, H.~H$\ddot{o}$fling, and R.~Tibshirani (2007).
\newblock {Pathwise coordinate optimization}.
\newblock {\em The Annals of Applied Statistics\/}~{\em 1\/}(2), 302--332.

\bibitem[\protect\citeauthoryear{Friedman, Hastie, and Tibshirani}{Friedman
  et~al.}{2010}]{FHT2010}
Friedman, J., T.~Hastie, and R.~Tibshirani (2010).
\newblock {Regularization paths for generalized linear models via coordinate
  descent}.
\newblock {\em Journal of Statistical Software\/}~{\em 33\/}(1), 1--22.

\bibitem[\protect\citeauthoryear{Huang, Jin, Liu, Lu, and Yang}{Huang
  et~al.}{2021}]{HJJ2019}
Huang, J., a. J. Y.~L., B.~T. Jin, J.~Liu, X.~L. Lu, and C.~Yang (2021).
\newblock {A unified primal dual active set algorithm for nonconvex sparse
  recovery}.
\newblock {\em Statistical Science\/}~{\em 36\/}(2), 215--238.

\bibitem[\protect\citeauthoryear{Huang, Liu, and Lu}{Huang
  et~al.}{2018}]{HJL2018}
Huang, J., a. J. Y.~L., Y.~Y. Liu, and X.~L. Lu (2018).
\newblock {A constructive approach to $L_0$ penalized regression}.
\newblock {\em Journal of Machine Learning Research\/}~{\em 19\/}(1), 403--439.

\bibitem[\protect\citeauthoryear{Huang, Jiao, Kang, Liu, Liu, and Lu}{Huang
  et~al.}{2021}]{HJK2020}
Huang, J., Y.~Jiao, L.~Kang, J.~Liu, Y.~Liu, and X.~Lu (2021).
\newblock Gsdar: a fast newton algorithm for $\ell _0$ regularized generalized
  linear models with statistical guarantee.
\newblock {\em Computational Statistics\/}~(2).

\bibitem[\protect\citeauthoryear{Jaggi}{Jaggi}{2011}]{J2011}
Jaggi, M. (2011).
\newblock {Sparse convex optimization methods for machine learning}.
\newblock {\em ETH Zurich\/}.

\bibitem[\protect\citeauthoryear{Jiao, Jin, and Lu}{Jiao
  et~al.}{2015}]{JJL2015}
Jiao, Y.~L., B.~T. Jin, and X.~L. Lu (2015).
\newblock {A primal dual active set with continuation algorithm for the
  $\ell_0$-regularized optimization problem}.
\newblock {\em Applied and Computational Harmonic Analysis\/}~{\em 39\/}(3),
  400--426.

\bibitem[\protect\citeauthoryear{Jin, Lorenz, and Schiffler}{Jin
  et~al.}{2009}]{JLS2009}
Jin, B.~T., D.~A. Lorenz, and S.~Schiffler (2009).
\newblock {Elastic-net regularization: error estimates and active set methods}.
\newblock {\em Inverse Problems\/}~{\em 25\/}(11), 115022.

\bibitem[\protect\citeauthoryear{Li and Xiao}{Li and Xiao}{2018}]{LX2018}
Li, P.~L. and Y.~H. Xiao (2018).
\newblock {An efficient algorithm for sparse inverse covariance matrix
  estimation based on dual formulation}.
\newblock {\em Computational Statistics and Data Analysis\/}~{\em 128},
  292--307.

\bibitem[\protect\citeauthoryear{Lu}{Lu}{2009}]{L2009}
Lu, Z.~S. (2009).
\newblock {Smooth optimization approach for sparse covariance selection}.
\newblock {\em SIAM Journal on Optimization\/}~{\em 19\/}(4), 1807--1827.

\bibitem[\protect\citeauthoryear{Luo and Chen}{Luo and Chen}{2014}]{LC2014}
Luo, S. and Z.~H. Chen (2014).
\newblock {Sequential lasso cum EBIC for feature selection with ultra-high
  dimensional feature space}.
\newblock {\em Journal of the American Statistical Association\/}~{\em
  109\/}(507), 1229--1240.

\bibitem[\protect\citeauthoryear{Lustig, Donoho, and Pauly}{Lustig
  et~al.}{2007}]{LDP2007}
Lustig, M., D.~Donoho, and J.~M. Pauly (2007).
\newblock {Sparse MRI: The application of compressed sensing for rapid MR
  imaging}.
\newblock {\em Magnetic Resonance in Medicine: An Official Journal of the
  International Society for Magnetic Resonance in Medicine\/}~{\em 58\/}(6),
  1182--1195.

\bibitem[\protect\citeauthoryear{Meier, Van~de Geer, and B\"{u}hlmann}{Meier
  et~al.}{2008}]{MVB2008}
Meier, L., S.~A. Van~de Geer, and P.~B\"{u}hlmann (2008).
\newblock {The group lasso for logistic regression}.
\newblock {\em Journal of the Royal Statistical Society: Series B (Statistical
  Methodology)\/}~{\em 70\/}(1), 53--71.

\bibitem[\protect\citeauthoryear{Nikolova}{Nikolova}{2013}]{N2013}
Nikolova, M. (2013).
\newblock {Description of the minimizers of least squares regularized with
  $\ell_0$-norm. Uniqueness of the global minimizer}.
\newblock {\em SIAM Journal on Imaging Sciences\/}~{\em 6\/}(2), 904--937.

\bibitem[\protect\citeauthoryear{Park and Hastie}{Park and
  Hastie}{2007}]{PH2007}
Park, M.~Y. and T.~Hastie (2007).
\newblock {$L_1$-regularization path algorithm for generalized linear models}.
\newblock {\em Journal of the Royal Statistical Society: Series B (Statistical
  Methodology)\/}~{\em 69\/}(4), 659--677.

\bibitem[\protect\citeauthoryear{Roth and Fischer}{Roth and
  Fischer}{2008}]{RF2008}
Roth, V. and B.~Fischer (2008).
\newblock {The group-lasso for generalized linear models: uniqueness of
  solutions and efficient algorithms}.
\newblock {\em Proceedings of the 25th international conference on Machine
  Learning\/}, 848--855.

\bibitem[\protect\citeauthoryear{Searle and Gruber}{Searle and
  Gruber}{2016}]{S1971}
Searle, S.~R. and M.~Gruber (2016).
\newblock {Linear models}.
\newblock {\em John Wiley and Sons\/}.

\bibitem[\protect\citeauthoryear{Shen and Li}{Shen and Li}{2017}]{SL2017}
Shen, J. and P.~Li (2017).
\newblock {On the iteration complexity of support recovery via hard
  thresholding pursuit}.
\newblock {\em International Conference on Machine Learning\/}, 3115--3124.

\bibitem[\protect\citeauthoryear{Tibshirani}{Tibshirani}{1996}]{T1996}
Tibshirani, R. (1996).
\newblock {Regression shrinkage and selection via the lasso}.
\newblock {\em Journal of the Royal Statistical Society: Series B
  (Methodological)\/}~{\em 58\/}(1), 267--288.

\bibitem[\protect\citeauthoryear{Van~de Geer}{Van~de Geer}{2008}]{V2008}
Van~de Geer, S.~A. (2008).
\newblock {High-dimensional generalized linear models and the lasso}.
\newblock {\em The Annals of Statistics\/}~{\em 36\/}(2), 614--645.

\bibitem[\protect\citeauthoryear{Wainwright}{Wainwright}{2019}]{M2019}
Wainwright, M.~J. (2019).
\newblock {High-dimensional statistics: A non-asymptotic viewpoint}.
\newblock {\em Cambridge University Press\/}.

\bibitem[\protect\citeauthoryear{Wang and Li}{Wang and Li}{2013}]{WKL2013}
Wang, L., a. K.~Y. and R.~Z. Li (2013).
\newblock {Calibrating nonconvex penalized regression in ultra-high dimension}.
\newblock {\em The Annals of Statistics\/}~{\em 41\/}(5), 2505--2536.

\bibitem[\protect\citeauthoryear{Wang, Xiu, and Zhou}{Wang
  et~al.}{2019}]{WXZ2019}
Wang, R., N.~H. Xiu, and S.~L. Zhou (2019).
\newblock {Fast Newton method for sparse logistic regression}.
\newblock {\em arXiv preprint arXiv:1901.02768\/}, URL
  \url{https://arxiv.org/pdf/1901.02768.pdf}.

\bibitem[\protect\citeauthoryear{Wu and Lange}{Wu and Lange}{2008}]{WL2008}
Wu, T.~T. and K.~Lange (2008).
\newblock {Coordinate descent algorithms for Lasso penalized regression}.
\newblock {\em The Annals of Applied Statistics\/}~{\em 2\/}(1), 224--244.

\bibitem[\protect\citeauthoryear{Xiao and Zhang}{Xiao and Zhang}{2013}]{XZ2012}
Xiao, L. and T.~Zhang (2013).
\newblock {A proximal-gradient homotopy method for the sparse least-squares
  problem}.
\newblock {\em SIAM Journal on Optimization\/}~{\em 23\/}(2), 1062--1091.

\bibitem[\protect\citeauthoryear{Yuan, Li, and Zhang}{Yuan
  et~al.}{2014}]{YLZ2014}
Yuan, X.~T., P.~Li, and T.~Zhang (2014).
\newblock {Gradient hard thresholding pursuit for sparsity-constrained
  optimization}.
\newblock {\em International Conference on Machine Learning\/}, 127--135.

\bibitem[\protect\citeauthoryear{Yuan, Li, and Zhang}{Yuan
  et~al.}{2018}]{YLZ2018}
Yuan, X.~T., P.~Li, and T.~Zhang (2018).
\newblock {Gradient hard thresholding pursuit}.
\newblock {\em Journal of Machine Learning Research\/}~{\em 18}, 1--43.

\bibitem[\protect\citeauthoryear{Yuan and Liu}{Yuan and Liu}{2017}]{YL2017}
Yuan, X.~T. and Q.~S. Liu (2017).
\newblock {Newton-type greedy selection methods for $\ell_0$ -constrained
  minimization}.
\newblock {\em IEEE Transactions on Pattern Analysis and Machine
  Intelligence\/}~{\em 39\/}(12), 2437--2450.

\bibitem[\protect\citeauthoryear{Zhang}{Zhang}{2010a}]{Z2010}
Zhang, C.~H. (2010a).
\newblock {Nearly unbiased variable selection under minimax concave penalty}.
\newblock {\em The Annals of Statistics\/}~{\em 38\/}(2), 894--942.

\bibitem[\protect\citeauthoryear{Zhang}{Zhang}{2008}]{Z2008}
Zhang, T. (2008).
\newblock {Adaptive forward-backward greedy algorithm for sparse learning with
  linear models}.
\newblock {\em Advances in Neural Information Processing Systems\/}~{\em 21},
  1921--1928.

\bibitem[\protect\citeauthoryear{Zhang}{Zhang}{2010b}]{Zh2010}
Zhang, T. (2010b).
\newblock {Analysis of multi-stage convex relaxation for sparse
  regularization}.
\newblock {\em Journal of Machine Learning Research\/}~{\em 11}, 1081--1107.

\bibitem[\protect\citeauthoryear{Zhou, Pan, and Xiu}{Zhou
  et~al.}{2021}]{ZPX2021}
Zhou, S.~L., L.~L. Pan, and N.~H. Xiu (2021).
\newblock {Newton method for $\ell_0$-regularized optimization}.
\newblock {\em Numerical Algorithms\/}, URL
  \url{https://doi.org/10.1007/s11075--021--01085--x}.

\bibitem[\protect\citeauthoryear{Zou and Hastie}{Zou and Hastie}{2005}]{ZH2005}
Zou, H. and T.~Hastie (2005).
\newblock {Regularization and variable selection via the elastic net}.
\newblock {\em Journal of the Royal Statistical Society: Series B (Statistical
  Methodology)\/}~{\em 67\/}(2), 301--320.

\end{thebibliography}
%%%%%%%%%%%%%%%%%%%%%%%%%%%%%%%%%%%%%%%%%%%%%%%%%%%%%%%%%%%%%%%%%%%%%%%%%%%%%%%%%%%%%%%%%%%%%%%%%%%%

\end{document}